\documentclass[11pt]{article}
\usepackage[margin=1in]{geometry}

\usepackage[utf8]{inputenc} 
\usepackage[T1]{fontenc}    
\usepackage{amsthm,amsmath,amssymb,mathtools,amsfonts}
\usepackage[pagebackref,colorlinks=true,linkcolor=blue,filecolor=magenta,urlcolor=cyan,citecolor=blue]{hyperref}
\renewcommand*{\backrefalt}[4]{%
    \ifcase #1 \footnotesize{(Not cited.)}%
    \or        \footnotesize{(Cited on page~#2)}%
    \else      \footnotesize{(Cited on pages~#2)}%
    \fi}
\usepackage{cleveref}
\usepackage{booktabs,multirow,threeparttable,adjustbox}
\usepackage{url}            
\usepackage{nicefrac}       
\usepackage{microtype}      
\usepackage{xcolor}         
\usepackage[colorinlistoftodos,bordercolor=orange,backgroundcolor=orange!20,linecolor=orange,textsize=scriptsize]{todonotes}

\newcommand\ev[1]{\ensuremath{\left \langle #1 \right \rangle}}

\newcommand{\sqn}[1]{{\left\lVert#1\right\rVert}^2}
\newcommand{\norm}[1]{\left\| #1 \right\|}
\newcommand{\br}[1]{\left( #1 \right)}
\newcommand{\cD}{\mathcal{D}}
\newcommand{\cE}{\mathcal{E}}

\newcommand{\sumin}{\sum_{i=1}^n}

\newcommand{\avein}{\frac{1}{n}\sum_{i=1}^n}

\newcommand{\EE}{\mathbb{E}}
\newcommand{\BB}{\mathbb{B}}
\newcommand{\RR}{\mathbb{R}}
\newcommand{\cC}{{\cal C}}

\newcommand{\cX}{{\cal X}}

\newcommand{\del}[1]{}

\newcommand{\eqdef}{\stackrel{\text{def}}{=}}

\newcommand{\R}{\mathbb{R}}

\newcommand{\xglb}{x_\mathrm{global}}
\newcommand{\xavg}{x^\mathrm{avg}}

\newtheorem{innercustomthm}{Theorem}
\newenvironment{customthm}[1]
  {\renewcommand\theinnercustomthm{#1}\innercustomthm}
  {\endinnercustomthm}

\newtheorem{innercustomprop}{Proposition}
\newenvironment{customprop}[1]
  {\renewcommand\theinnercustomprop{#1}\innercustomprop}
  {\endinnercustomprop}

\DeclareMathOperator*{\argmin}{argmin}

\newtheorem{lemma}{Lemma}

\newtheorem{theorem}{Theorem}
\newtheorem{proposition}{Proposition}

\newtheorem{corollary}{Corollary}

\theoremstyle{plain}

\theoremstyle{definition}

\usepackage{enumitem}
\usepackage{blindtext}

\graphicspath{{./plots/}}
\usepackage{algorithm}
\usepackage{algcompatible}
\usepackage[noend]{algpseudocode}
\usepackage{subcaption}
\usepackage{multirow}
\usepackage{makecell}
\usepackage[round]{natbib}
\newcommand\numberlessfootnote[1]{%
  \begingroup
  \renewcommand\thefootnote{}\footnote{#1}%
  \addtocounter{footnote}{-1}%
  \endgroup
}

\title{\textbf{FLIX: A Simple and Communication-Efficient Alternative to Local Methods in Federated Learning} }

\author{Elnur Gasanov\\KAUST \and Ahmed Khaled\\Princeton University \and Samuel Horv\'{a}th\\KAUST \and Peter Richt\'{a}rik\\KAUST}

\date{}

\begin{document}

\maketitle

\begin{abstract}
  Federated Learning (FL) is an increasingly popular machine learning paradigm in which multiple nodes try to collaboratively learn under privacy, communication and multiple heterogeneity constraints. A persistent problem in federated learning is that it is not clear what the optimization objective should be: the standard average risk minimization of supervised learning is inadequate in handling several major constraints specific to federated learning, such as communication adaptivity and personalization control. We identify several key desiderata in frameworks for federated learning and introduce a new framework, FLIX, that takes into account the unique challenges brought by federated learning. FLIX has a standard finite-sum form, which enables practitioners to tap into the immense wealth of existing (potentially non-local) methods for distributed optimization. Through a smart initialization that does not require any communication, FLIX does not require the use of local steps but is still provably capable of performing dissimilarity regularization on par with local methods. We give several algorithms for solving the FLIX formulation efficiently under communication constraints. Finally, we corroborate our theoretical results with extensive experimentation. \numberlessfootnote{* To appear in AISTATS 2022. An earlier version of this paper appeared in the NeurIPS 2021 Meta-Learn workshop with the algorithm name FedMix, which we have changed to FLIX to distinguish from a different algorithm developed in concurrent work.}
\end{abstract}

\section{INTRODUCTION}

There is a wealth of data today that is decentralized among many clients and which can not be centralized for privacy reasons. Federated Learning (FL) aims to enable machine learning in this decentralized setting while respecting data privacy. Application domains of federated learning include healthcare, learning language models for virtual keyboards, and speech recognition~\citep{kairouz19_advan_open_probl_feder_learn}. The promise of federated learning is that by participating in a distributed training process, clients can learn better machine learning models than they can using only their own data. The main cost in using federated learning over local training lies in the network bandwidth used for the distributed training process. Hence, federated learning must be flexible enough to provide a benefit to users without a prohibitive communication cost. The standard formulation of FL is to cast it as an optimization problem of the form
\begin{equation}\label{eq:ERM}\tag{ERM} \min \limits_{x\in \R^d}  \left[ f(x) \eqdef \frac{1}{n}\sum \limits_{i=1}^n f_i(x)\right],\end{equation}
where $f_i$ is the loss function on client $i$. Thus, the goal of classical FL is for the $n$ clients to collaboratively learn a single model, $x^* = \arg\min f$, to be deployed on all clients. Recent development shows that using a single model for all clients can be severely detrimental to individual performance on many clients~\citep{yu20_salvag_feder_learn_by_local_adapt}, defeating the purpose of joining distributed training. Furthermore, \eqref{eq:ERM} offers no clear tunable knobs that can accommodate constraints on the network bandwidth. Motivated by these concerns, the chief question of our paper is:
\begin{quote}
 \em  Can we find a formulation for federated learning that is flexible enough to accommodate the needs of federated learning, yet also solvable using standard methods?
\end{quote}

\subsection{Key properties of the FLIX framework}
\label{sec:desiderata}

Our main contribution is \emph{FLIX} (Federated Learning mIXtures), a novel and flexible formulation for federated learning: define $\alpha_i > 0$ to be the \emph{personalization parameter for node $i$}, and let $x_i \eqdef \min_{x \in \R^d} f_i (x)$ be the local solution to the $i$-th objective-- note that $x_i$ can be found by solely running a local optimizer, and hence computing it requires no communication at all. If $f_i$ is nonconvex, then $x_i$ can be merely a stationary point, i.e. a point such that $\nabla f_i (x_i) = 0$. The FLIX problem is
\begin{equation}
  \label{eq:FLIX-problem}
  \tag{FLIX}
  \min \limits_{x \in \R^d} \tilde{f} (x) \eqdef \frac{1}{n} \sum \limits_{i=1}^{n} f_i (\alpha_i x + \br{1 - \alpha_i} x_i).
\end{equation}

Once we find a solution $x_\ast$ of sufficient quality for~\eqref{eq:FLIX-problem}, we deploy $T_i (x_\ast) = \alpha_i x_\ast + \br{1 - \alpha_i} x_i$ on node $i$ as its final personalized model. We now enumerate some of the key properties of \eqref{eq:FLIX-problem}:

\begin{itemize}[leftmargin=0.15in,itemsep=0.01in,topsep=0pt]
\item \textbf{Efficiently solvable as a finite-sum problem.} Formulation~\eqref{eq:FLIX-problem} preserves the standard finite-sum formulation of empirical risk minimization. Moreover, it preserves problem structure: we show (in Section~\ref{sec:algorithms-for-FLIX}) that when the $f_i$ are smooth (and/or convex), $\tilde{f}$ is also smooth (resp. convex). Hence, we can leverage standard empirical risk minimization method for federated learning under communication and/or personalization constraints. In Section~\ref{sec:algorithms-for-FLIX} we instantiate several such methods and show they enjoy fast convergence guarantees under standard assumptions on the functions $f_1, \ldots, f_n$.
\item \textbf{Adaptive to communication constraints.} Communication efficiency is an important concern federated learning, as often bandwidth is valuable and limited \citep{Konecny2016, li19_feder_learn, kairouz19_advan_open_probl_feder_learn}. The FLIX formulation is adaptive to communication constraints: observe that computing $x_i$, a precondition to solving \eqref{eq:FLIX-problem}, requires no communication at all, and can be done purely locally on node $i$. If $\alpha_i = 0$, then no communication at all is needed to compute the personalized model $T_i (x_\ast)$. By varying $\alpha_i$ between $0$ and $1$, we can control the amount of communication needed to compute $T_i (x_\ast)$. We show (in Section~\ref{sec:algorithms-for-FLIX}) that given a communication budget of $R$ steps, we can find parameters $\alpha_i$ that allows us to solve \eqref{eq:FLIX-problem} in no more than $R$ communication steps.
\item \textbf{Adaptivity to personalization.} Our end-goal in federated learning is to generalize well \emph{on each client}: this means that the solution deployed on node $i$ should be tailored to its local data distribution, which may differ from the data distributions on other nodes. Hence, we want to use the data from other nodes to train the model deployed on node $i$ only if that data is useful on its own distribution. Indeed, a naive application of empirical risk minimization can have terrible effects in federated learning situations where the pure local model (i.e. $x_i$) does better than average empirical risk minimizer \citep{yu20_salvag_feder_learn_by_local_adapt}. In FLIX, varying $\alpha_i$ enables us to amplify or reduce the effect of other objectives on the solution deployed on node $i$. In situations where the data on all of the nodes is sufficiently heterogeneous, we set $\alpha_i$ to be small and the effect of other data on node $i$ will be neglibile. On the other hand, when the data on the different nodes is related we may set $\alpha_i$ to be closer to $1$. We observe a benefit to varying $\alpha$ in this manner in practice: Figure~\ref{fig:alphai-matters} shows the effect of varying the $\alpha_i$ on real data (see Section~\ref{sec:experiments} for the details and for other experiments).
\end{itemize}

\begin{figure}[h]
\centering
\begin{subfigure}{0.45\linewidth}
\centering
\includegraphics[width=\linewidth]{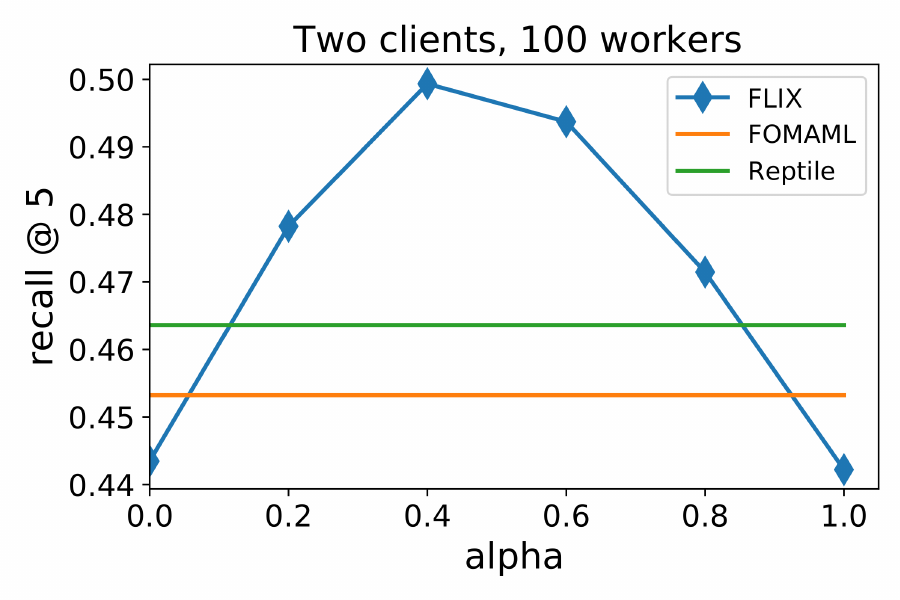}
\caption{100 workers created out of two clients' data}
\label{subfig:100w_2c}
\end{subfigure}
\begin{subfigure}{0.45\linewidth}
\centering
\includegraphics[width=\linewidth]{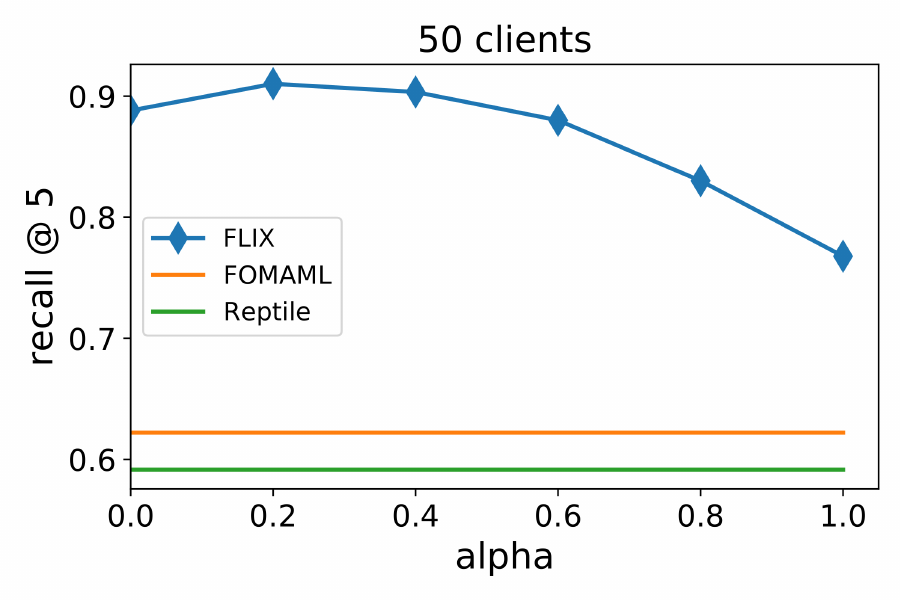}
\caption{50 workers with distinct data distributions}
\label{subfug:50w_50c}
\end{subfigure}
\vspace{.1in}
\caption{Test accuracy of FLIX model for different personalization parameter values, FOMAML and Reptile. $\alpha_i = \alpha$ is set to the value indicated on horizontal axis. FOMAML and Reptile are independent from the personalization parameter $\alpha$. Plots correspond to different data splittings.}
\label{fig:alphai-matters}
\end{figure}

FLIX fills a gap that is unsatisfied by existing methods. To the best of our knowledge, there is no other method for federated learning that is efficiently solvable via standard algorithms and also adaptive to communication and personalization constraints, and indeed both constraints are important in practice~\citep{li2020industryfriendly}. We believe the key properties we enumerate can also serve as natural desiderata in the development of new formulations and methods for federated learning.

\subsection{Related work}

Personalization has garnered significant recent interest in federated learning as personalized models often perform well in practice compared to non-personalized models~\citep{jiang19_improv_feder_learn_person_via,yu20_salvag_feder_learn_by_local_adapt}. FLIX is a \emph{model mixture} method: the personalized solution is a mixture of a global model and a local model. In recent work, \citet{expl_mixture_Deng2020} and \citet{mansour2020approaches} propose model mixture methods and prove their statistical benefits, while \citet{zec2021specialized} introduce a similar formulation based on the mixture of experts framework. Unfortunately, we show in the supplementary material that, from the perspective of optimization and without additional data, the formulations in all three works are trivially minimized at the local minimizers $x_1, \ldots, x_n$. An alternative to model mixing is mixing in function space, where we optimize a mixture of objectives rather than a model mixture. This mixture is often constructed to control model variance: examples of this approach can be found in~\citep{FL2020mixture, dinh2021personalized,huang2021personalized, dinh2021fedu}. In FLIX, we take the model mixture approach as it allows us to use pretraining to better solve the problem while still regularizing model variance (see Section~\ref{sec:motivation-1}). A parallel line of work applies meta-learning methods like MAML to federated learning~\citep{jiang19_improv_feder_learn_person_via,fallah2020personalized}: in Section~\ref{sec:motivation-2} we motivate FLIX by taking MAML as our starting point. \citet{chen2021theorem} discuss the statistical limits of personalization and show that either solving empirical risk minimization or local training is optimal, depending on certain problem parameters; However, as of yet there is no single optimal adaptive algorithm (from the statistical perspective). There are several other techniques in federated learning that can be combined with our approach for better results, such as clustering~\citep{sattler19_clust_feder_learn} or robust optimization~\citep{reisizadeh20_robus_feder_learn}.

\section{THE FLIX FORMULATION}\label{sec:FLIX_formulation}
In this section we reintroduce and motivate the FLIX formulation in detail. We define the FLIX objective as
\begin{equation}
  \label{eq:FLIX-objective}
\begin{split}
\tilde{f}(x; \alpha_1, \ldots, \alpha_n, x_1, \ldots, x_n) &\eqdef \frac{1}{n} \sum \limits_{i=1}^{n} f_i (\alpha_i x + \br{1-\alpha_i} x_i),
\end{split}
\end{equation}
where $\alpha_i \in (0, 1)$ is the personalization coefficient for node $i$ and $x_i$ is the minimizer of $f_i$, for all $i = 1, 2, \ldots, n$. We will use $\tilde{f}(x)$ to refer to the objective in \eqref{eq:FLIX-objective} when the $\alpha_i$ and $x_i$ are clear from the context. The FLIX problem is then
\begin{equation}
  \label{eq:FLIX-2}
  \min_{x \in \R^d} \left [ \tilde{f}(x) = \frac{1}{n} \sum \limits_{i=1}^{n} f_i (\alpha_i x + \br{1-\alpha_i} x_i) \right ].
\end{equation}
Let $\alpha = [\alpha_1, \ldots, \alpha_n]$ be the vector of the personalization coefficients. If $x_\ast = x_\ast (\alpha)$ is a solution of \eqref{eq:FLIX-2}, we call $T_{i} (x; \alpha_i, x_i) = \alpha_i x_\ast + \br{1 - \alpha_i} x_i$ the \emph{deployed solution on node $i$}. Like with $\tilde{f}$, we will refer to the deployed solution on node $i$ as $T_i (x)$ when $x_i$ and $\alpha_i$ are clear from the context.

\subsection{Motivation 1: from Local GD to FLIX}
\label{sec:motivation-1}
The most popular algorithm for solving federated learning problems is the Federated Averaging algorithm~\citep{kairouz19_advan_open_probl_feder_learn}, also known as Local (Stochastic) Gradient Descent (Local GD/SGD). Local GD alternates steps of local computation on each node with steps of communication and aggregation. More concretely, the Local GD update is:
\begin{equation*}
  x_{t+1}^{i} = \begin{cases}
    x_{t}^i - \gamma \nabla f_{i} (x_t^i) & \text{ if } t \mod H \neq 0 \\
    \frac{1}{n} \sum_{i=1}^{n} \left [x_{t}^i - \gamma \nabla f_{i} (x_t^i) \right ] & \text{ if } t \mod H = 0
  \end{cases},
\end{equation*}

where $H$ is the number of local steps. Early papers on federated learning (such as e.g. \citep{Konecny2016}) motivated local methods as communication-efficient ways of solving~\eqref{eq:ERM}, but subsequent theoretical development reveals that local methods are, in fact, quite bad solvers for~\eqref{eq:ERM} whenever there is significant statistical heterogeneity among the clients \citep{woodworth20_minib_vs_local_sgd_heter_distr_learn}. Moreover, \citet{pathak20_fedsp} show that for the linear least-squares problem, Local GD converges to a different point than the minimizer of \eqref{eq:ERM}. More generally, the fixed points of Local GD can be very different from the minimizer of \eqref{eq:ERM} whenever $H > 1$ \citep{malinovsky20_from_local_sgd_to_local}. \citet{FL2020mixture} show that a mild variant of Local GD can be interpreted as SGD applied on the $nd$-dimensional regularized objective $f_\lambda$ defined by
\begin{equation*}
  f_{\lambda} (y_1, y_2, \ldots, y_n) \eqdef  \left[ \frac{1}{n} \sum \limits_{i=1}^{n} f_i (y_i) + \frac{\lambda}{2n} \sum \limits_{i=1}^{n} \sqn{y_i - \bar{y}}\right],
\end{equation*}
where $\bar{y} = \frac{1}{n} \sum_{i=1}^{n} y_{i}$ is the counterpart of $x$ in \eqref{eq:ERM}, and where $\lambda$ is a regularization parameter determined according to the number of local steps. Objective $f_\lambda$ is the summation of two terms: the first asks that each node $i$ finds a solution $y_i$ that minimizes its local objective well, while the regularizer $\psi (y_1, \ldots, y_n) = \frac{1}{2n} \sum_{i=1}^{n} \sqn{y_i - \bar{y}}$ forces the solutions $y_1, y_2, \ldots, y_n$ to be close to their average $\bar{y}$. Hence, Local GD incentivizes finding \emph{personalized} solutions $y_1, y_2, \ldots, y_n$ that have small population variance. \citet{FL2020mixture} note that as the $\lambda$ parameter varies between $0$ and $\infty$, the solutions found by Local GD interpolate between the pure local optimal models (i.e. $x_i = \argmin_{x} f_i (x)$) and the solution of the global problem $x_\ast$ (the minimizer of~\eqref{eq:ERM}). Our starting point is then the following observation:
\begin{quote}
  The solutions $y_1, \ldots, y_n$ found by Local GD are an \emph{implicit mixture} of the local minimizers $x_1, \ldots, x_n$ and the global empirical risk minimizer $x_\ast$.
\end{quote}

Rather than seeking an implicit mixture of the local and global optimal models, we instead propose to find an \emph{explicit} mixture of the local optimal models and a global model: given any global model $x$ (not necessarily the empirical risk minimizer), we choose coefficients $\alpha_1, \alpha_2, \ldots, \alpha_n$ (all between $0$ and $1$) and then deploy on node $i$ the mixture
\begin{equation}
  \label{eq:deployed-models-def}
  T_i (x) = \alpha_i x + \br{1-\alpha_i} x_i,
\end{equation}
we may then choose $x$ as \emph{the best} such global model by explicitly solving the optimization problem
\[ \min \limits_{x \in \R^d} \frac{1}{n} \sum \limits_{i=1}^{n} f_{i} \br{\alpha_i x + \br{1 - \alpha_i} x_i},  \]
and this is exactly the FLIX formulation. Observe that coefficients $\alpha_1, \alpha_2, \ldots, \alpha_n$ should regularize the population variance of the deployed solutions $T_1 (x), T_2 (x), \ldots, T_n (x)$, as in local methods. We show this rigorously for equal $\alpha_i$ in Proposition~\ref{proposition:alpha-controls-variance}. Our development thus leads us to a natural framework that captures the strength of local methods while also satisfying the desiderata specified in Section~\ref{sec:desiderata}.

\subsection{Motivation 2: from model-agnostic meta-learning to FLIX}
\label{sec:motivation-2}
We now motivate FLIX differently by starting with \emph{personalization via fine-tuning}. The ordinary formulation of the federated learning problem~\eqref{eq:ERM} asks for a single global model to be used on all clients. If the clients are sufficiently heterogeneous, a single model may perform badly on many of them~\citep{jiang19_improv_feder_learn_person_via}. Personalizing a global model to each of the users' custom data is often beneficial in practice; For example, \citet{wang19_feder_evaluat_devic_person} study the benefits of personalizing language models for a virtual keyboard application used by tens of millions of users. They observe that a sizeable fraction of the users benefit from personalization. Personalization is often done in two steps:
\begin{algorithmic}
  \State \textbf{Step I: initial model training.} Find a ``good'' global model $\xglb$.
  \State \textbf{Step II: fine-tuning.} Personalize the global model $\xglb$ on each client to get the \emph{personalized} local models $x_i$.
\end{algorithmic}

Methods that fit this framework are known as \emph{finite-tuning approaches}: they include the model-agnostic meta-learning (MAML) family of methods~\citep{finn17_model_agnos_meta_learn_fast}. In addition to its practical popularity, recent theoretical investigations reveal that fine-tuning approaches, such as MAML, are also benefical from a statistical perspective~\citep{fallah21_gener_model_agnos_meta_learn_algor,chua21_how_fine_tunin_allow_effec_meta_learn}. In MAML, we find $\xglb$ by optimizing for the loss after a single step of gradient descent, i.e.\ the MAML objective is
\begin{equation}
  \label{eq:MAML-objective}
  \text{Find } \quad \xglb \in \argmin \limits_{x \in \R^d} \frac{1}{n} \sum \limits_{i=1}^n f_{i} (x - \gamma \nabla f_{i} (x)),
\end{equation}

where $\gamma$ is a given stepsize. Once $\xglb$ is found, we may then fine-tune it by running gradient descent for a number of steps on each node $i$ locally using its own objective $f_i$~\citep{finn17_model_agnos_meta_learn_fast}. To gain further insight into what fine-tuning is doing, we now consider the case when each $f_i$ is a quadratic function. Because this problem is amenable to analysis, several authors have used it to study the theoretical properties of MAML~\citep{collins20_why_does_maml_outper_erm, Charles2021,gao20_model_optim_trade_off_meta_learn}, and we follow in their footsteps. We assume that each $f_i$ can be written as $f_i (x) = \frac{1}{2} x^T A_i x - b_i^T x + c_i$,
where $A_i \in \R^{d \times d}$ is a positive definite matrix, $b_i \in \R^d$, and $c_i \in \R$. Now suppose that we have some initial global model $x^{0}$, and we fine-tune it by running gradient descent for $H$ steps on node $i$: the next proposition shows the final iterate is a matrix-weighted average of the initial solution and the optimal local solution:

\begin{proposition}
  \label{proposition:quad-fine-tuned-sol}
  Suppose that we run gradient descent for $H$ steps on the quadratic objective $f_i~=~ \frac{1}{2}~x^T A_i x~-~b_i^T x~+~c$ starting from $x^{0}$ with stepsize $\gamma > 0$. Suppose that the stepsize satisfies $\gamma \leq \frac{1}{L_i}$, where $L_i = \lambda_{\max} (A_i)$. Then the final iterate $x_i^{H}$ can be written as
  \[ x_i^{H} = \br{I - J_i^H} x_i + J_i^H x^{0}, \]
  where $x_i$ minimizes $f_i$ and $J_i \in \R^{d \times d}$ is a matrix with maximum eigenvalue smaller than $1$, i.e.\ $\lambda_{\max} (J) < 1$.
\end{proposition}

The proof of Proposition~\ref{proposition:quad-fine-tuned-sol} and all subsequent proofs are relegated to the supplementary. Plugging the result of Proposition~\ref{proposition:quad-fine-tuned-sol} into Equation~\eqref{eq:MAML-objective}, observe that in MAML we find the initial model $x^{0}$ by solving the problem
\[ \min \limits_{x \in \R^d} \frac{1}{n} \sum \limits_{i=1}^{n} f_i ( (I - J_i) x_i + J_i x ).  \]

Hence, MAML is optimizing \emph{for a specific weighted average} of the initial model $x^{0}$ and the local solutions $x_1, x_2, \ldots, x_n$. We thus propose to dispense with the specific matrix $J_i$ and instead optimize an average weighted with an arbitrary constant $\alpha_i$:
\[ \min \limits_{x \in \R^d} \frac{1}{n} \sum \limits_{i=1}^{n} f_i (\alpha_i x + \br{1-\alpha_i} x_i), \]
and this is exactly the FLIX formulation. Observe that by choosing $\alpha_i$ we may accomplish a similar effect to multiplying by $J_i^H$ for large $H$, so this is a valid approximation at least for large $H$. This gives FLIX a new interpretation as an approximate generalized MAML, where we optimize the global model for performance after potentially many gradient descent steps rather than just a single step. By exploiting information about the objective's optimum $x_i$ (available at zero communication cost in federated learning), FLIX takes into account higher-order information that ordinary MAML can not.

\section{THEORY AND ALGORITHMS}
\label{sec:algorithms-for-FLIX}

In this section we aim to develop algorithms to solve~\eqref{eq:FLIX-problem} in a communication-efficient manner. Before discussing concrete algorithms, we study a few algorithm-independent properties of~\eqref{eq:FLIX-problem} that will come in handy for understanding the formulation and proving convergence bounds. The following proposition shows that the formulation preserves smoothness and convexity. This is in contrast to other meta-learning methods such as MAML, where the objective may generally be nonsmooth~\citep{fallah19_conver_theor_gradien_based_model}.

\begin{proposition}
  \label{proposition:preserve-smooth-and-convexity}
  Suppose that each objective $f_i$ is $L_i$-smooth. That is, for any $x, y \in \R^d$ we have $\norm{\nabla f_i (x) - \nabla f_i (y)} \leq L_i \norm{x-y}$. Then the FLIX objective $\tilde{f}$ defined in \eqref{eq:FLIX-objective} is $L_\alpha$-smooth for $L_\alpha \eqdef \frac{1}{n} \sum_{i=1}^{n} \alpha_i^2 L_i$. If each $f_i$ is convex, then $\tilde{f}$ is also convex. If each $f_i$ is $\mu_i$-strongly convex, then $\tilde{f}$ is $\mu_\alpha$ strongly convex for $\mu_\alpha \eqdef \frac{1}{n} \sum_{i=1}^{n} \alpha_i^2 \mu_i$.
\end{proposition}

Our next result offers some insight into the variance-regularizing effect of the $\alpha_i$: in particular, when all the $\alpha_i$ are equal, increasing $\alpha$ in~\eqref{eq:FLIX-problem} directly decreases the variance of the deployed local models from their mean. As discussed in Section~\ref{sec:motivation-1}, this is a key property of local descent methods that the FLIX formulation captures.

\begin{proposition}
  \label{proposition:alpha-controls-variance}
  Suppose that $\alpha_1 = \alpha_2 = \ldots = \alpha_n = \beta$ in the FLIX formulation~\eqref{eq:FLIX-problem}. Let $T_1 (x), T_2 (x), \ldots, T_n(x)$ be the deployed models defined in~\eqref{eq:deployed-models-def}. If $y_1, \ldots, y_n$ are vectors in $\R^d$ and $\bar{y}$ is their mean, we define $V(y_1, \ldots, y_n)$ as the population variance $$ V(y_1, \ldots, y_n) \eqdef \frac{1}{n} \sum \limits_{i=1}^{n} \sqn{y_i - \bar{y}}.$$ Then,
  \[ V(T_1 (x), \ldots, T_n (x)) = \br{1 - \beta}^2 V(x_1, \ldots, x_n). \]
\end{proposition}

One-shot learning is a learning paradigm where we may use only a single round of communication to solve the federated learning problem~\citep{guha19_one_shot_feder_learn, salehkaleybar19_one_shot_feder_learn}. When the personalization parameters are small enough, we can provably solve the FLIX problem with a single round of communication by computing a certain weighted average of the local solutions $x_1, x_2, \ldots, x_n$. We first consider if each $x_i$ is a minimizer of $f_i$, then we can get an approximate minimizer:
\begin{theorem}
  \label{thm:one-shot-averaging}
  Suppose that each objective $f_i$ is $L_i$-smooth, that each $x_i$ minimizes $f_i$, and let $\hat{L} \eqdef \frac{1}{n} \sum_{i=1}^{n} L_i$. Given the pure local models $x_1, x_2, \ldots, x_n$, define the weighted average
  \begin{align}
    \label{eq:xavg-def}
    \xavg \eqdef \sum \limits_{i=1}^{n} w_i x_i, &&  w_i \eqdef \frac{\alpha_i^2 L_i}{n L_\alpha}, &&  L_\alpha \eqdef \frac{1}{n}\sum \limits_{i=1}^{n} \alpha_i^2 L_i.
  \end{align}
  We further define the constants
  \begin{equation}
  \begin{aligned}
    \label{eq:var-consts-def}
    D &\eqdef \max \limits_{i, j = 1, \ldots, n, i \neq j} \norm{x_i - x_j} &&\text{ and, }  V &\eqdef \sum \limits_{i=1}^{n} w_i \sqn{x_i - \xavg}
  \end{aligned}
\end{equation}
  Fix any $\epsilon > 0$. Assume that either $\max_{i=1, \ldots, n} \alpha_i \leq \sqrt{2\epsilon}/\sqrt{\hat{L}D}$, or $\alpha_i = \beta$ for all $i$ and $\beta \leq \sqrt{2\epsilon}/\sqrt{\hat{L}D}$. Then $\xavg$ is an $\epsilon$-approximate minimizer of ~\eqref{eq:FLIX-problem}.
\end{theorem}

We can relax this to the requirement that each $x_i$ is a stationary point, in which we case we also get another stationary point using weighted averaging, see Theorem~\ref{thm:one-shot-averaging-stationary-points} in the supplementary for details. We point out that this implies FLIX applies equally well to objectives with minimizers (e.g.\ when the $f_i$ are strongly convex) as well as when the objective is non-convex and the goal is attaining a stationary point rather than minimization. For larger $\alpha_i$ than that required by Theorem~\ref{thm:one-shot-averaging} or Theorem~\ref{thm:one-shot-averaging-stationary-points}, we need more communication rounds than one. In the next subsection, we describe how distributed gradient descent can be used to solve the problem.

\subsection{Distributed gradient descent}
The simplest approach to solving~\eqref{eq:FLIX-problem} is via distributed gradient descent: given the local models $x_1, x_2, \ldots, x_n$ (precomputed before starting the process) and an initial global model $x^0$, we run the update
\begin{equation}
  \label{eq:dgd-update}
  \tag{DGD}
  x^{k+1} = x^k - \frac{\gamma}{n} \sum \limits_{i=1}^{n} \alpha_i \nabla f_i \br{\alpha_i x^k + \br{1 - \alpha_i} x_i}.
\end{equation}

Each iteration of~\eqref{eq:dgd-update} requires that every node sends its local gradient $\nabla \tilde{f}_i (x^k)$ to the server. The server then averages the received gradients, computes the new iterate $x^{k+1}$, and broadcasts it to the nodes. The next theorem shows that under smoothness and strong convexity, distributed gradient descent converges linearly to the solution of~\eqref{eq:FLIX-problem}.

\begin{theorem}
  \label{theorem:dgd-convergence}
  Suppose that each $f_i$ in \eqref{eq:FLIX-problem} is $L_i$-smooth and $\mu_i$-strongly convex. Define $\xavg, L_\alpha$, and $\hat{L}$ by~\eqref{eq:xavg-def} and $V, D$ by~\eqref{eq:var-consts-def}. Suppose that we run~\eqref{eq:dgd-update} for $K$ iterations starting from $x^0 = \xavg$. Then the following hold:
  \begin{itemize}[leftmargin=0.15in,itemsep=0.01in,topsep=0pt]
    \item [i)] If the $\alpha_i$ are allowed to be arbitrary, then for $\alpha_{\max} \eqdef \max_{i=1,\ldots,n} \alpha_i$ we have
      \[  \tilde{f}(x^k) - \min\limits_{x \in \R^d} \tilde{f}(x) \leq \br{1 - \frac{\mu_\alpha}{L_\alpha}}^K \frac{\alpha_{\max}^2 \hat{L} D}{2}. \]
    \item [ii)] If $\alpha_i = \beta$ for all $i$, then
      \begin{equation}
        \tilde{f}(x^k) - \min_{x \in \R^d} \tilde{f}(x) \leq \br{1 - \frac{\hat{\mu}}{\hat{L}}}^K \frac{\beta^2 \hat{L} V}{2},
      \end{equation}
      where $\hat{\mu} \eqdef \frac{1}{n} \sum_{i=1}^{n} \mu_i$.
\end{itemize}
\end{theorem}

There are four ways of making the right hand side in~\eqref{eq:dgd-conv-beta} small (i.e.\ decrease the communication complexity):
\begin{itemize}[leftmargin=0.15in,itemsep=0.01in,topsep=0pt]
\item \textbf{Communicate more.} Increase the number of communications $K$.
\item {\bf Homogeneous data}. The variance $V=\sum_{i=1}^{n} w_i \sqn{x_i - \xavg}$ can be seen as a measure of data heterogeneity. More homogeneous data sets will have smaller $V$, which leads to better performance.
\item {\bf Train simpler models.} Focusing attention on models with smaller $\hat{L}$ (adjust model design), or larger $\hat{\mu}$ (e.g., add more regularization).
\item {\bf Put more weight on local models.} If we prefer local models to the global model, then  $\alpha_i$ is small, and hence fewer communications are needed to achieve any given accuracy.
\end{itemize}

Armed with Theorem~\ref{theorem:dgd-convergence}, we make good on our promise in Section~\ref{sec:desiderata} and show that FLIX can be solved using \emph{any} communication budget. Looking at~\eqref{eq:dgd-conv-beta} we see that for any fixed $\epsilon > 0$ we have $$\tilde{f}(x^k) - \min_{x \in \R^d} \tilde{f}(x) \leq \epsilon$$ as long as $$\beta \leq A q^k,$$ where $A=\sqrt{2\epsilon/(\hat{L} V)}$ and $q = 1/\sqrt{1-\frac{\hat{\mu}}{\hat{L}}}$.

Putting this together leads to the following observations:
\begin{itemize}[leftmargin=0.15in,itemsep=0.01in,topsep=0pt]
\item If $\beta=0$, the problem can be solved with $0$ communications (i.e.\ ,\ each device $i$ independently computes the pure local model $x_i$).
\item If $0<\beta \leq A$, the problem can be solved with $1$ communication (i.e., compute $x^{avg}$). This follows from Theorem~\ref{thm:one-shot-averaging}, and also from the more general result Theorem~\ref{theorem:dgd-convergence} by setting $K=0$.
\item If  $A <\beta \leq Aq$, the problem can be solved with $2$ communications (1 communication to compute $x^0=x^{avg}$, followed by one iteration of distributed GD).
\item  If $Aq^{k-1} <\beta \leq A q^k$,  the problem can be solved with $k+1$ communications (1 communication to compute $x^0=x^{avg}$, followed by $K$ iterations of distributed gradient descent).
\item If $\beta=1$, we need $1$ communication to compute $x^0=x^{avg}$, followed by $k\geq \frac{\bar{L}}{\bar{\mu}}\log \frac{\bar{L}V}{2\varepsilon}$ iterations of distributed gradient descent. This is recovers the standard communication complexity of gradient descent needed to find the optimal solution of the average risk minimization problem~\eqref{eq:ERM}.
\end{itemize}

In the supplementary, we develop other algorithms for solving~\eqref{eq:FLIX-problem} such as distributed gradient descent with compression~\citep{alistarh16_qsgd} and DIANA~\citep{diana_paper}. We point out that \Cref{thm:non_cvx_dcgd} and \Cref{thm:diana_nncvx} analyze the compressed gradient descent and DIANA algorithms applied in the non-convex case, showing explicit dependence of the convergence rates on $\alpha$, thus this convergence theory applies to convex as well as non-convex objectives. Because~\eqref{eq:FLIX-problem} has a standard finite-sum form, many more algorithms can be used to solve it, e.g.\ accelerated minibatch SGD~\citep{cotter11_better_mini_batch_algor_via} or SARAH~\citep{nguyen17_sarah}. We leave the task of finding the \emph{optimal} algorithm for~\eqref{eq:FLIX-problem} to future work.

\section{EXPERIMENTS}{\label{sec:experiments}}

{\bfseries Logistic regression with $l_2$ regularizer. } For our first experiment, we consider a setup where each device runs regularized logistic regression:
\begin{align}\label{fnct:logreg_l2}
 f_i(x) \coloneqq \frac{1}{k_i}\sum\limits_{j=1}^{k_i} \left[ \log{(1 + \exp{(-a_{i,j}^\top x)})}\right] + \frac{\lambda}{2} \|x\|^2,
\end{align}
where $a_{i,j} \in \RR^d$ are given for all $j = 1 \dots k_i$, $k_i$ is a number of data points associated with device $i$ and $\lambda$ is a regularization parameter. The objective given by~\eqref{fnct:logreg_l2} is typically used in classification problems. It is not difficult to show that $f_i$ is $\nicefrac{1}{4k_i}\norm{\sum_{i=1}^{k_i} a_i a_i^\top } + \lambda$-smooth and $\lambda$-strongly convex.

\begin{figure}[h]
\begin{subfigure}{0.32\linewidth}
\includegraphics[width=\linewidth]{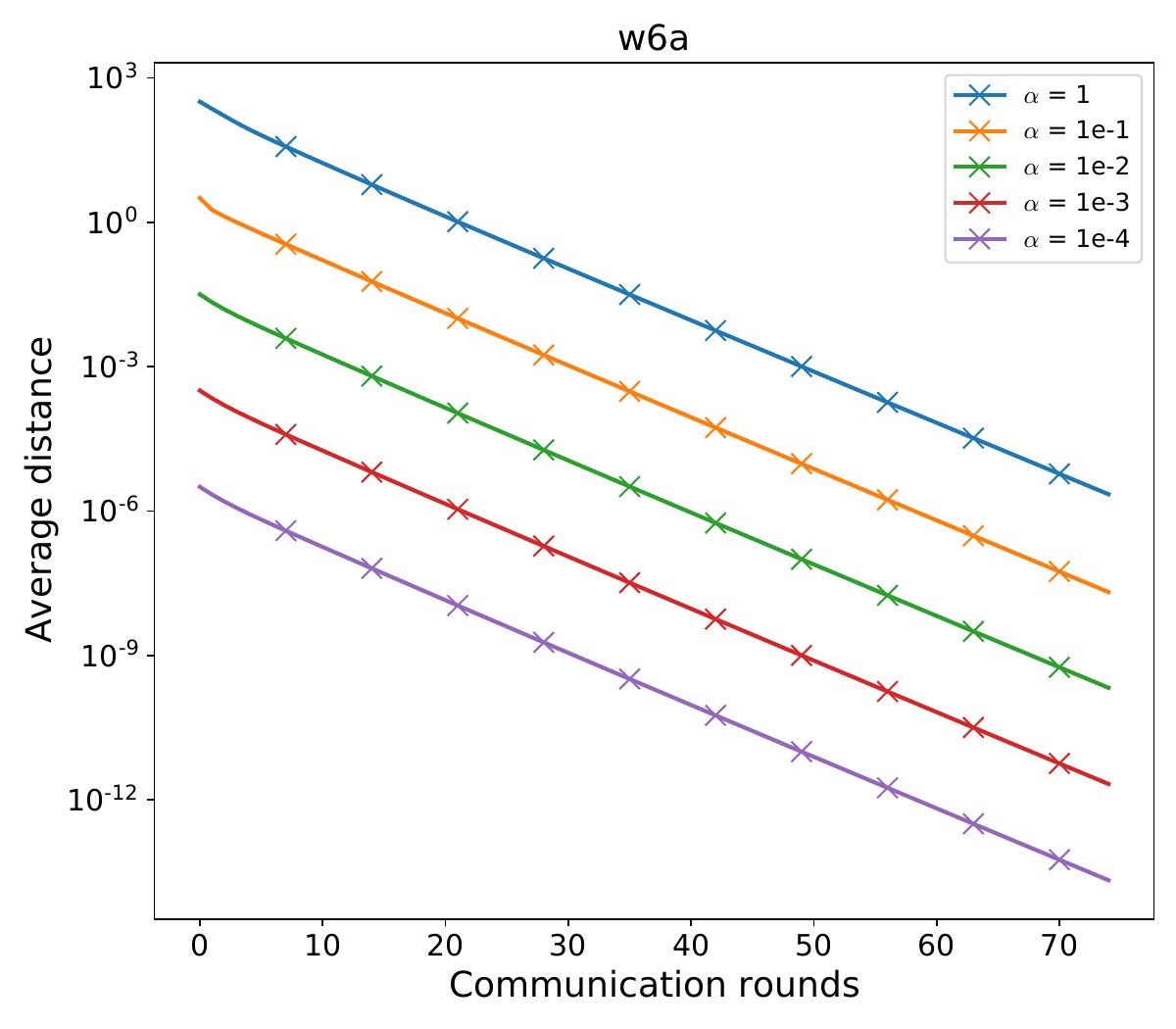}
\end{subfigure}
\begin{subfigure}{0.32\linewidth}
\includegraphics[width=\linewidth]{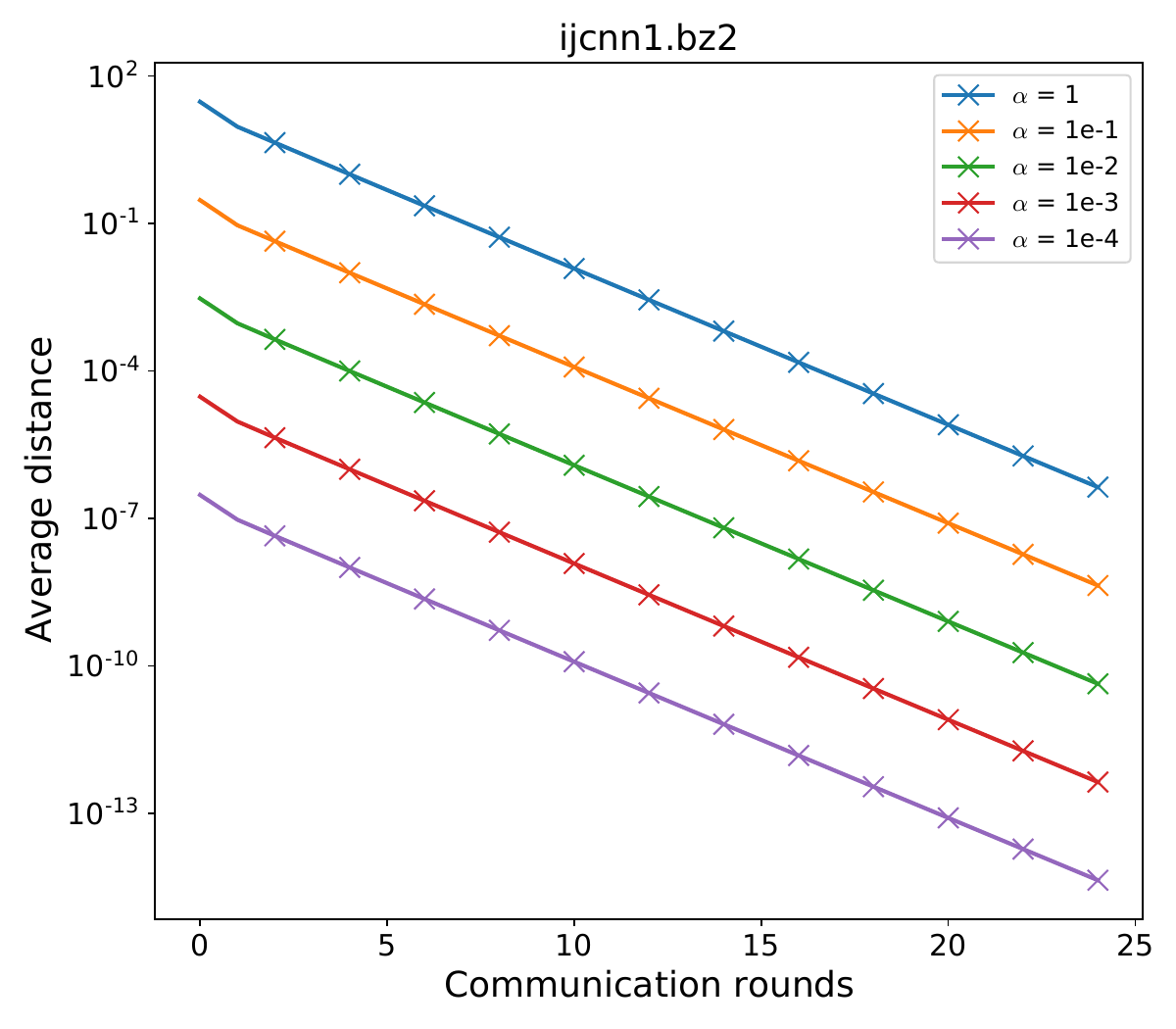}
\end{subfigure}
\begin{subfigure}{0.32\linewidth}
\includegraphics[width=\linewidth]{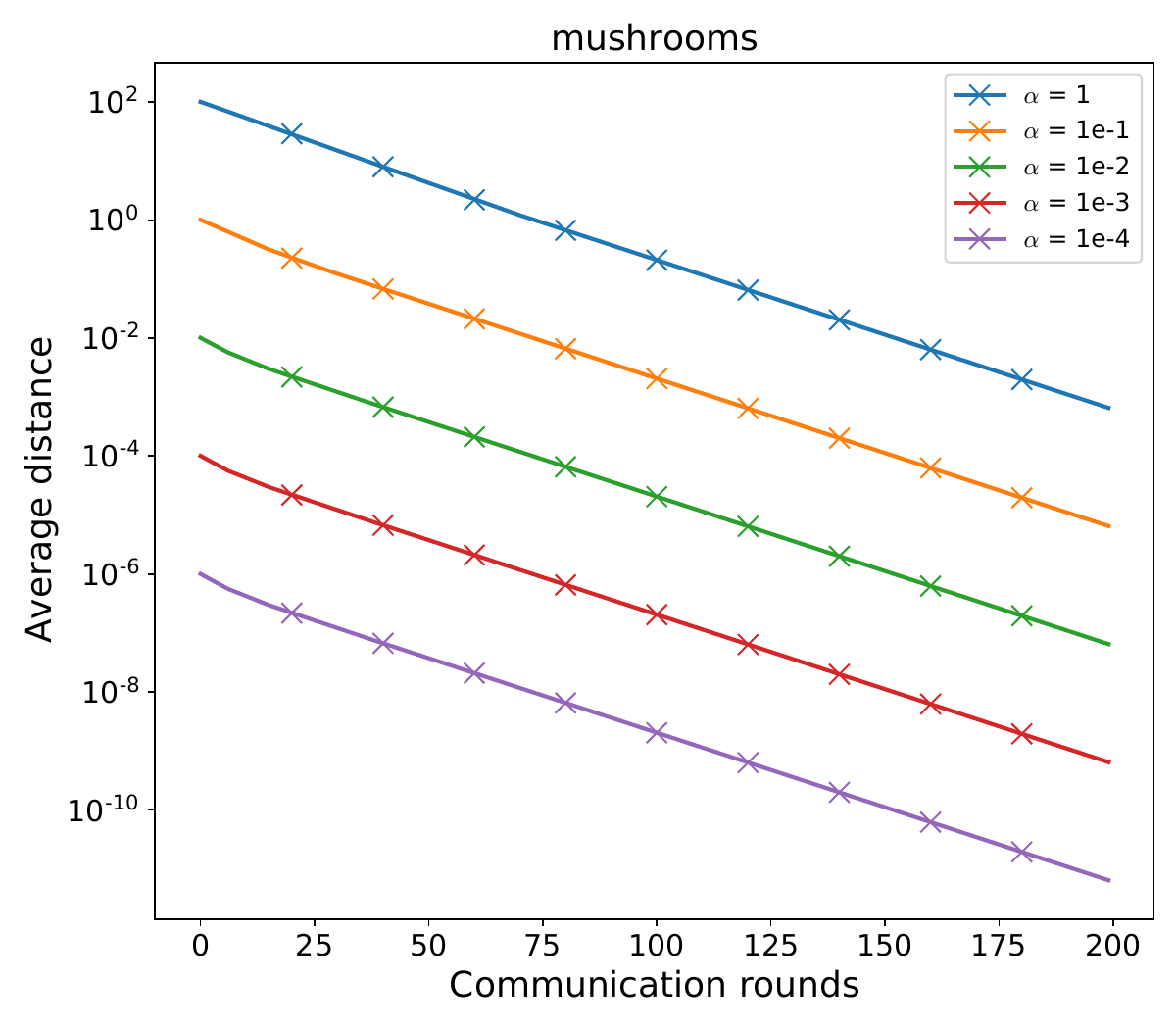}
\end{subfigure}
~\\
\begin{subfigure}{0.32\linewidth}
\includegraphics[width=\linewidth]{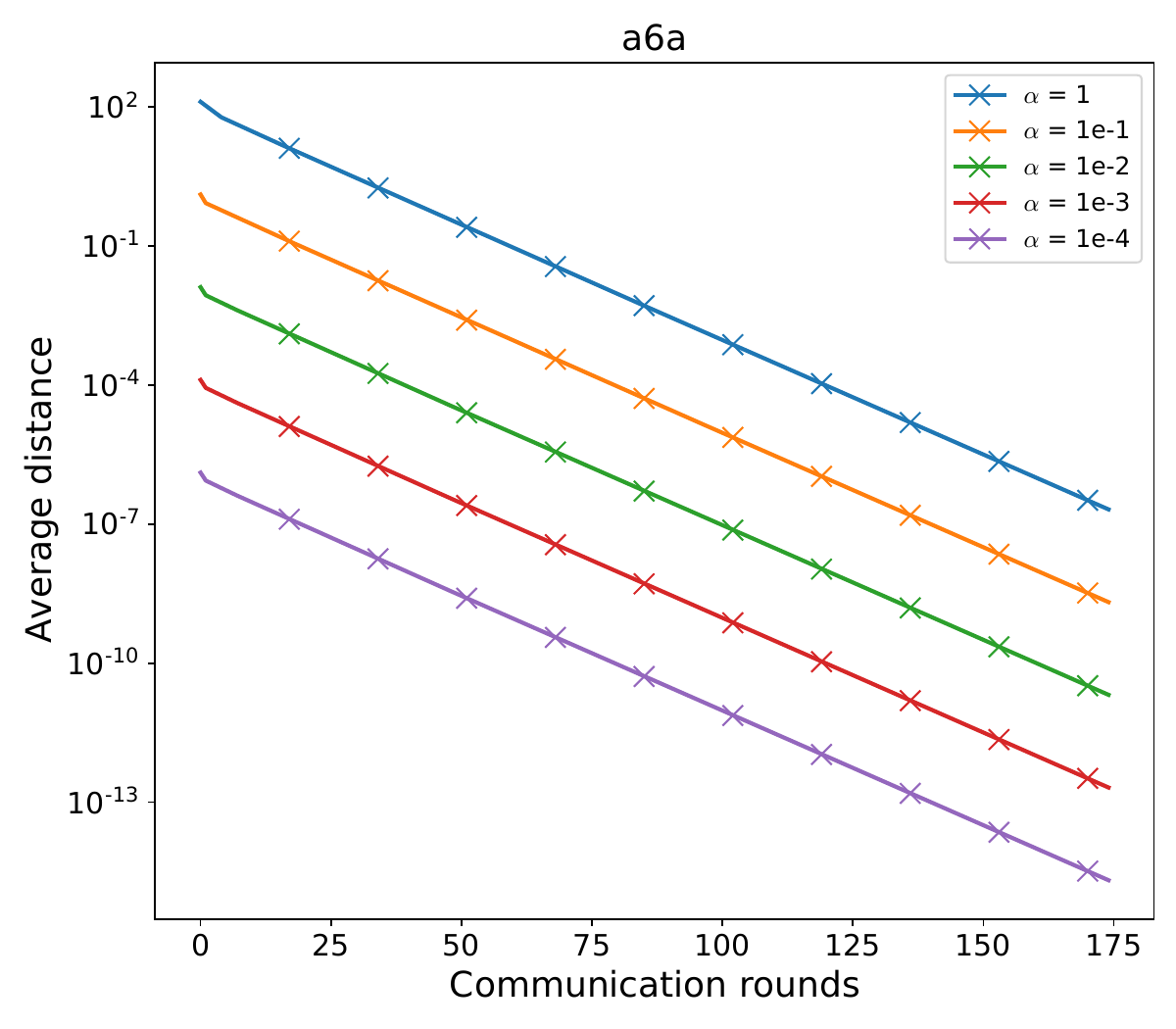}
\end{subfigure}
\begin{subfigure}{0.32\linewidth}
\includegraphics[width=\linewidth]{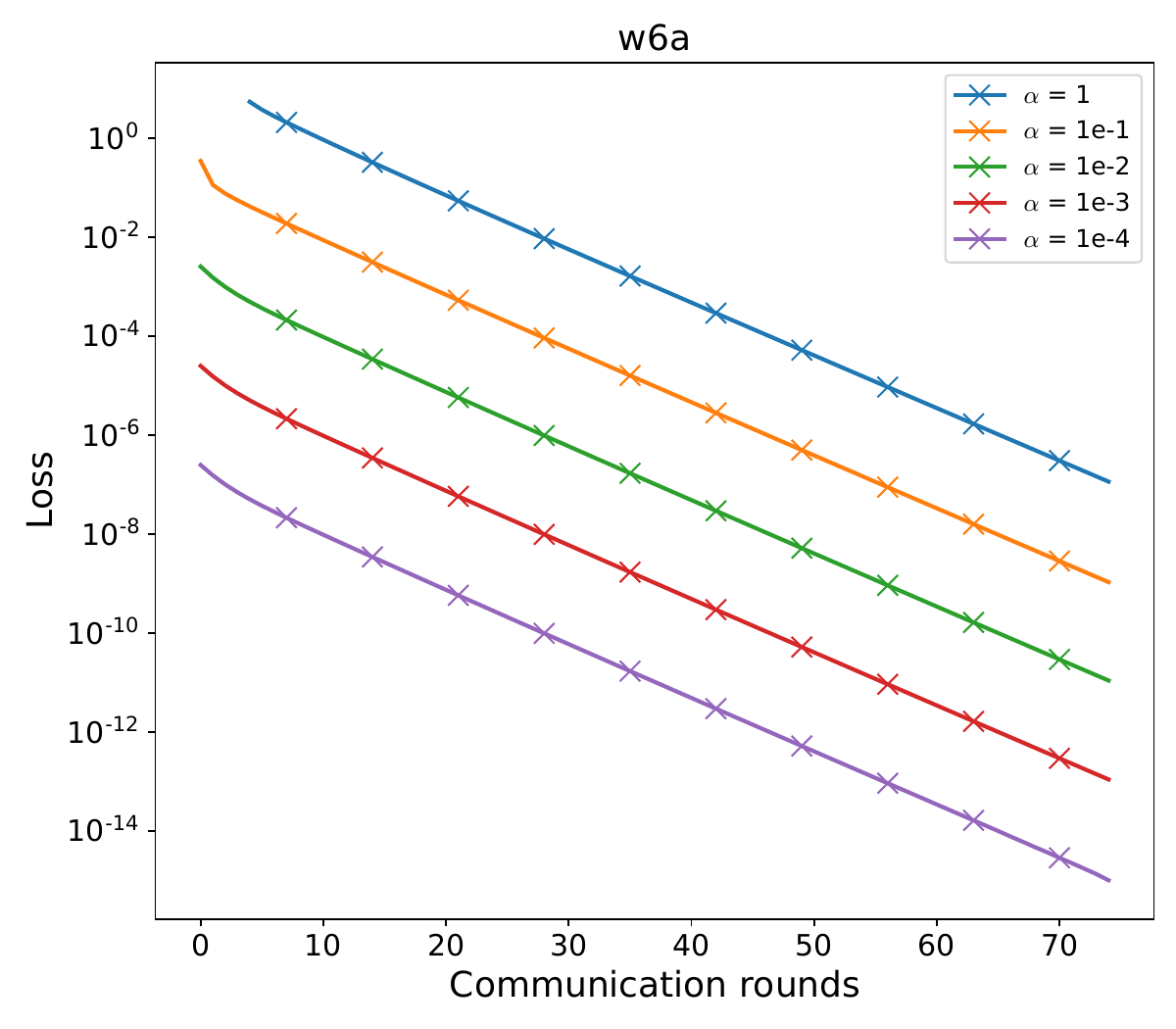}
\end{subfigure}
\begin{subfigure}{0.32\linewidth}
\includegraphics[width=\linewidth]{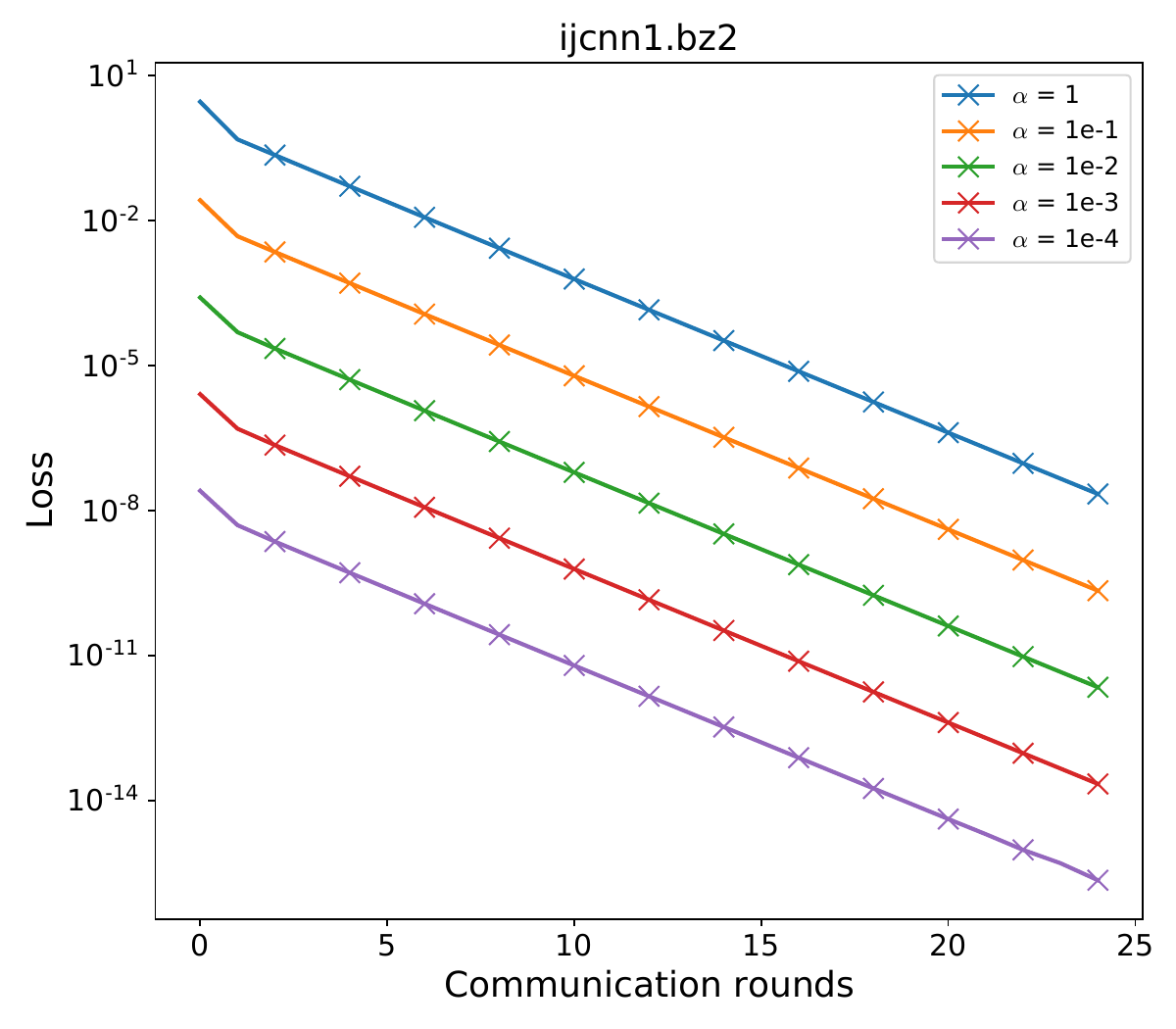}
\end{subfigure}
~\\
\centering
\begin{subfigure}{0.32\linewidth}
\includegraphics[width=\linewidth]{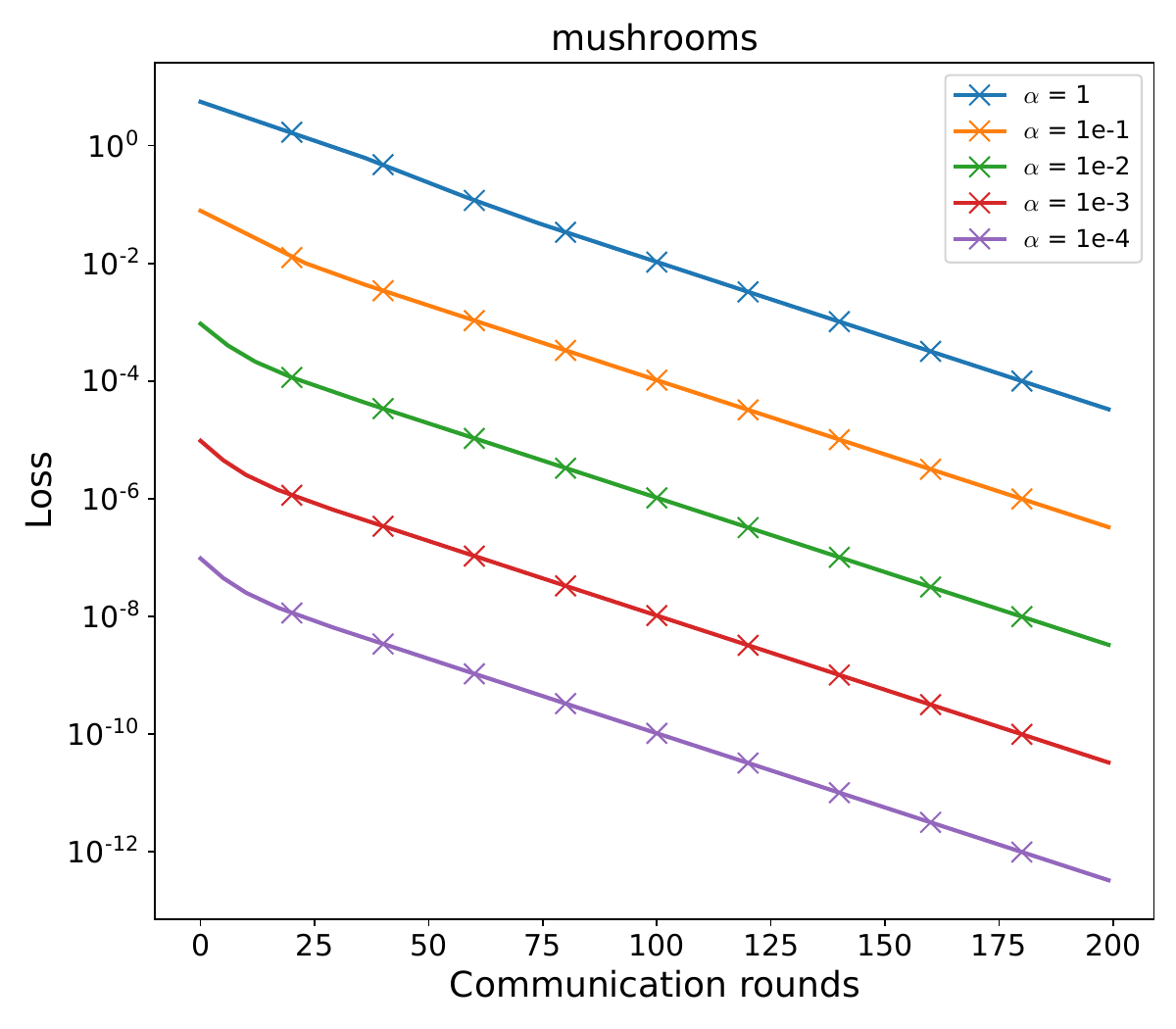}
\end{subfigure}
\begin{subfigure}{0.32\linewidth}
\includegraphics[width=\linewidth]{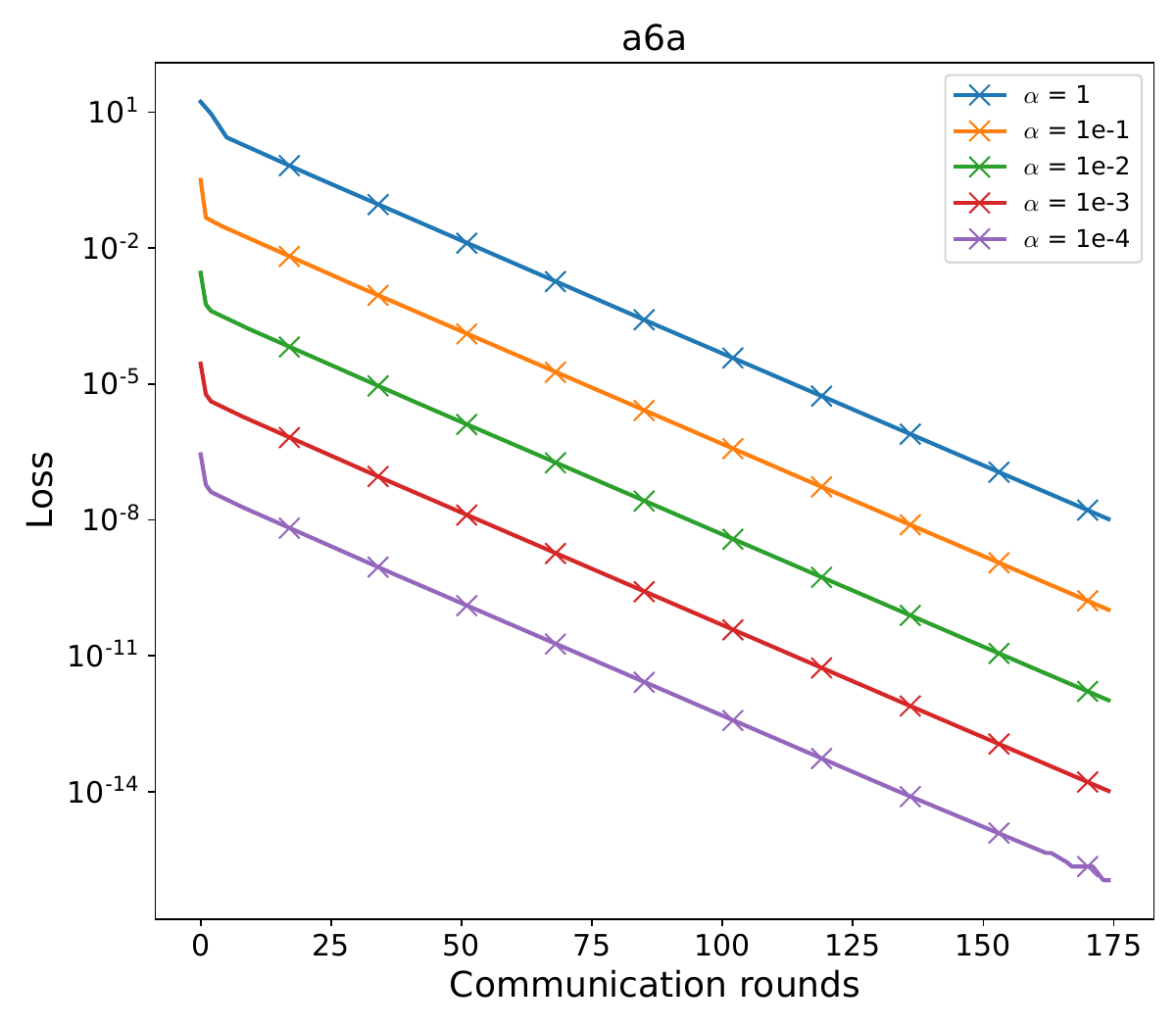}
\end{subfigure}
\caption{Squared averaged distance $\nicefrac{1}{n}\sum_{i=1}^n\|x_i - x_i^\ast\|^2$ and loss $f(x) - f^\ast $ vs. \# of communication rounds of Gradient Descent for logistic regression with $l_2$ regularizer. $\alpha_i = \alpha$ is set to the value indicated in the legend. }
\label{gd_plot}
\end{figure}

We use four datasets from LIBSVM~\citep{chang2011libsvm} for~\eqref{fnct:logreg_l2}: \texttt{w6a, mushrooms, ijcnn1.bz2, a6a}. We set all $\alpha_i$'s to be equal. We divide data equally between all machines while preserving the order of data points such that $i$-th machine owns data with indices $\lfloor \nicefrac{(i-1)r}{n}\rfloor  + 1$ up to $\lfloor \nicefrac{ir}{n}\rfloor$, where $r$ is the total number of data points. We set the regularization parameter $\lambda$ to $0.1$.
To find pure local models for each machine we run gradient descent with step size $\nicefrac{1}{L_i}$, where $L_i$ is $f_i$'s smoothness parameter until the norm of the gradient is below $10^{-6}$. We set the condition number to be $\kappa = \frac{1}{n} \sum_{i=1}^{n} L_i/\lambda$.

In this experiment, we investigate the convergence of gradient descent and look into the dependence between convergence and value of $\alpha$. As expected, Figure~\ref{gd_plot} confirms that smaller values of $\alpha$ lead to better convergence as we rely more on local solutions and thus start closer to the optimal solution. Also note that the speed of the convergence appears to be constant among different values of $\alpha$ which is also predicted by our theory as the same $\alpha$ on each machine does not affect the conditioning of the global problem, see Proposition~\ref{proposition:preserve-smooth-and-convexity}. We also use the DIANA algorithm~\citep{diana_paper} with the random sparsification (also known as Rand-$k$ compression), where set $k$ coordinates to zero at random before communicating gradients to the server. Figure~\ref{diana_log_reg_omega} shows the effect of varying $k$ on the convergence of the method in terms of communication rounds. We consider 7 values of $k$ linearly spaced between $1$ and $d$. Similar to Figure~\ref{gd_plot}, we observe that smaller values of $\alpha$ lead to better convergence. The rate at which the algorithm converges linearly is controlled by the compression constant $\omega +1 = \nicefrac{d}{k}$ or the effective conditioning $\nicefrac{\kappa (\omega + 1)}{n}$. This is in line with the theory for DIANA~\citep{diana_paper}.

\begin{figure*}[h]
\includegraphics[width=\linewidth]{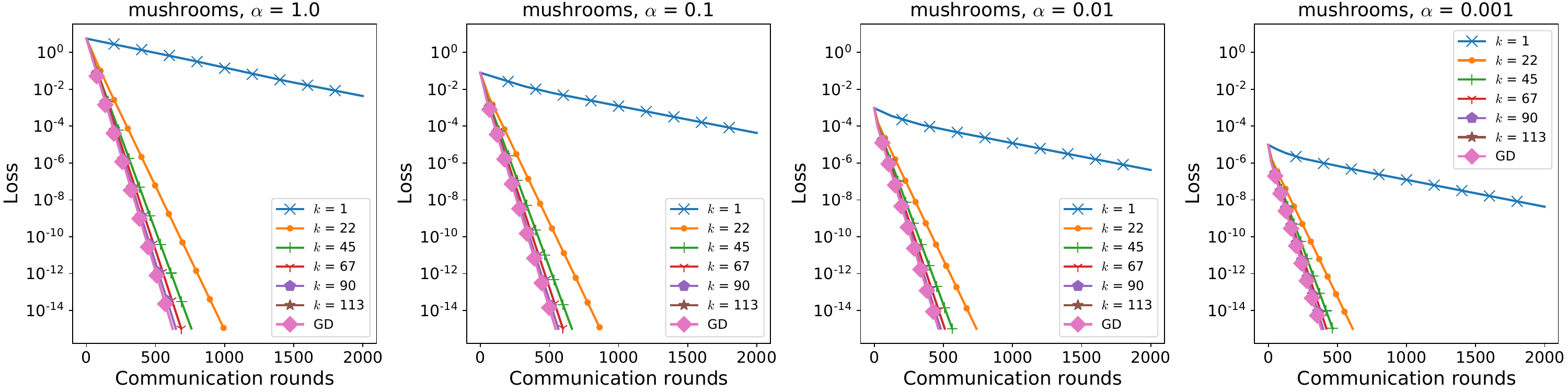}
\caption{Loss $f(x) - f^\ast$ vs. \# of communication rounds of DIANA for logistic regression problem $l_2$ regularizer, $k$  is a sparsification parameter of Random-$k$ compressor.}
\label{diana_log_reg_omega}
\end{figure*}

\begin{figure*}[ht!]
	\centering
	\begin{subfigure}{0.245\linewidth}
		\includegraphics[width=\linewidth]{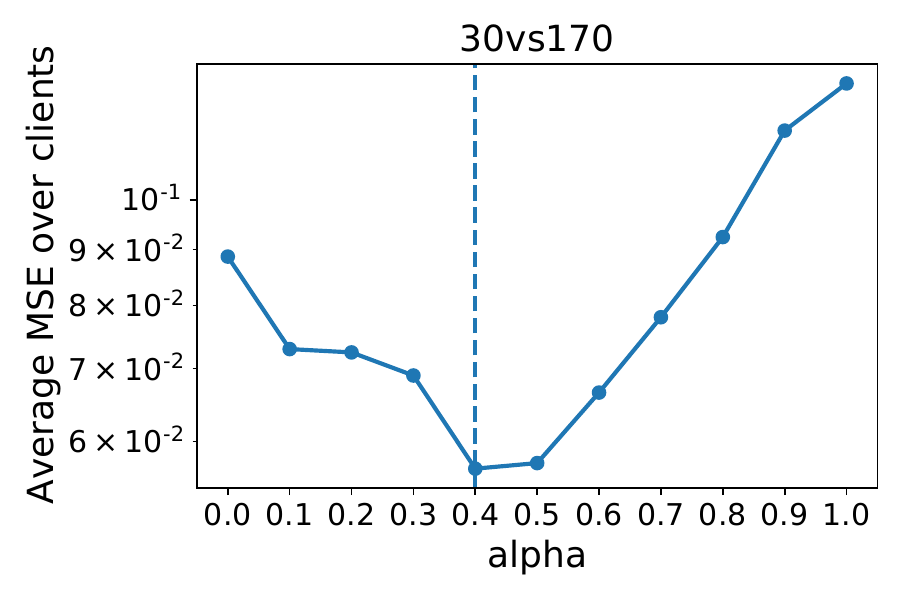}
	\end{subfigure}
	\begin{subfigure}{0.245\linewidth}
		\includegraphics[width=\linewidth]{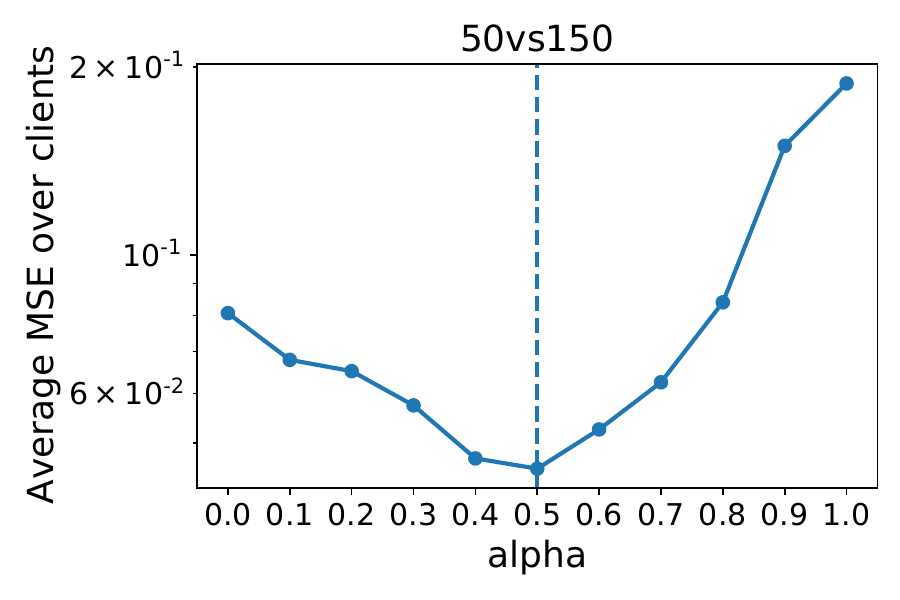}
	\end{subfigure}
	\begin{subfigure}{0.245\linewidth}
		\includegraphics[width=\linewidth]{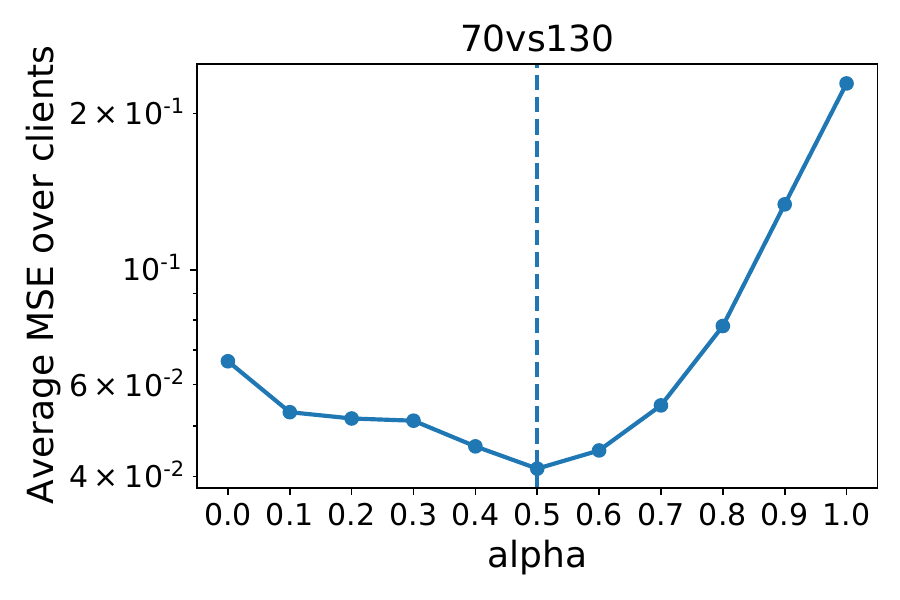}
	\end{subfigure}
	\begin{subfigure}{0.245\linewidth}
		\includegraphics[width=\linewidth]{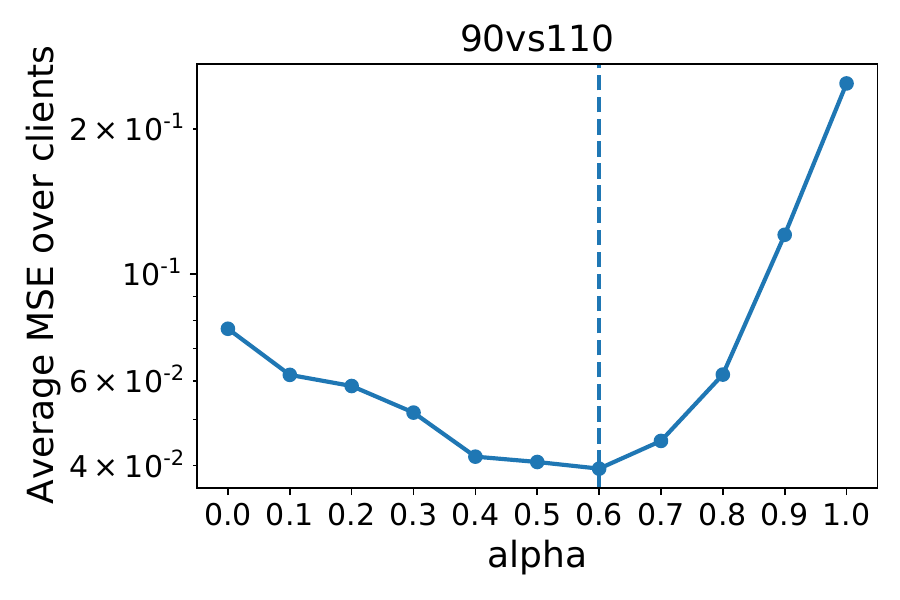}
	\end{subfigure}
	~\\
	\begin{subfigure}{0.245\linewidth}
		\includegraphics[width=\linewidth]{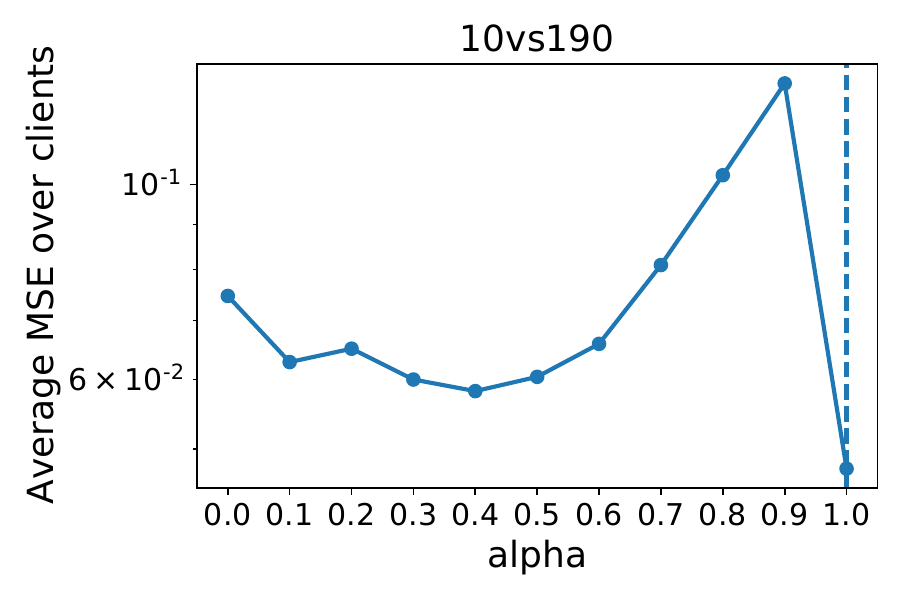}
	\end{subfigure}
	\begin{subfigure}{0.245\linewidth}
		\includegraphics[width=\linewidth]{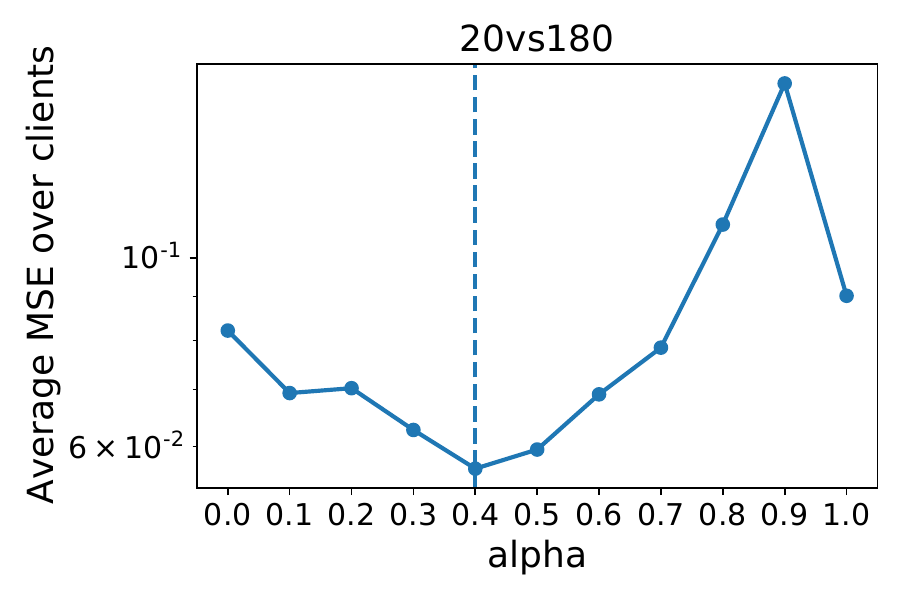}
	\end{subfigure}
	\begin{subfigure}{0.245\linewidth}
		\includegraphics[width=\linewidth]{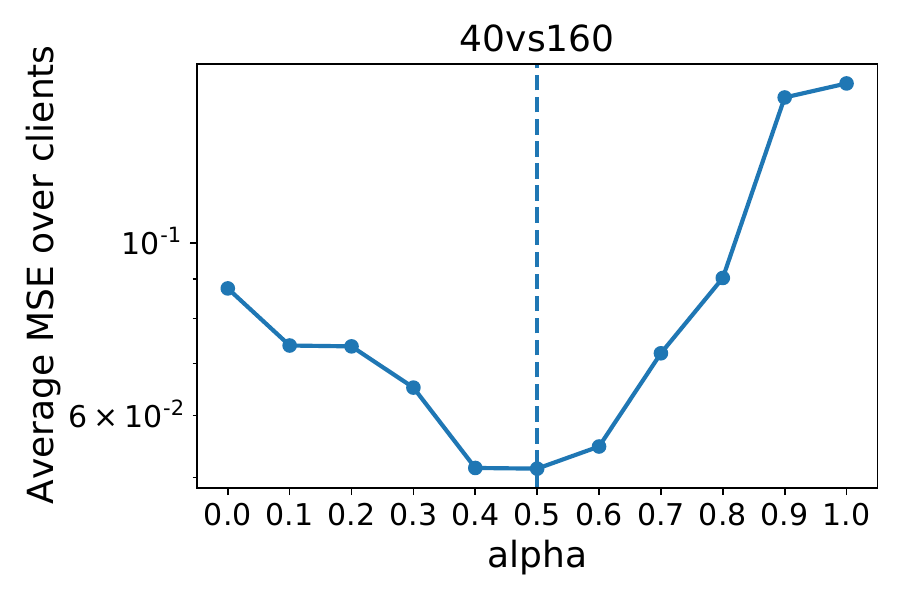}
	\end{subfigure}
	~\\
	\begin{subfigure}{0.245\linewidth}
		\includegraphics[width=\linewidth]{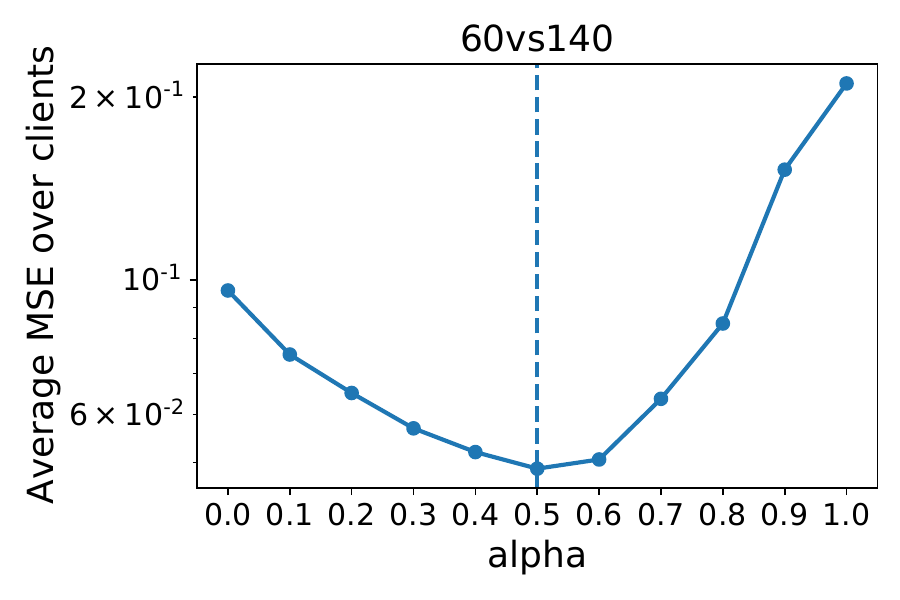}
	\end{subfigure}
	\begin{subfigure}{0.245\linewidth}
		\includegraphics[width=\linewidth]{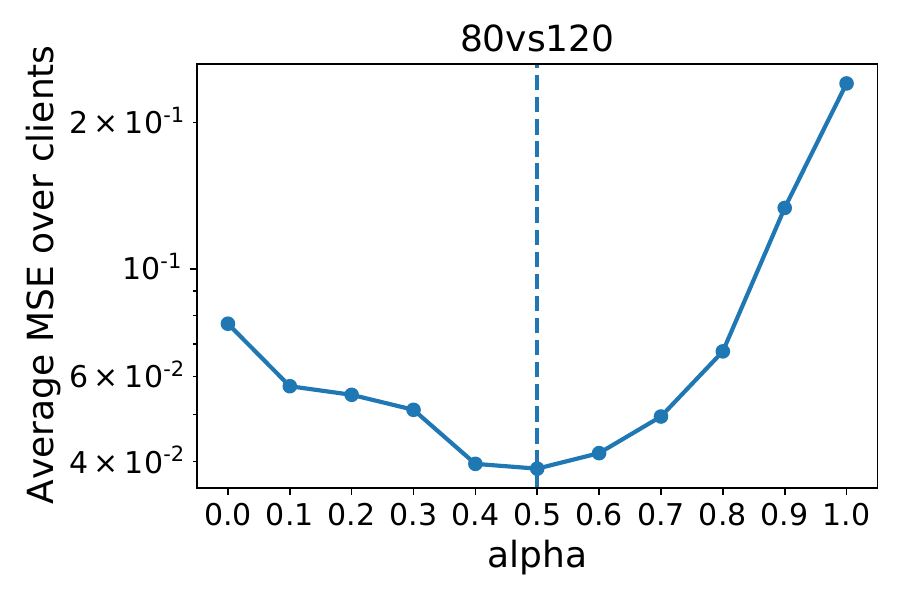}
	\end{subfigure}
	\begin{subfigure}{0.245\linewidth}
		\includegraphics[width=\linewidth]{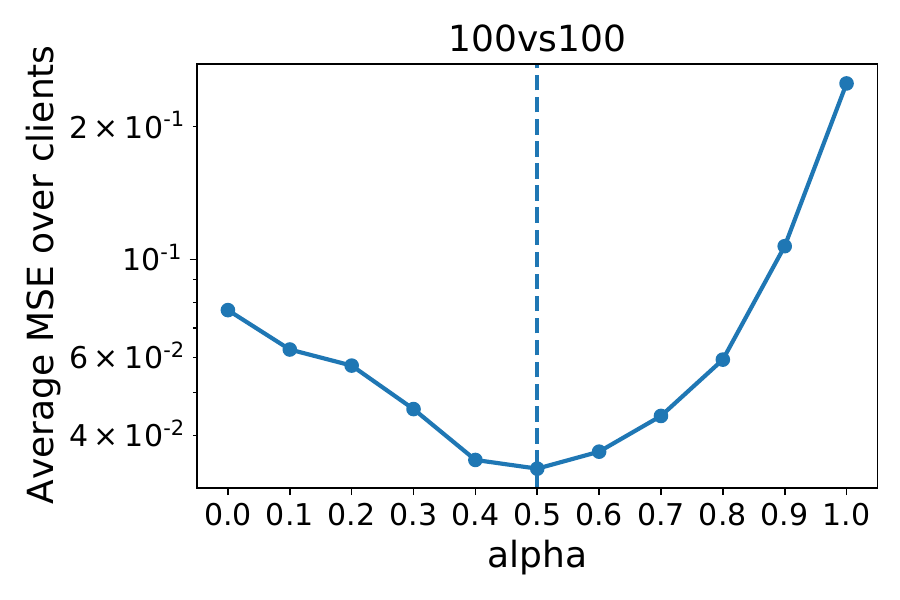}
	\end{subfigure}
	\caption{Average MSE vs. personalization parameter $\alpha$. }
	\label{two_sine_func_values}
\end{figure*}

\begin{figure*}[ht]
	\centering
	\begin{subfigure}{0.45\linewidth}
		\includegraphics[width=\linewidth]{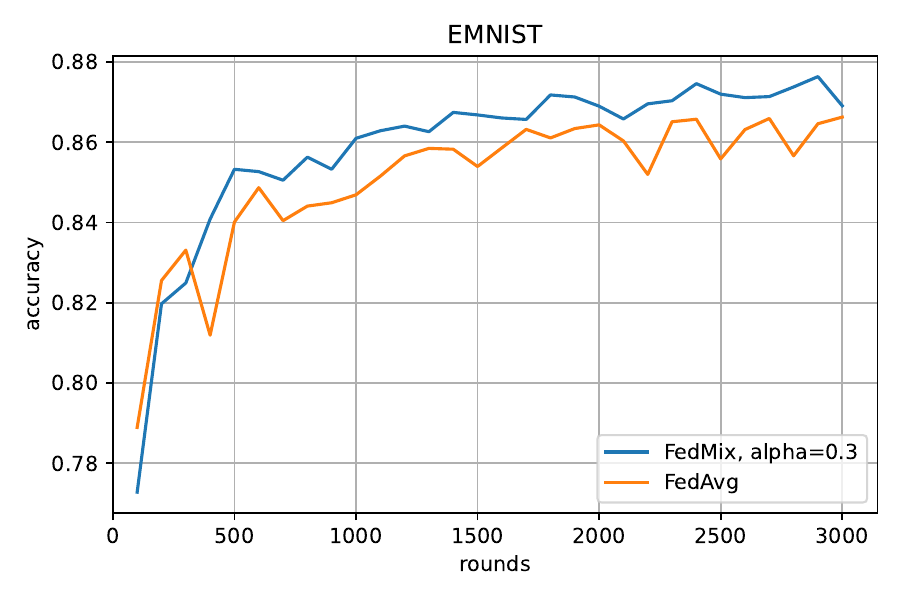}
		\caption{EMNIST}
	\end{subfigure}
	\begin{subfigure}{0.45\linewidth}
		\includegraphics[width=\linewidth]{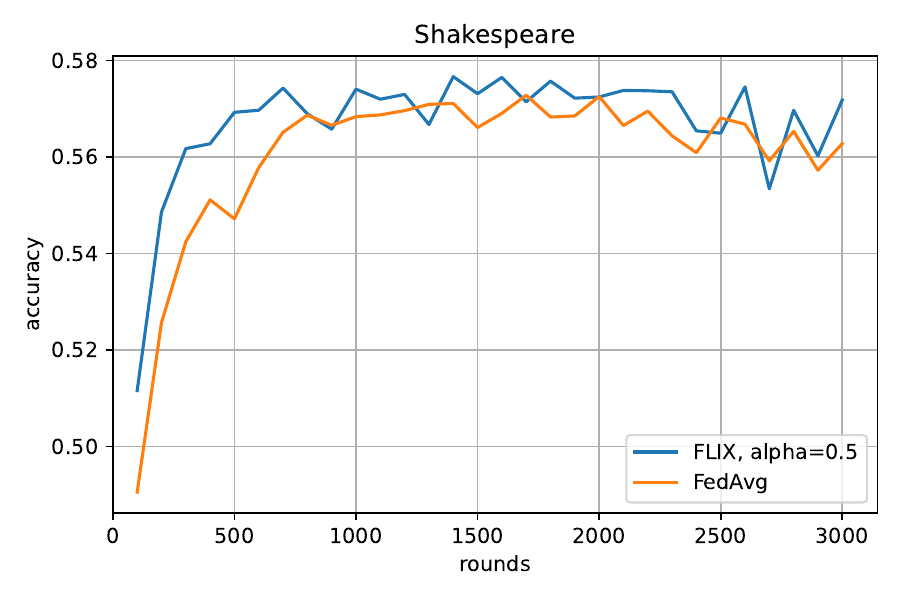}
		\caption{Shakespeare}
	\end{subfigure}
	\caption{Generalization progress of FLIX vs. FedAvg on high-scale datasets.}
	\label{fig:high_scale_progress}
\end{figure*}

\begin{table}[h]
	\caption{Test accuracy comparisons of different methods after 3k rounds of training on EMNIST and Shakespeare datasets. Numbers next to FLIX model denote the value of $\alpha_i = \alpha$.}\label{table:flix_high_scale}
	\begin{center}
		\begin{tabular}{lll}
			\textbf{MODEL} &\textbf{EMNIST} &\textbf{SHAKESPEARE}\\
		    \hline \\
			FLIX, 0.1 & 0.862 &  0.5620\\
			FLIX, 0.3 & {\bfseries 0.8691} & 0.5291\\
			FLIX, 0.5 & 0.8634 & {\bfseries 0.5718}\\
			FLIX, 0.7 & 0.8652 & 0.5146\\
			FLIX, 0.9 & 0.8397 & \\
			FedAvg & 0.8663 &  0.5629\\
		\end{tabular}
	\end{center}
\end{table}

For completeness, we include convex experiments with unregularized logistic loss, extra experiments for all the combinations of 5 datasets and 3 algorithms -- gradient descent, compressed gradient descent, and DIANA, as well as a detailed experimental description in the supplementary material (see \Cref{sec:extra-experiments}).

{\bfseries Generalization experiment 1: Fitting Sine Functions with two-layer neural networks.} Following~\citet{finn17_model_agnos_meta_learn_fast} and~\citet{minprox}, we show the generalization advantages of the proposed method on the following regression problem. We define $i$-th client's function $f_i(x) = a_i \sin{(x + b_i)}$, where amplitude $a_i$ and phase $b_i$ lie in the intervals $[0.1, 0.5]$ and $[0, 2\pi]$, respectively. For each client, we fix $a_i$ and $b_i$ and sample $50$ points uniformly at random from the interval $[-5.0, 5.0]$. We measure regression fit in terms of mean squared error (MSE) loss.  To train a model each client adopts a neural net with 2 hidden layers of size 40 with tanh activation. Further technical details are deferred to the appendix. For the experiment, we first sample 2 pairs $\{a_i, b_i\}$ and each of $200$ clients is assigned one pair, we investigate different proportion-- $(30, 170), (50, 150),  (70, 130), (90, 110)$. We then train our FLIX formulation with $\alpha_i = \alpha = 0.1, 0.2, \hdots, 1$. For testing for each client generates a new dataset of size 2000. Figure~\ref{two_sine_func_values} shows average MSE over clients against different values of $\alpha$ for different proportions. As this figure indicates, optimal $\alpha$ for which test average MSE is minimal can dramatically outperform the edge cases of either global model for all tasks or personalized model trained only on the local dataset. This underlines the main benefit of FLIX formulation, which aims to find an optimal trade-off between local personalization and a single global model.

{\bfseries Generalization experiment 2: Comparison to FOMAML and Reptile with mutli-class logistic regression.} Inspired by~\cite{reddi2021adaptive}, we conduct a similar experiment to compare generalization capabilities, i.e.,  test accuracy, of FLIX and its two  baselines FOMAML~\citep{finn17_model_agnos_meta_learn_fast}, and Reptile~\citep{nichol18_first_order_meta_learn_algor} on the following tasks. For the first experiment (see Figure~\ref{subfig:100w_2c}), we take 500 train data points of two clients (with client ids `00000267' and `00000459') from the Stack Overflow dataset~\citep{stackoverflow_dataset} and divide them among 100 workers so that there are 50 workers with 10 train data points from the first client and another 50 with 10 data points from the second client. For the second experiment (see Figure~\ref{subfug:50w_50c}), a worker gets 90 train data points from a distinct client. For both experiments, each objective component  $f_i$ is a cross-entropy loss for multi-class logistic regression. Further technical details and the hyperparameters tuning for a fair comparison can be found in the supplementary. In the test phase, for each client, we used a hold-out testing dataset of size 300 (the same dataset has been used for workers related to the same client in the first experiment). It can be observed from Figure~\ref{subfig:100w_2c}, that for wide range of $\alpha_i$, $\alpha_i \in \{0.2, 0.4, 0.6, 0.8 \}$ FLIX exhibits a better generalization than its classical meta-learning competitors (FOMAML and Reptile), and it can lead to improvement of up to $11\%$ in recall@5. Figure~\ref{subfug:50w_50c} shows that in the more real-world scenario FLIX outperforms FOMAML and Reptile while showing its best test accuracy in non-edge $\alpha$. More experimental details and figures are included in the supplementary (\Cref{sec:extra-experiments}).

{\bfseries Generalization experiment 3: Comparison to FedAvg with CNNs and LSTMs on large-scale datasets.} For this section, we implement all the methods using the Python FedJAX library~\citep{fedjax2020github}. We compare FLIX's solution versus pure global obtained by FedAvg on two large-scale FL datasets: the EMNIST~\citep{leaf} and Shakespeare~\citep{mcmahan54} datasets, solving a character prediction task on the first dataset and a next-character prediction task on the second dataset. For the first task, we adopt the EMNIST CNN model as described in~\citep{reddi2021adaptive}. To choose the hyperparameters for FLIX, we run grid search on a five-dimensional grid: pure local model (PLM) batch sizes [$10, 20, 50, 100$], pure local model learning rates [$10^{-5}, 10^{-4}, 10^{-3}, 10^{-2}, 10^{-1}, 1$], FLIX batch sizes [$20, 50, 100, 200$], and FLIX learning rates [$10^{-10}, 10^{-9}, \dots, 1$]. We run SGD for 100 local epochs to compute PLMs for each client; with PLMs computed, we run Adam to compute FLIX global model. At each round of training in FLIX, we apply client subsampling and randomly choose ten clients to participate in the training. For the final training, we choose the five hyperparameters showing the highest accuracy on a hold-out validation (different from the test) dataset (cut from train EMNIST dataset). For training FedAvg, we use the same hyperparameters and setting used by~\citep{reddi2021adaptive}. For the second task, we use the next-character prediction model described in~\citep{mcmahan54}. We choose hyperparameters as follows: for the pure local model learning rates, we search over the grid [$10^{-1}, 10^{-0.5}, 10^0, 10^{0.5}, 10^1$], for FLIX learning rates, we search over the grid [$10^{-2}, 10^{-1.5}, \dots, 10^2$], for the batch sizes for PLM and FLIX we search over the grid [$1, 4, 10, 20$]. To compute PLMs, we run SGD for 25 epochs on each client locally. We then use SGD to compute the FLIX global model. At each round, only ten clients were chosen uniformly at random to participate in the training. We report test accuracy of the final models in~\Cref{table:flix_high_scale} and show the generalization progress for the best alphas in~\Cref{fig:high_scale_progress}. As~\Cref{table:flix_high_scale} shows the generalization capability of FLIX is superior to the purely global model obtained by FedAvg on both datasets.

\section{ACKNOWLEDGMENTS}

We would like to thank David Pugh for his support in configuring experiments on the KAUST HPC cluster.

\bibliographystyle{plainnat}
\bibliography{FL-mixture.bib}

\clearpage
\onecolumn
\part*{APPENDIX}
\tableofcontents

\clearpage
\section{BASIC FACTS AND NOTATION} \label{sec:basic_facts}

We will use the triangle inequality for norms:
\begin{equation}
  \label{eq:triangle-inequality}
  \norm{\sum_{i=1}^{n} a_i} \leq \sum_{i=1}^n \norm{a_i}.
\end{equation}
We say that a function $g$ is $L_g$-smooth if for any $x, y \in \R^d$ we have
\begin{equation}
  \label{eq:smoothness-def}
  \norm{g(x) - g(y)} \leq L_g \norm{x - y}.
\end{equation}
Note that \eqref{eq:smoothness-def} implies
\begin{equation}
  \label{eq:smoothness-func-consequence}
  g(x) \leq g(y) + \ev{\nabla g(y), x-y} + \frac{L_g}{2} \sqn{x-y}.
\end{equation}
We say that a function $g$ is $\mu_g$-strongly convex for $\mu_g > 0$ if for all $x, y \in \R^d$ we have
\begin{equation}
  \label{eq:strong-conv-def}
  g(x) \geq g(y) + \ev{\nabla g(y), x-y} + \frac{\mu_g}{2} \sqn{x-y}.
\end{equation}
If \eqref{eq:strong-conv-def} holds with $\mu_g = 0$, we say that $g$ is convex.

We say $\cC \in \mathbb{B}^d(\omega)$ is a compression operator if $\cC$ is unbiased (i.e., $\EE [\cC(x)] = x$ for all $x \in \RR^d$) and if the second moment is bounded as
\begin{align}\label{ineq:cmp_oper_variance}
\EE{ \|\cC(x) - x\|^2 } \leq \omega \|x\|^2 \ \forall x\in \R^d.
\end{align}

Note that if $\cC \in \mathbb{B}^d(\omega)$, then
\begin{align}
\label{eq:comress_cor}
\EE{\|\cC(x)\|^2} \leq (1 + \omega) \|x\|^2, \forall x\in \R^d.
\end{align}

Let $f$ be a convex function. Then the Bregman divergence associated with $f$ for points $x, y \in \RR^d$ is defined in the following way:
$$
D_f(x, y) = f(x) - f(y) - \langle \nabla f(y), x - y \rangle.
$$

\cite{Nesterov2018} shows that if function $f$ is convex and $L$-smooth, then for all $x$ and $y$
\begin{align}
\label{ineq:cvx_smooth}
\|\nabla f(x) - \nabla f(y)\|^2 \leq 2 L D_f(x, y).
\end{align}

Let $a, b \in \RR^n$ be arbitrary vectors. Then, it holds that

\begin{align}
\label{ineq:a_b}
\|a + b\|^2 \leq 2 (\|a\|^2 + \|b\|^2).
\end{align}

\clearpage
\section{PROOFS FOR SECTIONS~\ref{sec:FLIX_formulation} AND~\ref{sec:algorithms-for-FLIX}}
This section collects the proofs of all propositions and theorems mentioned in the paper.

\subsection{Proof of \Cref{proposition:quad-fine-tuned-sol}}
In \Cref{sec:motivation-2} we provided the motivation for FLIX through the lens of fine-tuning. We used the following proposition on fine-tuning quadratics:

\begin{customprop}{\ref{proposition:quad-fine-tuned-sol}}
  Suppose that we run gradient descent for $H$ steps on the quadratic objective $f_i = \frac{1}{2} x^T A_i x - b_i^T x + c$ starting from $x^{0}$ with stepsize $\gamma > 0$. Suppose that the stepsize satisfies $\gamma \leq \frac{1}{L_i}$, where $L_i = \lambda_{\max} (A_i)$, and suppose that $A_i$ is positive definite. Then the final iterate $x_i^{H}$ can be written as
  \begin{equation}
    \label{eq:prop-quad-claim}
    x_i^H = \br{I - J_i^H} x_i + J_i^H x^0,
  \end{equation}
  where $x_i$ minimizes $f_i$ and $J_i \in \R^{d \times d}$ is a matrix with maximum eigenvalue smaller than $1$, i.e.\ $\lambda_{\max} (J) < 1$.
\end{customprop}
\begin{proof}
  The gradient descent update is
  \begin{align}
    x_{i}^{t+1} &= x_i^t - \gamma \nabla f(x_i^t) \nonumber \\
                &= x_i^t - \gamma (A_i x_i^t - b_i) \nonumber \\
                &= (I - \gamma A_i) x_i^t + \gamma b_i. \label{eq:prop-quad-1}
  \end{align}
  Note that because $x_i$ minimizes $f_i$, we have $\nabla f_i (x_i) = 0$ by first-order optimality. Hence,
  \begin{align*}
    \nabla f_{i} (x_i) = 0 \Longleftrightarrow A_i x_i = b_i.
  \end{align*}
  Using this in~\eqref{eq:prop-quad-1},
  \begin{align*}
    x_i^{t+1} &= (I - \gamma A_i) x_i^t + \gamma A_i x_i
  \end{align*}
  Subtracting $x_i$ from both sides,
  \begin{align*}
    x_i^{t+1} - x_i &= (I - \gamma A_i) x_i^t + (\gamma A_i - I) x_i \\
    &= (I - \gamma A_i) (x_i^t - x_i).
  \end{align*}
  Iterating the above equality for $H$ steps we get
  \begin{equation*}
    x_i^{H} - x_i = (I - \gamma A_i)^H (x^0 - x_i).
  \end{equation*}
  Rearranging the terms we get \eqref{eq:prop-quad-claim} with $J_i \eqdef (I - \gamma A_i)$. Observe that when $\gamma \leq \frac{1}{\lambda_{\max} (A_i)}$ and $\lambda_{\min} (A_i) > 0$ we have that $\lambda_{\max} (I - \gamma A_i) < 1$.
\end{proof}

\subsection{Proofs for algorithm-independent results}
\subsubsection{Proof of Proposition~\ref{proposition:preserve-smooth-and-convexity}}
\begin{customprop}{\ref{proposition:preserve-smooth-and-convexity}}
  Suppose that each objective $f_i$ is $L_i$-smooth. That is, for any $x, y \in \R^d$ we have
  \begin{equation}
    \label{eq:fi-smoothness}
    \norm{\nabla f_i (x) - \nabla f_i (y)} \leq L_i \norm{x-y}.
  \end{equation}
  Then the FLIX objective $\tilde{f}$ defined in \eqref{eq:FLIX-objective} is $L_\alpha$-smooth for $L_\alpha \eqdef \frac{1}{n} \sum_{i=1}^{n} \alpha_i^2 L_i$. If each $f_i$ is convex, then $\tilde{f}$ is also convex. If each $f_i$ is $\mu_i$-strongly convex, then $\tilde{f}$ is $\mu_\alpha$ strongly convex for $\mu_\alpha \eqdef \frac{1}{n} \sum_{i=1}^{n} \alpha_i^2 \mu_i$.
\end{customprop}
\begin{proof}
  We separate the proofs in two cases: when each $f_i$ is smooth, and when each $f_i$ is (strongly) convex.
  \begin{itemize}[leftmargin=0.15in,itemsep=0.01in,topsep=0pt]
  \item[(i)] Suppose that each $f_i$ is $L_i$-smooth, and let $x, y \in \R^d$. Then by direct computation we have
    \begin{align*}
      \norm{\nabla \tilde{f} (x) - \nabla \tilde{f} (y)} &= \norm{\frac{1}{n} \sum_{i=1}^{n} \alpha_i \left [ \nabla f_i (\alpha_i x + (1-\alpha_i) x_i) - \nabla f_i (\alpha_i y + (1-\alpha_i) y_i) \right ] } \\
                                                         &\overset{\eqref{eq:triangle-inequality}}{\leq} \frac{1}{n} \sum_{i=1}^{n} \norm{\alpha_i} \norm{\nabla f_i (\alpha_i x + \br{1-\alpha_i} x_i) - \nabla f_i (\alpha_i y + (1-\alpha_i) y_i)} \\
                                                         &\overset{\eqref{eq:fi-smoothness}}{\leq} \frac{1}{n} \sum_{i=1}^{n} \alpha_i^2 L_i \norm{x-y} = L_{\alpha} \norm{x-y},
    \end{align*}
    hence $\tilde{f}$ is $L_\alpha$-smooth.
  \item[(ii)] Suppose that each $f_i$ is $\mu_i$-strongly convex for $\mu_i \geq 0$ (where $\mu_i = 0$ corresponds to just convexity). Then for $x, h \in \R^d$ we have
    \begin{align*}
      \tilde{f} (x+h) &= \frac{1}{n} \sum_{i=1}^{n} f_i (\alpha_i (x+h) + (1-\alpha_i) x_i) \\
                      &\overset{\eqref{eq:strong-conv-def}}{\geq} \frac{1}{n} \sum_{i=1}^{n} \left [ f_i (\alpha_i x + (1-\alpha_i) x_i) + \ev{\nabla f_i (\alpha_i x + (1-\alpha_i) x_i), \alpha_i h} + \frac{\mu_i }{2} \sqn{\alpha_i h}  \right ] \\
                      \begin{split}
                        &=\frac{1}{n} \sum_{i=1}^{n} f_i (\alpha_i x + (1-\alpha_i) x_i) + \ev{\frac{1}{n} \sum_{i=1}^{n} \alpha_i \nabla f_i (\alpha_i x + (1-\alpha_i) x_i), h } \\
                        &\qquad\qquad + \frac{1}{n} \sum_{i=1}^{n} \frac{\alpha_i^2 \mu_i}{2} \norm{h}.
                      \end{split} \\
      &= \tilde{f} (x) + \ev{\nabla \tilde{f} (x), h} + \frac{\mu_\alpha}{2} \sqn{h},
    \end{align*}
    hence $\tilde{f}$ is $\mu_\alpha$-strongly convex if all the $\mu_i$ are positive (resp. convex if they are equal to $0$).
  \end{itemize}
\end{proof}

\subsubsection{A proposition for bounding the gradient norm}
\Cref{proposition:preserve-smooth-and-convexity} provides us with a simple way to bound the gradient of $\tilde{f}$ in terms of the distance to the optimum $x^\alpha$:

\begin{proposition}
  \label{proposition:bound-on-grad-and-func}
  Define $x^\alpha$ as the solution to the FLIX problem~\eqref{eq:FLIX-problem}:
  \[ x^\alpha \eqdef \argmin_{x \in \R^d} \left [ \tilde{f} (x) = \frac{1}{n} \sum_{i=1}^{n} f_i (\alpha_i x + (1-\alpha_i) x_i) \right ]. \]
  Define $L_\alpha$ as in \Cref{proposition:preserve-smooth-and-convexity}. Then for any $x \in \R^d$ we have
  \begin{equation}
    \label{eq:norm-f-bound}
    \norm{\nabla \tilde{f}(x)} \leq L_{\alpha} \norm{x - x^\alpha},
  \end{equation}
  and
  \begin{equation}
    \label{eq:val-f-bound}
    \tilde{f}(x) - \tilde{f}(x^\alpha) \leq \frac{L_\alpha}{2} \sqn{x - x^\alpha}.
  \end{equation}
\end{proposition}
\begin{proof}
  By \Cref{proposition:preserve-smooth-and-convexity}, we have for any $x, y \in \R^d$ that
  \[ \norm{\nabla \tilde{f} (x) - \nabla \tilde{f} (y)} \leq L_\alpha \norm{x-y}. \]
  Putting $y=x^\alpha$ and using that $\nabla \tilde{f} (x^\alpha) = 0$ we get \eqref{eq:norm-f-bound}. The $L_\alpha$-smoothness of $\tilde{f}$ implies
  \[ \tilde{f}(x) \leq \tilde{f}(y) + \ev{\nabla \tilde{f} (y), x-y} + \frac{L_\alpha}{2} \sqn{x-y}. \]
  Putting $y = x^\alpha$ recovers \eqref{eq:val-f-bound}.
\end{proof}

\subsubsection{Proof of \Cref{proposition:alpha-controls-variance}}
\begin{customprop}{\ref{proposition:alpha-controls-variance}}
    Suppose that $\alpha_1 = \alpha_2 = \ldots = \alpha_n = \beta$ in the FLIX formulation~\eqref{eq:FLIX-problem}. Let $T_1 (x), T_2 (x), \ldots, T_n(x)$ be the deployed models defined in~\eqref{eq:deployed-models-def}. If $y_1, \ldots, y_n$ are vectors in $\R^d$ and $\bar{y}$ is their mean, we define $V(y_1, \ldots, y_n)$ as the population variance $V(y_1, \ldots, y_n) \eqdef \frac{1}{n} \sum_{i=1}^{n} \sqn{y_i - \bar{y}}$. Then,
  \[ V(T_1 (x), T_2 (x), \ldots, T_n (x)) = \br{1 - \beta}^2 V(x_1, x_2, \ldots, x_n). \]
\end{customprop}
\begin{proof}
  By direct computation observe
  \begin{align*}
    V(T_1 (x), T_2 (x), \ldots, T_n (x)) &= \frac{1}{n} \sum_{i=1}^{n} \sqn{T_i (x) - \frac{1}{n} \sum_{j=1}^{n} T_j (x) } \\
                                         &= \frac{1}{n} \sum_{i=1}^{n} \sqn{\beta x + (1-\beta) x_i - \frac{1}{n} \sum_{j=1}^{n} (\beta x + (1 - \beta) x_j) } \\
                                         &= \frac{1}{n} \sum_{i=1}^{n} \sqn{(1 - \beta) \br{x_i - \frac{1}{n} \sum_{j=1}^n x_j} } \\
                                         &= (1-\beta)^2 V(x_1, x_2, \ldots, x_n).
  \end{align*}
\end{proof}

\subsection{Results on one shot averaging}

Before proving Theorems~\ref{thm:one-shot-averaging} and \ref{thm:one-shot-averaging-stationary-points}, we will need the following lemma which shows that the gradient of the FLIX and the functional suboptimality can be bound using a weighted average of the iterate norms:

\begin{lemma}
  \label{lemma:bounds-local-average}
  Suppose that each $f_i$ is $L_i$-smooth and that each $x_i$ is a stationary point of $f_i$. Then for any $x \in \R^d$ we have,
  \begin{equation}
    \label{eq:bounds-gradf-local-average}
    \norm{\nabla \tilde{f}(x)} \leq \frac{1}{n} \sum_{i=1}^{n} \alpha_i^2 L_i \norm{x-x_i},
  \end{equation}
  and,
\begin{equation}
\label{eq:bounds:sqgradf-local-average}
\sqn{\nabla \tilde{f} (x)} \leq \frac{1}{n} \sum_{i=1}^n \alpha_i^4 L_i^2 \norm{x - x_i}^2,
\end{equation}
and,
  \begin{equation}
    \label{eq:bounds-func-local-average}
    \tilde{f}(x) \leq \frac{1}{n} \sum_{i=1}^{n} f_i (x_i) + \frac{1}{2n} \sum_{i=1}^{n} \alpha_i^2 L_i \sqn{x - x_i}.
  \end{equation}
\end{lemma}
\begin{proof}
  For the first inequality, using the fact that $\nabla f_i (x_i) = 0$ for all $i$, and applying the triangle inequality~\eqref{eq:triangle-inequality} and the $L_i$-smoothness of $f_i$, we get
  \begin{align*}
    \norm{\nabla \tilde{f}(x)} &= \norm{\frac{1}{n} \sum_{i=1}^{n} \alpha_i \nabla f_i (\alpha_i x + (1-\alpha_i) x_i) } \\
                       &= \norm{\frac{1}{n} \sum_{i=1}^{n} \alpha_i \left [ \nabla f_i (\alpha_i x + (1-\alpha_i) x_i) - \nabla f_i (x_i) \right ]  } \\
                       &\overset{\eqref{eq:triangle-inequality}}{\leq} \frac{1}{n} \sum_{i=1}^{n} |\alpha_i| \norm{\nabla f_i (\alpha_i x + (1-\alpha_i) x_i) - \nabla f_i (x_i)} \\
                       &\overset{\eqref{eq:smoothness-def}}{\leq} \frac{1}{n} \sum_{i=1}^{n} \alpha_i^2 L_i \norm{x-x_i}.
  \end{align*}

For the second inequality, we use that $\nabla f_i (x_i) = 0$, the convexity of the squared norm, and smoothness:
\begin{align*}
\sqn{\nabla \tilde{f}(x)} &= \sqn{\frac{1}{n} \sum_{i=1}^{n} \alpha_i \nabla f_i (\alpha_i x + (1-\alpha_i) x_i) } \\
                       &= \sqn{\frac{1}{n} \sum_{i=1}^{n} \alpha_i \left [ \nabla f_i (\alpha_i x + (1-\alpha_i) x_i) - \nabla f_i (x_i) \right ]  } \\
                       &\leq \frac{1}{n} \sum_{i=1}^{n} \alpha_i^2 \sqn{\nabla f_i (\alpha_i x + (1-\alpha_i) x_i) - \nabla f_i (x_i)} \\
                       &\overset{\eqref{eq:smoothness-def}}{\leq} \frac{1}{n} \sum_{i=1}^{n} \alpha_i^4 L_i^2 \sqn{x-x_i}.
\end{align*}

  Using $L_i$-smoothness of $f_i$ and that $\nabla f_i (x_i) = 0$, we get
  \begin{align*}
    f_i (\alpha_i x + (1-\alpha_i) x_i) &= f_i (x_i + \alpha_i (x - x_i)) \\
                                        &\overset{\eqref{eq:smoothness-func-consequence}}{\leq} f_i (x_i) + \ev{\nabla f_i (x_i), \alpha_i (x - x_i)} + \frac{\alpha_i^2 L_i}{2} \sqn{x - x_i} \\
    &= f_i (x_i) + \frac{\alpha_i^2 L_i }{2} \sqn{x - x_i}.
  \end{align*}
  Averaging the above inequality yields \eqref{eq:bounds-func-local-average}.
\end{proof}

As communicated earlier in the paper, we can solve~\eqref{eq:FLIX-problem} by taking a weighted average of the pure local models $x_i$ if the $\alpha_i$ are small enough. This is a consequence of preprocessing step where we compute $x_i$, which is communication-free, and also a consequence of the new formulation, where for small $\alpha$ the FLIX objective is less responsive to argument change (see Proposition~\ref{proposition:preserve-smooth-and-convexity}).

\begin{customthm}{\ref{thm:one-shot-averaging}}
  Suppose that each objective $f_i$ is $L_i$-smooth, that $x_i$ minimizes $f_i$, and let $\hat{L} \eqdef \frac{1}{n} \sum_{i=1}^{n} L_i$. Given the pure local models $x_1, x_2, \ldots, x_n$, define the weighted average
  \begin{align}
    \xavg \eqdef \sum_{i=1}^{n} w_i x_i, && w_i \eqdef \frac{\alpha_i^2 L_i}{n L_\alpha}, && L_\alpha \eqdef \frac{1}{n}\sum_{i=1}^{n} \alpha_i^2 L_i.
  \end{align}
  We further define the constants
  \begin{align}
    D \eqdef \max_{i, j = 1, \ldots, n, i \neq j} \norm{x_i - x_j}^2, &&\text{ and } && V \eqdef \sum_{i=1}^{n} w_i \sqn{x_i - \xavg}
  \end{align}
  Fix any $\epsilon > 0$.  Assume that either $\max_{i=1, \ldots, n} \alpha_i \leq \sqrt{2\epsilon}/\sqrt{\hat{L}D}$, or $\alpha_i = \beta$ for all $i$ and $\beta \leq \sqrt{2\epsilon}/\sqrt{\hat{L}D}$. Then $\xavg$ is an $\epsilon$-approximate minimizer of \eqref{eq:FLIX-problem}. That is,
\[ \tilde{f}(\xavg; \alpha_1, \ldots, \alpha_n, x_1, \ldots, x_n) - \min_{x \in \R^d} \tilde{f}(x; \alpha_1, \ldots, \alpha_n, x_1, \ldots, x_n) \leq \epsilon.  \]
\end{customthm}
\begin{proof}

Define
\begin{equation}\label{eq:cE-def}\cE_{\alpha}(x)\eqdef \frac{1}{2n} \sum \limits_{i=1}^n \alpha_i^2 L_i  \norm{x-x_i}^2 .\end{equation}

Note that since the vectors $\{x_i\}$ are known, the expression $\cE_{\alpha}(x)$ can be minimized in $x$, leading to a weighted average of the pure local models:
\begin{equation}\label{eq:x^avg}x^{avg} \eqdef \sum \limits_{i=1}^n w_i x_i, \quad w_i \eqdef \frac{\alpha_i^2 L_i}{nL_{\alpha}}, \quad L_{\alpha}\eqdef \frac{1}{n}\sum \limits_{i=1}^n \alpha_i^2 L_i.\end{equation}
Note that $x^{avg}$,  can be computed using a single communication round.
By plugging \eqref{eq:x^avg} into \eqref{eq:cE-def}, we can evaluate the error of the average model:
\begin{equation}\label{eq:cE-identity}\cE(x^{avg}) = \frac{L_{\alpha}}{2} \sum \limits_{i=1}^n w_i \norm{x^{avg}-x_i}^2= \frac{L_{\alpha} V_{\alpha}}{2},\end{equation}
where \begin{equation}\label{eq:D_alpha} V_\alpha \eqdef \sum \limits_{i=1}^n w_i \norm{x^{avg}-x_i}^2 = \EE{\norm{x_i - \EE{x_i}}^2}\end{equation}
can be interpreted as the {\em variance of the local optimal models} $\{x_i\}$.
We now give two ways how $\cE(x^{avg})$ can be bounded:
\begin{itemize}
\item Assume the customization parameters  $\{\alpha_i\}$ are allowed to be arbitrary. Since $x^{avg}$ is in the convex hull of the set $\{x_1,\dots,x_n\}$, we have $\norm{x^{avg}-x_i} \leq \max_{i,j} \norm{x_i-x_j}$. Therefore, \begin{equation}\label{eq:bf7g9vd898gf_98y}V_{\alpha} \leq D\eqdef \max_{i,j} \norm{x_i-x_j}^2.\end{equation}
Moreover,
\begin{equation}\label{eq:bif7g98gff} L_{\alpha} =  \frac{1}{n}\sum \limits_{i=1}^n \alpha_i^2 L_i \leq \alpha_{\max}^2 \bar{L},\end{equation} where $\alpha_{\max} \eqdef \max_i \alpha_i$ and $\bar{L}\eqdef \frac{1}{n}\sum_{i=1}^n L_i$. By plugging \eqref{eq:bf7g9vd898gf_98y} and \eqref{eq:bif7g98gff} into \eqref{eq:cE-identity}, we get
\begin{equation}\cE(x^{avg}) \leq  \frac{\alpha_{\max}^2 \bar{L}D}{2}. \label{eq:9gfdg789fd}\end{equation}
\item Assume the customization parameters $\{\alpha_i\}$ are all equal: $\alpha_i=\beta$ for all $i$. Then $w_i=\frac{L_i}{\sum_j L_j}$, and hence $x^{avg}$ and $V_{\alpha}=V$ are independent of $\beta$. Since $L_{\alpha} = \beta^2 \bar{L}$, by plugging these expressions into \eqref{eq:cE-identity}, we get
\begin{equation}\cE(x^{avg}) \leq  \frac{\beta^2 \bar{L} V}{2}. \label{eq:njfu_y8g7dfhd}\end{equation}
Note that $V\leq D$.
\end{itemize}
Equations~\eqref{eq:9gfdg789fd} and~\eqref{eq:njfu_y8g7dfhd} yield the following observation: fix any $\epsilon > 0$ and assume one of the following conditions holds: (1) The maximum personalization parameter satisfies
\begin{equation} \label{eq:alpha-max-bound} \alpha_{\max} \leq \sqrt{\frac{2\epsilon}{\bar{L} D}}. \end{equation}
Or, (2) All the personalization parameters are equal, i.e., $\alpha_i = \beta$ for all $i$, and
\begin{equation}
\label{eq:alpha-const-bound}
\beta \leq \sqrt{\frac{2 \epsilon}{\bar{L} V}}.
\end{equation}
Then
\begin{equation} \label{eq:bound-on-ce-xavg} \cE(\xavg) \leq \epsilon. \end{equation}

Then combining \eqref{eq:bounds-func-local-average} from \Cref{lemma:bounds-local-average} with \eqref{eq:bound-on-ce-xavg} we get that the weighted average of the local optimal models $x^{\mathrm{avg}}$ satisfies
\[ \tilde{f}(x^{\mathrm{avg}}) \leq \frac{1}{n} \sum_{i=1}^{n} f_i (x_i) + \epsilon \leq \frac{1}{n} \sum_{i=1}^{n} f_i (x^\alpha)  + \epsilon = \tilde{f}(x^\alpha) + \epsilon,  \]
where $x^\alpha = \argmin_{x \in \R^d} \tilde{f} (x)$ and where we used that $x_i$ minimizes $f_i$. This shows the second part of this Theorem's claim.
\end{proof}

The next theorem extends the previous result to stationary points instead of minimizers.

\begin{theorem}
\label{thm:one-shot-averaging-stationary-points}
  Suppose that each objective $f_i$ is $L_i$-smooth and that $x_i$ is a stationary point of $f_i$, and let $\hat{L} \eqdef \frac{1}{n} \sum_{i=1}^{n} L_i$. Given the pure local models $x_1, x_2, \ldots, x_n$, define the weighted average
  \begin{align}
\label{eq:avg-stationary-point}
    \xavg \eqdef \sum_{i=1}^{n} w_i x_i, && w_i \eqdef \frac{\alpha_i^4 L_i^2}{n \hat{L}_{\alpha}^2} && \hat{L}^2_{\alpha} \eqdef \frac{1}{n} \sum_{i=1}^n \alpha_i^4 L_i^2.
  \end{align}
  Fix any $\epsilon > 0$ and define $D$ and $V$ as in Theorem~\ref{thm:one-shot-averaging}.  Assume that either $\max_{i=1, \ldots, n} \alpha_i \leq \sqrt{2\epsilon}/\sqrt{\hat{L}D}$, or $\alpha_i = \beta$ for all $i$ and $\beta \leq \sqrt{2\epsilon}/\sqrt{\hat{L}V}$. Then $\xavg$ is a stationary point of \eqref{eq:FLIX-problem}. That is,
\[ \norm{\nabla\tilde{f}(\xavg)} \leq \epsilon.  \]
\end{theorem}
\begin{proof}
We will proceed very similarly to Theorem~\ref{thm:one-shot-averaging}. Let
\[ \zeta_{\alpha} (x) = \frac{1}{n} \sum_{i=1}^{n} \alpha_{i}^4 L_i^2 \sqn{x-x_i}. \]
It is easy to show that \eqref{eq:avg-stationary-point} minimizes $\zeta_{\alpha} (x)$, giving
\begin{equation} \label{eq:zeta-xavg-bound} \zeta_{\alpha} (\xavg) = \frac{\hat{L}^2_{\alpha} V_{\alpha}}{2},  \end{equation}
where
\[ V_{\alpha} \eqdef \sum_{i=1}^n w_i \sqn{\xavg - x_i}. \]
\begin{itemize}
\item Assume the customization parameters  $\{\alpha_i\}$ are allowed to be arbitrary. Since $x^{avg}$ is in the convex hull of the set $\{x_1,\dots,x_n\}$, we have $\norm{x^{avg}-x_i} \leq \max_{i,j} \norm{x_i-x_j}$. Therefore, \begin{equation}\label{eq:copythm-1}V_{\alpha} \leq D\eqdef \max_{i,j} \norm{x_i-x_j}^2.\end{equation}
Moreover,
\begin{equation}\label{eq:copythm-2} \hat{L}_{\alpha}^2 =  \frac{1}{n}\sum \limits_{i=1}^n \alpha_i^4 L_i^2 \leq \alpha_{\max}^4 \hat{L}^2,\end{equation} where $\alpha_{\max} \eqdef \max_i \alpha_i$ and $\hat{L}^2 \eqdef \frac{1}{n}\sum_{i=1}^n L_i^2$. By plugging \eqref{eq:copythm-1} and \eqref{eq:copythm-2} into \eqref{eq:zeta-xavg-bound}, we get
        \begin{equation}\zeta_{\alpha}(x^{avg}) \leq  \alpha_{\max}^4 \hat{L}^2D. \label{eq:copythm-3}\end{equation}
\item Assume the customization parameters $\{\alpha_i\}$ are all equal: $\alpha_i=\beta$ for all $i$. Then $w_i=\frac{L_i^2}{\sum_j L_j^2}$, and hence $x^{avg}$ and $V_{\alpha}=V$ are independent of $\beta$. Since $L_{\alpha}^2 = \beta^4 \hat{L}^2$, by plugging these expressions into \eqref{eq:zeta-xavg-bound}, we get
\begin{equation}\zeta_{\alpha}(\xavg) \leq  \beta^4 \hat{L}^2 V. \label{eq:copythm-4}\end{equation}
\end{itemize}
Combining \eqref{eq:copythm-3} and \eqref{eq:copythm-4} we get that if (1) $\alpha_{\mathrm{max}} \leq \br{\frac{\epsilon^2}{\hat{L}^2 D}}^{1/4}$ or (2) $\alpha_i = \beta$ and $\beta \leq \br{\frac{\epsilon^2}{\hat{L}^2 V}}^{1/4}$ then
\[ \zeta_{\alpha}(\xavg) \leq \epsilon^2. \]
It remains to use \eqref{eq:bounds:sqgradf-local-average} from \Cref{lemma:bounds-local-average} and then take square roots.
\end{proof}

\subsection{Distributed Gradient Descent}

\begin{algorithm}[t]
  \caption{Distributed gradient descent for FLIX.}
  \label{alg:dgd-FLIX}
  \begin{algorithmic}[1]
    \Require Number of communication rounds $K$, stepsize $\gamma$, initial point $x^0$
    \For{$k=0, 1, \ldots, K-1$}
    	\For{$i=1, \ldots, n$ in parallel}
	    	\State Compute $\nabla \tilde{f_i}(x^k) = \alpha_i \nabla f_i \br{\alpha_i x^k + \br{1 - \alpha_i} x_i}$ and communicate it to server.
        \EndFor
        \State Average and broadcast the new iterate
        \[ x^{k+1} = x^k - \frac{\gamma}{n} \sum_{i=1}^{n} \nabla \tilde{f_i}(x^k). \]
    \EndFor
  \end{algorithmic}
\end{algorithm}

Gradient descent is a simple but yet informative way to solve many optimization problems. Here we provide the proof for the convergence of DGD.

\begin{customthm}{\ref{theorem:dgd-convergence}}
  Suppose that each $f_i$ in \eqref{eq:FLIX-problem} is $L_i$-smooth and $\mu_i$-strongly convex. Define $\xavg, L_\alpha$, and $\hat{L}$ by~\eqref{eq:xavg-def} and $V, D$ by~\eqref{eq:var-consts-def}. Suppose that we run~\eqref{eq:dgd-update} for $K$ iterations starting from $x^0 = \xavg$. Then the following hold:
  \begin{itemize}[leftmargin=0.15in,itemsep=0.01in,topsep=0pt]
    \item [i)] If the $\alpha_i$ are allowed to be arbitrary, then for $\alpha_{\max} \eqdef \max_{i=1,\ldots,n} \alpha_i$ we have
      \[ \tilde{f}(x^k) - \min_{x \in \R^d} \tilde{f}(x) \leq \br{1 - \frac{\mu_\alpha}{L_\alpha}}^k \frac{\alpha_{\max}^2 \hat{L} D}{2}. \]
    \item [ii)] If $\alpha_i = \beta$ for all $i$, then
      \begin{equation*}
        \label{eq:dgd-conv-beta}
        \tilde{f}(x^k) - \min_{x \in \R^d} \tilde{f}(x) \leq \br{1 - \frac{\hat{\mu}}{\hat{L}}}^k \frac{\beta^2 \hat{L} V}{2},
      \end{equation*}
      where $\hat{\mu} \eqdef \frac{1}{n} \sum_{i=1}^{n} \mu_i$.
\end{itemize}
\end{customthm}
\begin{proof}
  Recall the standard result that gradient descent for an $L$-smooth and $\mu$-strongly convex objective $g$ satisfies for any initial point $x^0$
  \begin{equation}
    \label{eq:standard-gd-conv}
    g(x^k) - g_\ast \leq \br{1 - \gamma \mu}^k \br{g(x_0) - g_\ast},
  \end{equation}
  where $g_\ast = \min_{x \in \R^d} g(x)$. For a proof, see \citep{Nesterov2018}. Note that by \Cref{proposition:preserve-smooth-and-convexity} we have that $\tilde{f}$ is $L_\alpha$-smooth and $\mu_\alpha$-strongly convex. Specializing \eqref{eq:standard-gd-conv} to this case yields
  \begin{equation}
    \label{eq:gd-conv-no-init}
    \tilde{f} (x^k) - \min_{x \in \R^d} \tilde{f}(x) \leq \br{1 - \frac{\mu_{\alpha}}{L_{\alpha}}}^k \br{\tilde{f}(x_0) - \min_{x \in \R^d} \tilde{f}(x)}
  \end{equation}
  Note that our initialization is the same $\xavg$ from Theorem~\ref{thm:one-shot-averaging}. Plugging the bounds of that theorem into \eqref{eq:gd-conv-no-init} yields the theorem's claims.
\end{proof}

\clearpage
\section{OTHER ALGORITHMS}

In this section, we include additional algorithms to solve \eqref{eq:FLIX-problem} formulation, namely, DCGD~\citep{gorbunov2019unified} and DIANA~\citep{diana_paper}. In Table~\ref{table:results}, we display convergence guarantees for these algorithms. One can see that similarly to our results in Section~\ref{sec:algorithms-for-FLIX} both of the algorithms requires less iterations in terms of convergence in both local deploy iterates $T_i(x^k)$'s and functional value $\tilde{f}(x^k)$. In addition, for the standard \eqref{eq:ERM} problem, i.e., $\alpha = 1$, we recover the best known convergence guarantees. Below, we provide a derivation of these claims.

\renewcommand{\arraystretch}{1.5}
\begin{table}
\caption{Convergence results for Distributed Compressed Gradient Descent and DIANA in different settings. All the constants are independent of $\alpha$.}
\label{table:results}
\begin{center}
\begin{tabular}{llll}
\textbf{ALGORITHM} & \textbf{ASSUMPTION} & \textbf{CONVERGENCE GUARANTEE} & \textbf{COROLLARY} \\
\hline\\
\multirow{2}{*}{DCGD} & \multirow{2}{*}{smoothness, str cvx}&$\EE\|x^k - x^\ast\|^2 \leq (1 - \gamma_0 \overline{\mu})^k C_1 + C_2$ &  \multirow{2}{*}{\ref{crl:cgd_str_cvx}}\\
 					  & 														  &$\EE [\tilde{f}(x^k) - \tilde{f}^\ast] \leq ((1 - \gamma_0 \overline{\mu})^k C_1 + C_2)\alpha^2$ &\\

\multirow{2}{*}{DIANA} & \multirow{2}{*}{smoothness, str cvx}&$\EE\|x^k - x^\ast\|^2 \leq (1 - \rho)^k C$ & \multirow{2}{*}{\ref{crl:DIANA}} \\
 					  & 														  &$\EE [\tilde{f}(x^k) - \tilde{f}^\ast] \leq  (1 - \rho)^k C \alpha^2$ & \\

DCGD & smoothness, cvx & $\EE[\tilde{f}(\overline{x}^k) - \tilde{f}(x^\ast)] \leq \frac{1}{k} C_1 \alpha^2 + C_2 \alpha$ & \ref{crl:cgd_cvx} \\

DIANA & smoothness, cvx & $\EE[\tilde{f}(\overline{x}^k) - \tilde{f}(x^\ast)] \leq \frac{1}{k} (C_1 \alpha^2 + C_2 \alpha)$ &  \ref{crl:diana_cvx} \\

DCGD & smoothness & $\min\limits_{0 \leq t \leq k - 1} \EE \|\nabla \tilde{f}(x^t)\|^2 \leq  \frac{\left(1 +  C_1\gamma_0^2\right)^k C_2}{\gamma_0 k} \alpha^2$ &  \ref{crl:dcgd_non_convex} \\

DIANA & smoothness & $\EE \|\nabla \tilde{f}(\hat{x})	\|^2 \leq \frac{C}{k} \alpha^2$ & \ref{crl:diana_nncvx} \\
\end{tabular}
\end{center}
\end{table}

\subsection{Strongly convex objectives}

\subsubsection{DCGD}

We firstly introduce the DCGD algorithm followed by general convergence results that we later exploit to obtain a convergence guarantee for DCGD applied to \eqref{eq:FLIX-problem}.

\begin{algorithm}[H]
\caption{Distributed Compressed Gradient Descent with different noise levels $\omega_i$}
\begin{algorithmic}[1]
\Require{$x^0 \in \ \RR^d$, learning rate $\gamma$}
\For{$k = 0, 1, 2, \dots$}
	\State Broadcast $x^k$ to all workers
	\For{$i = 1, \dots, n $ in parallel}
		\State Evaluate $\nabla f_i(x^k)$
		\State $g_i^k = \mathcal{C}_i(\nabla f_i(x^k))$
	\EndFor
	\State $g^k = \avein g_i^k $
	\State $x^{k + 1} = x^k - \gamma g^k$
\EndFor
\end{algorithmic}
\end{algorithm}

\begin{lemma}\label{lm:CGD}
Suppose each $f_i$ is $L_i$-smooth and convex, and $f$ is $L$-smooth. Let
$\cC_i : \RR^d \rightarrow \RR^d$ be randomized compression operators satisfying $\cC_i \in \mathbb{B}^d(\omega_i)$. Let $g_k = \avein {\cal C}_i(\nabla f_i(x^k))$. Then
\begin{equation}
\begin{aligned}
\EE \|g^k - \nabla f(x^\ast)\|^2 \leq 2 \left(L + \frac{2 \max\{L_i \omega_i\}}{n} \right) D_f(x^k, x^\ast) + \sigma_{DCGD} ,
\end{aligned}
\end{equation}
where $\sigma_{DCGD} = \frac{2}{n^2} \sumin \omega_i \|\nabla f_i(x^\ast)\|^2$.
\end{lemma}

\begin{proof}
By applying bias-variance decomposition to $\EE\|g^k - \nabla f(x^\ast)\|^2$, we get
\begin{align*}
\EE \|g^k - \nabla f(x^\ast)\|^2 = \|\nabla f(x^k) - \nabla f(x^\ast)\|^2 +  \EE \|g^k - \nabla f(x^k)\|^2.
\end{align*}

Since all functions are convex holds, function $f$ as a linear combination of convex functions is also convex. That is why the first term enjoys the 'classic' bound for \textit{convex} and smooth functions expressed in Bregman divergence between $x^k$ and $x^\ast$ (see inequality~\ref{ineq:cvx_smooth}):
\begin{align} \label{ineq:cgd_first_term}
\|\nabla f(x^k) - \nabla f(x^\ast)\|^2 \leq 2 L D_f(x^k, x^\ast).
\end{align}

We start looking into the second term by expanding the quadratic:
\begin{align*}
&\EE \|g^k - \nabla f(x^k)\|^2 \\
& = \EE\left\|\avein \cC_i(\nabla f_i(x^k)) - \avein \nabla f_i(x^k)  \right\|^2 \\
& = \EE\left\|\avein (\cC_i(\nabla f_i(x^k)) - \nabla f_i(x^k)) \right\|^2 \\
& = \frac{1}{n^2} \sumin \EE \|\cC_i(\nabla f_i(x^k)) - \nabla f_i(x^k)\|^2 + \\
& \qquad \frac{1}{n^2} \sum\limits_{i \neq j} \EE\langle \cC_i(\nabla f_i(x^k)) - \nabla f_i(x^k), \cC_j(\nabla f_j(x^k)) - \nabla f_j(x^k) \rangle.
\end{align*}

Since $\cC_i \in \BB(\omega_i) \ \forall i $ is drawn independently, the expected value of each scalar product in the second sum becomes the scalar product of expected values, each of which is zero due to unbiasedness of a compression operator (see Section~\ref{sec:basic_facts}):
\begin{align*}
&\EE\langle \cC_i(\nabla f_i(x^k)) - \nabla f_i(x^k), \cC_j(\nabla f_j(x^k)) - \nabla f_j(x^k) \rangle \\
&= \langle \underbrace{\EE\cC_i(\nabla f_i(x^k)) - \nabla f_i(x^k)}_{=0}, \underbrace{\EE\cC_j(\nabla f_j(x^k)) - \nabla f_j(x^k) \rangle}_{= 0} = 0.
\end{align*}
Then, we apply~\ref{ineq:cmp_oper_variance} (the second property of compressed operators) to get the final upper-bound on $\EE\|g^k - \nabla f(x^k)\|^2$.
\begin{align*}
&\EE \|g^k - \nabla f(x^k)\|^2 = \frac{1}{n^2} \sumin \EE \|\cC_i(\nabla f_i(x^k)) - \nabla f_i(x^k)\|^2 \leq \frac{1}{n^2} \sumin \omega_i \|\nabla f_i(x^k)\|^2.
\end{align*}

Curiously but not surprisingly, the result resembles the Law of large numbers (indeed, if each squared gradient norm and noise level are bounded above by values $R$ and $\omega_0$ respectively, then the right side is bounded by $\frac{\omega_0 R}{n}$, which converges to zero as $n$ converges to infinity), what partially verifies the correctness of the proof.
The error $\EE \|g^k - \nabla f(x^k)||^2$ is now estimated by squared gradient norms at iterate $x^k$, which is, in general, a random point generated by CGD. Since we already have some dependence on Bregman divergence in~\ref{ineq:cgd_first_term}, which perfectly fits assumptions for unified theories from papers~\cite{gorbunov2019unified} and~\cite{khaled2020unified} required in later theorems, we express the term in the right side of the last inequality in Bregman divergence, too. By subtracting and adding the same vector $\nabla f_i(x^\ast)$ inside the norm operator we get
\begin{align*}
&\frac{1}{n^2} \sumin \omega_i \|\nabla f_i(x^k)\|^2 = \frac{1}{n^2} \sumin \omega_i \|\nabla f_i(x^k) - \nabla f_i(x^\ast) + \nabla f_i(x^\ast)\|^2\\
&\leq \frac{2}{n^2} \sumin \omega_i \|\nabla f_i(x^k) - \nabla f_i(x^\ast)\|^2 + \frac{2}{n^2} \sumin \omega_i \|\nabla f_i(x^\ast)\|^2,\\
\end{align*}
where in the second line we used \eqref{ineq:a_b}. Applying again the bound from~\ref{ineq:cvx_smooth} to the first term and using the linearity of Bregman divergence we have
\begin{align*}
& \frac{2}{n^2} \sumin \omega_i \|\nabla f_i(x^k) - \nabla f_i(x^\ast)\|^2 \leq \frac{2}{n^2} \sumin \omega_i 2 L_i  D_{f_i}(x^k, x^\ast)\\
& \leq \frac{4\max \{ L_i \omega_i\}}{n^2} \sumin D_{f_i}(x^k, x^\ast) =\frac{4 \max \{L_i \omega_i\}}{n} D_f(x^k, x^\ast).\\
\end{align*}
Incorporation of the results above gives the statement.
\end{proof}

\begin{proposition}\label{prop:dist_opt_bound}
		Suppose each $f_i$ is $L_i$-smooth and $\mu_i$-strong. Assume $\alpha_i \equiv \alpha \in \RR$ for all $i$. For any $x^0 \in \RR^d$, it holds that
		\begin{align}\label{ineq:dist_opt_bound}
		\|x^0 - x^\ast\|^2 \leq 	\frac{1}{\overline{\mu}} \overline{L_i  \|x^0 - x_i\|^2}.
		\end{align}
\end{proposition}
\begin{proof}
	\begin{align*}
	 &\|x^0 - x^\ast\|^2\overset{\eqref{eq:strong-conv-def}}{\leq} \frac{2}{\mu_\alpha} (\tilde{f}(x^0) - \tilde{f}(x^\ast)) \overset{\text{Proposition }~\ref{proposition:preserve-smooth-and-convexity}}{=}\frac{2}{\overline{\mu}\alpha^2} (\tilde{f}(x^0) - \tilde{f}(x^\ast)) \overset{\eqref{eq:bounds-func-local-average}}{\leq} \frac{2}{\overline{\mu}\alpha^2} \cE_\alpha (x^0) \\
	 & \overset{\eqref{eq:cE-def}}{=} \frac{1}{\overline{\mu}\alpha^2} \avein L_i \alpha^2 \|x^0 - x_i\|^2 = \frac{1}{\overline{\mu}} \overline{L_i  \|x^0 - x_i\|^2}.
	\end{align*}
\end{proof}

\begin{proposition}\label{prop:aver_dist_bound}
	Suppose each $f_i$ is $L_i$-smooth and $\mu_i$-strong. Assume $\alpha_i \equiv \alpha \in \RR$ for all $i$. Then,
	\begin{align}\label{ineq:aver_dist_bound}
	\overline{\|x_i - x^\ast\|^2} \leq \frac{\max_i L_i}{\overline{\mu}}  \max_{i, j} \|x_i - x_j\|^2.
	\end{align}
\end{proposition}
\begin{proof}
	According to Proposition~\ref{prop:dist_opt_bound}, for any $x^\ast_l$ where $l \in \{1, \dots, n\}$ it follows that
	\begin{align*}
	\|x_l - x^\ast\|^2 \leq \frac{1}{\overline{\mu}} \overline{L_i  \|x_l - x_i\|^2} \leq \frac{1}{\overline{\mu}} \max_i L_i \cdot \overline{\|x_l - x_i\|^2} \leq \frac{1}{\overline{\mu}} \max_i L_i \cdot \max_{i, j} \|x_j - x_i\|^2.
	\end{align*}
	Taking the average of both sides finishes the proof.
\end{proof}

\begin{theorem}\label{thm:cgd_str_cvx}
Assume all conditions of Lemma \ref{lm:CGD} hold, and each function $f_i$ is $\mu_i$-strongly convex. Then, if $\gamma \leq \frac{1}{L_\alpha + \frac{2 \max\{L_i \alpha_i^2 \omega_i\}}{n}}$, then
\begin{equation}\label{main_eq_dcgd}
\begin{aligned}
\EE \|x^k - x^\ast\|^2 \leq (1 - \gamma \mu_\alpha)^k \|x^0 - x^\ast\|^2 + \frac{2\gamma}{\mu_\alpha n^2}	\sumin \omega_i \|\nabla [f_i(T_i(x^\ast))]\|^2.
\end{aligned}
\end{equation}
\end{theorem}

\begin{proof}
The lemma is a direct corollary of Theorem 4.1 from \cite{gorbunov2019unified} and Proposition~\ref{proposition:preserve-smooth-and-convexity} with constants $A = L_\alpha + \frac{2 \max\{L_i \alpha_i^2 \omega_i\}}{n},  D_1 = \sigma_{DCGD} = \frac{2}{n^2} \sumin \omega_i \|\nabla [f_i(T_i(x^\ast))]\|^2, \\ B~=~0, \sigma^2_k \equiv 0, \rho~=~1, C~=~0, D_2 = 0$.
\end{proof}

\begin{corollary}\label{crl:cgd_str_cvx}
	Assume all conditions of Theorem~\ref{thm:cgd_str_cvx} hold, and $\alpha_i \equiv \alpha \ \forall i$. If $\gamma = \frac{1}{\alpha^2 \left(\overline{L} + \frac{2 \omega \max_i L_i}{n}\right)} = \gamma_0 \frac{1}{\alpha^2}$, then
	\begin{align}
	\EE \|x^k - x^\ast\|^2 \leq (1 - \gamma_0 \overline{\mu})^k 	\frac{1}{\overline{\mu}} \overline{L_i  \|x^0 - x_i\|^2} + \frac{2\gamma_0\omega (\max_i L_i)^2 \max_{i, j} \|x_i - x_j\|^2}{\overline{\mu}^2 n}
	\end{align}
	and
	\begin{align}
	\EE [\tilde{f}(x^k) - \tilde{f}^\ast] \leq ((1 - \gamma_0 \overline{\mu})^k C_1 + C_2)\alpha^2,
	\end{align}
	where $C_1 = \frac{\overline{L_i}}{2\overline{\mu}} \overline{L_i  \|x^0 - x_i\|^2} $ and $C_2 = \frac{\overline{L_i}\gamma_0\omega (\max_i L_i)^2 \max_{i, j} \|x_i - x_j\|^2}{\overline{\mu}^2 n}$.
\end{corollary}

\begin{proof}
Notice that
\begin{align*}
&\sumin \omega_i \|\nabla [f_i(\alpha_i x^\ast + (1 - \alpha_i) x_i)]\|^2 = \sumin \omega_i \|\alpha_i \nabla f_i(\alpha_i x^\ast + (1 - \alpha_i) x_i)\|^2\\
& = \omega \alpha^2 \sumin \|\nabla f_i(\alpha x^\ast + (1 - \alpha) x_i)\|^2 \overset{\nabla f_i(x_i) = 0}{=} \omega \alpha^2 \sumin \|\nabla f_i(\alpha x^\ast + (1 - \alpha) x_i) - \nabla f_i(x_i)\|^2\\
& \overset{\eqref{eq:smoothness-def}}{\leq} \omega \alpha^2 \sumin L_i \|\alpha (x^\ast - x_i)\|^2 \leq \omega \alpha^4 \max_i L_i n \overline{\|x^\ast - x_i\|^2} \overset{\eqref{ineq:aver_dist_bound}}{\leq} \frac{\omega \alpha^4 n (\max_i L_i)^2}{\overline{\mu}} \max_{i, j} \|x_i - x_j\|^2.
\end{align*}
This brings us to the following bound on the neigbourhood:
\begin{equation}\label{ineq:cgd_neigb_str_cvx}
\begin{aligned}
& \frac{2\gamma}{\mu_\alpha n^2}	\sumin \omega_i \|\nabla [f_i(T_i(x^\ast))]\|^2 = \frac{2 \gamma_0}{\alpha^4  \overline{\mu} n^2} 	\sumin \omega_i \|\nabla [f_i(T_i(x^\ast))]\|^2\\
& \leq \frac{2\gamma_0\omega (\max_i L_i)^2 \max_{i, j} \|x_i - x_j\|^2}{\overline{\mu}^2 n}.
\end{aligned}
\end{equation}

We further investigate the rate of convergence
\begin{align}\label{eq:rate_ind_alpha}
1 - \gamma \mu_\alpha \overset{\text{Proposition }~\ref{proposition:preserve-smooth-and-convexity}}{=} 1 - \gamma_0 \frac{1}{\alpha^2} \overline{\mu} \alpha^2 = 1 - \gamma_0 \overline{\mu}.
\end{align}
As can be seen, the rate of convergence does not depend on $\alpha$.

Combining last two results we obtain the following dependence on $\alpha$ for CGD convergence:
\begin{equation}\label{ineq:cgd_str_cvx_alpha_dep}
\begin{aligned}
&\EE \|x^k - x^\ast\|^2 \\
& \overset{\eqref{main_eq_dcgd}}{\leq} (1 - \gamma \mu_\alpha)^k \|x^0 - x^\ast\|^2 + \frac{2\gamma}{\mu_\alpha n^2}	\sumin \omega_i \|\nabla [f_i(T_i(x^\ast))]\|^2\\
& \overset{\eqref{eq:rate_ind_alpha}}{=} (1 - \gamma_0 \overline{\mu})^k \|x^0 - x^\ast\|^2 + \frac{2\gamma}{\mu_\alpha n^2}	\sumin \omega_i \|\nabla [f_i(T_i(x^\ast))]\|^2\\
& \overset{\eqref{ineq:dist_opt_bound}}{\leq} (1 - \gamma_0 \overline{\mu})^k 	\frac{1}{\overline{\mu}} \overline{L_i  \|x^0 - x_i\|^2} + \frac{2\gamma}{\mu_\alpha n^2}	\sumin \omega_i \|\nabla [f_i(T_i(x^\ast))]\|^2\\
& \overset{\eqref{ineq:cgd_neigb_str_cvx}}{\leq} (1 - \gamma_0 \overline{\mu})^k 	\frac{1}{\overline{\mu}} \overline{L_i  \|x^0 - x_i\|^2} + \frac{2\gamma_0\omega (\max_i L_i)^2 \max_{i, j} \|x_i - x_j\|^2}{\overline{\mu}^2 n}\\
& = (1 - \gamma_0 \overline{\mu})^k \widetilde{C}_1 + \widetilde{C}_2,
\end{aligned}
\end{equation}
where constants $\widetilde{C}_1 = \frac{1}{\overline{\mu}} \overline{L_i  \|x^0 - x_i\|^2}$ and $\widetilde{C}_2 = \frac{2\gamma_0\omega (\max_i L_i)^2 \max_{i, j} \|x_i - x_j\|^2}{\overline{\mu}^2 n}$ do not depend on~$\alpha$.

In terms of suboptimality convergence for CGD we have
	\begin{equation}
	\begin{aligned}
	&\EE [\tilde{f}(x^k) - \tilde{f}^\ast] \overset{\eqref{eq:smoothness-func-consequence}}{\leq} \frac{L_\alpha}{2} \EE\|x^k - x^\ast\|^2 \overset{L_\alpha \ \text{def}}{=} \frac{\overline{L}}{2} \alpha^2 \EE \|x^k - x^\ast\|^2 \overset{\eqref{ineq:cgd_str_cvx_alpha_dep}}{\leq} ((1 - \gamma_0 \overline{\mu})^k C_1 + C_2)\alpha^2 ,
	\end{aligned}
	\end{equation}
	where $C_1 = \frac{\widetilde{C}_1\overline{L}}{2}$ and $C_2 =  \frac{\widetilde{C}_2\overline{L}}{2}$.
\end{proof}

\subsubsection{DIANA}

We follow the same procedure as for the previous subsection. Below, we introduce DIANA algorithm, followed by general theorem, which is then applied to \eqref{eq:FLIX-problem}.

\begin{algorithm}[H]
\caption{DIANA with different noise levels $\omega_i$}
\begin{algorithmic}[1]
\Require{$x^0, h^0_1, \dots, h^0_n~\in~\mathbb{R}^d$, $h^0 = \frac{1}{n} \sum_{i=1}^n h^0_i$}
\For{$k$ = 0, 1, 2, $\dots$}
	\State Broadcast $x^k$ to all workers
	\For {$i = 1, \dots, n$ in parallel}
		\State $\Delta_i^k = \nabla f_i(x^k) - h_i^k$
		\State Sample $\hat{\Delta}^k_i \sim \cC_i(\Delta_i^k)$
		\State $h^{k+1}_i = h^k_i + \beta_i \hat{\Delta}^k_i$
		\State $\hat{g}_i^k = h^k_i + \hat{\Delta}^k_i$
	\EndFor
	\State $g^k = \frac{1}{n} \sum_{i=1}^{n} \hat{g}_i^k = h^k + \avein \hat{\Delta}^k_i$
	\State $x^{k+1} = x^k - \gamma g^k $
	\State $h^k = \frac{1}{n} \sum_{i=1}^{n} h^{k+1}_i  = h^k + \avein \beta_i \hat{\Delta}^k_i $
\EndFor
\end{algorithmic}
\end{algorithm}

\begin{lemma}\label{lm:DIANA}
Suppose each $f_i$ is $L_i$-smooth and convex, and $f$ is $L$-smooth. Let
${\cal C}_i : \mathbb{R}^d \rightarrow \mathbb{R}^d$ be randomized compression operators satisfying ${\cal C}_i \in \mathbb{B}^d(\omega_i)$. Let $g_k = \avein h_i^k + {\cal C}_i(\nabla f_i(x^k) - h_i^k)$. Then
\begin{equation}
\begin{aligned}
\EE \|g^k - \nabla f(x^\ast)\|^2 \leq 2 \left(L + \frac{2 \max \{L_i \omega_i \} }{n} \right) D_f(x^k, x^\ast) + \frac{2}{n} \sigma_k^2,
\end{aligned}
\end{equation}
where $\sigma_k^2 = \avein \omega_i \|h^k_i - \nabla f_i(x^\ast)\|^2$.
\end{lemma}

\begin{proof}
We start with bias-variance decomposition:
\begin{align}\label{tmp:diana_bvd}
\EE \|g^k - \nabla f(x^\ast)\|^2 = \|\nabla f(x^k) - \nabla f(x^\ast)\|^2 +  \EE \|g^k - \nabla f(x^k)\|^2,
\end{align}
where the first term in RHS bounded due to convexity and smoothness of function $f$
\[ \|\nabla f(x^k)~-~\nabla f(x^\ast)\|^2 \overset{\eqref{ineq:cvx_smooth}}{\leq} 2 L D_f(x^k, x^\ast). \] Function $f$ is, indeed, convex as a linear combination of convex functions $f_i$.

When expanding the second term in RHS of equation~\eqref{tmp:diana_bvd}, we encounter scalar products, each of which is zero in expectation due to independence of $\cC_i$ for all $i$ and unbiasedness of a compression operator (see the definition in Section~\ref{sec:basic_facts}).
\begin{align*}
&\EE \|g^k - \nabla f(x^k)\|^2 \\
&= \EE \left\|\avein (h^k_i + \cC_i(\nabla f_i(x^k) - h^k_i)) - \avein \nabla f_i(x^k)\right\|^2\\
& = \EE \left\|\avein \cC_i(\nabla f_i(x^k) - h^k_i) - (\nabla f_i(x^k) -  h^k_i)\right\|^2\\
& = \frac{1}{n^2} \sumin \EE\|\cC_i(\nabla f_i(x^k) - h^k_i) - (\nabla f_i(x^k) -  h^k_i)\|^2 \\
& \qquad + \frac{1}{n^2}\sum\limits_{i \neq j} \underbrace{\EE \langle \cC_i(\nabla f_i(x^k) - h^k_i) - (\nabla f_i(x^k) -  h^k_i), \cC_j(\nabla f_j(x^k) - h^k_j) - (\nabla f_j(x^k) -  h^k_j) \rangle}_{=0}.
\end{align*}
Then, we apply bounded variance property of compression operators:
\begin{align*}
\EE \|g^k - \nabla f(x^k)\|^2  &= \frac{1}{n^2} \sumin \EE\|\cC_i(\nabla f_i(x^k) - h^k_i) - (\nabla f_i(x^k) -  h^k_i)\|^2 \\
& \leq \frac{1}{n^2} \sumin \omega_i \|\nabla f_i(x^k) -  h^k_i\|^2.
\end{align*}
We subtract and add $\nabla f_i(x^\ast)$ inside each norm operator to split $\nabla f_i(x^k)$ and $h^k_i$ from each other using \eqref{ineq:a_b}:
\begin{align*}
&\EE \|g^k - \nabla f(x^k)\|^2  \leq \frac{1}{n^2} \sumin \omega_i \|\nabla f_i(x^k) -  h^k_i\|^2 \\
& = \frac{1}{n^2} \sumin \omega_i \|\nabla f_i(x^k) - \nabla f_i(x^\ast) -  (h^k_i - \nabla f_i(x^\ast)) \|^2\\
& \leq \frac{2}{n^2} \sumin \omega_i \|\nabla f_i(x^k) - \nabla f_i(x^\ast)\|^2 + \frac{2}{n^2} \sumin \omega_i \|h^k_i - \nabla f_i(x^\ast)\|^2
\end{align*}
While leaving the second term in the last line unchanged (it is basically $2\sigma_k / n$), we apply \eqref{ineq:cvx_smooth} to each norm of gradient differences in the first term and use max-function over $L_i \omega_i$ to take out Bregman divergences:
\begin{align*}
&\frac{2}{n^2} \sumin \omega_i \|\nabla f_i(x^k) - \nabla f_i(x^\ast)\|^2 \leq \frac{2}{n^2} \sumin \omega_i 2 L_i D_{f_i}(x^k, x^\ast) \\
& \leq \frac{4 \max \{\omega_i L_i\}}{n^2}\sumin D_{f_i}(x^k, x^\ast) = \frac{4 \max \{\omega_i L_i\}}{n} D_f(x^k, x^\ast),
\end{align*}
where in the last part we used linearity property of Bregman divergence. It remains to incorporate all results to get the statement.
\end{proof}

\begin{lemma} \label{lm:DIANA_sigma_k}
Assuming all conditions of Lemma \ref{lm:DIANA} hold, let $h_i^{k+1} = h_i^k + \beta_i \cC_i(\nabla f_i(x^k) - h^k_i)$ and  $\beta_i~\leq~\frac{1}{\omega_i + 1}$, where $ i~\in~\{1, \dots, n\}$. Then,
\begin{align}
\sigma_{k+1}^2 \leq (1 - \min \beta_i) \sigma_k^2 + 2 \max\{\beta_i \omega_i L_i\} D_f (x^k, x^\ast).
\end{align}
\end{lemma}

\begin{proof}
\cite{gorbunov2019unified} state for $\beta_i \leq \frac{1}{\omega_i + 1}$
\begin{equation}
\begin{aligned}
&\EE \|h_i^{k+1} - \nabla f_i(x^\ast)\|^2 \leq (1 - \beta_i) \|h_i^k - \nabla f_i(x^\ast)\|^2 + 2 \beta_i L_i D_{f_i} (x^k, x^\ast).
\end{aligned}
\end{equation}
From this it follows that
\begin{align*}
&\sigma_{k+1}^2  = \avein \omega_i \|h_i^{k+1} - \nabla f_i(x^\ast)\|^2\\
& \leq \avein (1 - \beta_i) \omega_i \|h_i^k - \nabla f_i(x^\ast)\|^2 +  \avein 2 \beta_i \omega_i L_i D_{f_i} (x^k, x^\ast)\\
& \leq (1 - \min \beta_i) \sigma_k^2 + \avein 2 \beta_i \omega_i L_i D_{f_i} (x^k, x^\ast)\\
& \leq (1 - \min \beta_i) \sigma_k^2 + 2 \max\{\beta_i \omega_i L_i\} \avein D_{f_i} (x^k, x^\ast)\\
& = (1 - \min \beta_i) \sigma_k^2 + 2 \max\{\beta_i \omega_i L_i\} D_f (x^k, x^\ast).\\
\end{align*}
\end{proof}

\begin{theorem}\label{thm:diana}
	Assume all conditions of Lemmas~\ref{lm:DIANA}, \ref{lm:DIANA_sigma_k} hold, and each function $f_i$ is $\mu_i$-strongly convex. If $\gamma~\leq~\frac{1}{L_\alpha + \frac{2 \max\{L_i \alpha_i^2 \omega_i\}}{n} + \frac{4 \max\{\beta_i \omega_i L_i\alpha_i^2\}}{n \min \beta_i}}$, then
	\begin{align}\label{diana_thm_ineq}
	\EE [\cD^k] \leq \max \left\{(1 - \gamma \mu_\alpha)^k, \left(1 - \frac12 \min \beta_i\right)^k \right\} \cD^0,
	\end{align}
	where $\cD^k = \|x^k - x^\ast\|^2 + \frac{4}{n \min \beta_i}  \gamma^2 \sigma^2_k$ and $L_\alpha$ is defined as in Proposition~\ref{proposition:preserve-smooth-and-convexity}.
\end{theorem}

\begin{proof}
	The lemma is a direct corollary of Theorem 4.1 from~\cite{gorbunov2019unified}, \ref{lm:DIANA} and~\ref{lm:DIANA_sigma_k} and strong convexity with constants: $A = L_\alpha + \frac{2 \max\{L_i \alpha_i^2 \omega_i\}}{n}$, $B = \frac{2}{n}$, $D_1 = 0$, $\rho = \min \beta_i$, $C = \max\{\beta_i \omega_i L_i \alpha_i^2\}$, $D_2 = 0, M = \frac{4}{n \min \beta_i} = \frac{2 B}{\rho} \geq \frac{4 (1 + \max \omega_i)}{n}$.
\end{proof}

\begin{corollary}\label{crl:DIANA}
	Assume all conditions of Theorem~\ref{thm:diana} hold and $\alpha_i \equiv \alpha \ \forall i$. If $\gamma=\frac{1}{\alpha^2}~\cdot~\frac{1}{\overline{L} + \frac{2 \max\{L_i\omega_i\}}{n} + \frac{4 \max\{\beta_i \omega_i L_i\}}{n \min \beta_i}} =: \gamma_0 \frac{1}{\alpha^2}$, then
	\begin{equation}
	\begin{aligned}
	&\EE\|x^k - x^\ast\|^2\leq \max \left\{(1 - \gamma_0 \overline{\mu})^k, \left(1 - \frac12 \min \beta_i\right)^k \right\} \left(1 + \frac{4}{n \min \beta_i} \gamma_0^2 \overline{\omega_i L_i^2} \right) \frac{1}{\overline{\mu}} \overline{L_i  \|x^0 - x_i\|^2}
	\end{aligned}
	\end{equation}
	and
	\begin{equation}
	\begin{aligned}\label{diana_crl_ineq}
	&\EE[\tilde{f}(x^k)] - \tilde{f}^\ast \leq \max \left\{(1 - \gamma_0 \overline{\mu})^k, \left(1 - \frac12 \min \beta_i\right)^k \right\} C \alpha^2,
	\end{aligned}
	\end{equation}
	where $C = \frac{\overline{L_i}}{2}\left(1 + \frac{4}{n \min \beta_i} \gamma_0^2 \overline{\omega_i L_i^2} \right) \frac{1}{\overline{\mu}} \overline{L_i  \|x^0 - x_i\|^2}$.
\end{corollary}

\begin{proof}
	Let us investigate the quantity $\cD^0$. We assume memory $h_0^i$ in DIANA algorithm is equal to $\nabla [f_i(T_i(x^0))]$ for $i \in \{1, \dots, n\}$. Plugging in definitions and initializations we obtain expanded definition of $\cD^0$:
	\begin{align*}
	\cD^0 \eqdef& \|x^0 - x^\ast\|^2 + \frac{4}{n \min \beta_i}  \gamma^2 \sigma^2_0
	 \overset{\gamma = \gamma_0 \frac{1}{\alpha^2}}{=} \|x^0 - x^\ast\|^2 + \frac{4}{n \min \beta_i}  \gamma_0^2 \frac{1}{\alpha^4}\sigma^2_0\\
	 \overset{\sigma_0^2 \text{ def}}{=}& \|x^0 - x^\ast\|^2 + \frac{4}{n \min \beta_i}  \gamma_0^2 \frac{1}{\alpha^4} \avein \omega_i \|h^i_0 - \nabla [f_i (T_i(x^\ast))]\|^2\\
	 \overset{h_0^i \text{initialisation}}{=}& \|x^0 - x^\ast\|^2 + \frac{4}{n \min \beta_i}  \gamma_0^2 \frac{1}{\alpha^4} \avein \omega_i \|\nabla [f_i(T_i(x^0))] - \nabla [f_i (T_i(x^\ast))]\|^2\\.
	\end{align*}
	We further upper bound the sum in the last term using smoothness properties of functions $f_i$:
	\begin{align*}
	&\avein \omega_i \|\nabla [f_i(T_i(x^0))] - \nabla [f_i (T_i(x^\ast))]\|^2 \\
	& = \avein \omega_i \|\nabla [f_i(\alpha x^0 + (1 - \alpha)x^\ast_i)] - \nabla [f_i (\alpha x^\ast + (1 - \alpha)x^\ast_i)]\|^2\\
	& = \avein \omega_i \|\alpha(\nabla f_i(\alpha x^0 + (1 - \alpha)x^\ast_i) - \nabla f_i (\alpha x^\ast + (1 - \alpha)x^\ast_i))\|^2 \\
	& = \alpha^2\avein \omega_i \|\nabla f_i(\alpha x^0 + (1 - \alpha)x^\ast_i) - \nabla f_i (\alpha x^\ast + (1 - \alpha)x^\ast_i)\|^2\\
	& \leq \alpha^2\avein \omega_i L_i^2\|\alpha (x^0 -  x^\ast)\|^2 = \alpha^4 \overline{\omega_i L_i^2} \|x^0 - x^\ast\|^2.
	\end{align*}
	Plugging this result back into the definition of $\cD^0$ and applying further bounds we get
	\begin{align*}
	\cD^0  \leq & \left(1 + \frac{4}{n \min \beta_i} \gamma_0^2 \overline{\omega_i L_i^2} \right) \|x^0 - x^\ast\|^2\\
	 \overset{\eqref{eq:strong-conv-def}}{\leq} &\left(1 + \frac{4}{n \min \beta_i} \gamma_0^2 \overline{\omega_i L_i^2} \right) \frac{2}{\mu_\alpha} (\tilde{f}(x^0) - \tilde{f}(x^\ast))\\
	 \overset{\text{Proposition }~\ref{proposition:preserve-smooth-and-convexity}}{=}& \left(1 + \frac{4}{n \min \beta_i} \gamma_0^2 \overline{\omega_i L_i^2} \right) \frac{2}{\overline{\mu}\alpha^2} (\tilde{f}(x^0) - \tilde{f}(x^\ast)) \\
	 \overset{\eqref{eq:bounds-func-local-average}}{\leq} &\left(1 + \frac{4}{n \min \beta_i} \gamma_0^2 \overline{\omega_i L_i^2} \right) \frac{2}{\overline{\mu}\alpha^2} \cE_\alpha (x^0) \\
	  \overset{\eqref{eq:cE-def}}{=} & \left(1 + \frac{4}{n \min \beta_i} \gamma_0^2 \overline{\omega_i L_i^2} \right) \frac{1}{\overline{\mu}\alpha^2} \avein L_i \alpha^2 \|x^0 - x_i\|^2\\
	= &\left(1 + \frac{4}{n \min \beta_i} \gamma_0^2 \overline{\omega_i L_i^2} \right) \frac{1}{\overline{\mu}} \overline{L_i  \|x^0 - x_i\|^2},
	\end{align*}
	what shows that $\cD^0$ can be upper bounded by value \textit{independent} of $\alpha$. Since $M\gamma^2\sigma^2_k \geq 0$ $\cD^k = \|x^k - x^\ast\|^2 + M\gamma^2\sigma^2_k \geq \|x^k - x^\ast\|^2$, and that is why
	\begin{align*}
	\EE \|x^k - x^\ast\|^2 &\leq \EE \cD^k \leq \max \left\{(1 - \gamma \mu_\alpha)^k, \left(1 - \frac12 \min \beta_i\right)^k \right\} \cD^0\\
	& = \max \left\{(1 - \left(\gamma_0 \frac{1}{\alpha^2}\right) (\overline{\mu}\alpha^2))^k, \left(1 - \frac12 \min \beta_i\right)^k \right\} \cD^0\\
	& = \max \left\{(1 - \gamma_0 \overline{\mu})^k, \left(1 - \frac12 \min \beta_i\right)^k \right\} \cD^0\\
	& \leq \max \left\{(1 - \gamma_0 \overline{\mu})^k, \left(1 - \frac12 \min \beta_i\right)^k \right\} \left(1 + \frac{4}{n \min \beta_i} \gamma_0^2 \overline{\omega_i L_i^2} \right) \frac{1}{\overline{\mu}} \overline{L_i  \|x^0 - x_i\|^2},
	\end{align*}
	what means that the convergence of iterates is completely independent from $\alpha$. To derive the result for the convergence of functional values we use $L_\alpha$-smoothness of function $\tilde{f}$ (see Proposition~\ref{proposition:preserve-smooth-and-convexity}): $ \EE[\tilde{f}(x^k)] - \tilde{f}^\ast \leq \frac{L_\alpha}{2} \EE\|x^k - x^\ast\|^2 = \alpha^2 \frac{\overline{L_i}}{2} \EE \|x^k - x^\ast\|^2$.

\end{proof}

\subsection{Convex objectives}

\subsubsection{DCGD}
As for the previous section, we first introduce the algorithm, followed by general theorem, which is then applied to \eqref{eq:FLIX-problem}.

\begin{theorem}\label{thm:dcgd_cvx}
	Assume all conditions of Lemma~\ref{lm:CGD} hold. Let $0 < \gamma \leq \frac{1}{4 \left(L_\alpha + \frac{2 \max \{L_i \omega_i \alpha_i^2\}}{n}\right)}$. Then,
	\begin{equation}
	\begin{aligned}
	&\EE [\tilde{f}(\overline{x}^k) - \tilde{f}(x^\ast)] \leq \frac{2 (\tilde{f}(x^0) - \tilde{f}(x^\ast))}{k} +  \frac{\|x^0 - x^\ast\|^2}{\gamma k}+ 2\gamma \sigma_{DCGD},
	\end{aligned}
	\end{equation}
	where $\overline{x}^k = (1 / k) \sum\limits_{j=1}^k x^j$.
\end{theorem}

\begin{proof}
	The theorem is a corollary of Corollary 4.1 from~\cite{khaled2020unified} with constants
	\begin{align*}
	&A = L_\alpha~+~\frac{2 \max \{L_i \omega_i \alpha_i^2 \}}{n}, \ B = 0, \ \sigma_k^2 \equiv 0, \\
	&\ \rho = 1, \ C = 0, \ D_2 = 0,\\
	& D_1 = \sigma_{DCGD} = \frac{2}{n^2} \sumin \omega_i \|\nabla [f_i (T_i(x^\ast))]\|^2\\
	& = \frac{2}{n^2} \sumin \omega_i \alpha_i^2 \|\nabla f_i (T_i(x^\ast))\|^2.
	\end{align*}
\end{proof}

\begin{corollary}\label{crl:cgd_cvx}
Assume all conditions of Lemma~\ref{lm:CGD} hold. Let $\alpha_i = \alpha$, $\omega_i = \omega$ for all $i$, $\gamma = \frac{1}{4 \left(\overline{L_i} + \frac{2\omega \max_i L_i}{n} \right) \alpha^2}$ and $\sup\limits_{\alpha \in [0, 1]} \inf\limits_{x^\ast \in \cX} \|x^0 - x^\ast\| = R$. Then,
\begin{align*}
&\EE[\tilde{f}(\overline{x}^k) - \tilde{f}(x^\ast)]\\
& \leq \alpha^2 \cdot \frac{2}{k} \left(\overline{L_i \|x^0 - x_i\|^2} + 2 \left(\overline{L_i} + \frac{2 \max\{L_i \omega_i\}}{n} \right) R^2   \right) \\
& \qquad + \alpha \cdot \frac{\omega\max_i L_i}{n\overline{L_i} + 2 \omega\max_i L_i} \left(\widetilde{f}(x^\ast(0)) - \avein f_i(x_i)\right)\\
& = \frac{1}{k} C_1 \alpha^2 + C_2 \alpha.
\end{align*}
\end{corollary}

\begin{proof}
Let us investigate $\tilde{f}(x^0) - \tilde{f}(x^\ast)$. Let $\alpha_i \equiv \alpha \in \ \RR$ for all $i$. Since $\min \avein f_i(T_i(x)) \geq \avein \min f_i(T_i(x)) = \avein \min f_i(x)$ and each $f_i$ is $L_i$-smooth, we get
	\begin{align*}
	&\tilde{f}(x^0) - \tilde{f}(x^\ast) \leq \tilde{f}(x^0) - \avein f_i(x_i)\\
	&=\avein f_i(\alpha x^0 + (1 - \alpha) x_i) - f(x_i)\\
	&\overset{\eqref{eq:smoothness-def}}{\leq} \avein L_i \|\alpha (x^0 - x_i)\|^2\\
	&=\alpha^2 \overline{L_i \|x^0 - x_i\|^2}.
	\end{align*}
	This observation shows that the dependence between $\tilde{f}(x^0) - \tilde{f}(x^\ast)$ and $\alpha$ is quadratic. $\tilde{f}(x^0) - \tilde{f}(x^\ast)$ diminishes to zero as $\alpha$ goes to zero what proves our intuition: the closer $\alpha$ is to zero, the less steps the algorithm needs to make till convergence.

	We verify that the stepsize in the corollary fits the restriction on stepsizes in Theorem~\ref{thm:dcgd_cvx}.
	\begin{equation}
	\begin{aligned}\label{eq:gamma_cgd}
	&\frac{1}{\gamma} = 4 \left(L_\alpha + \frac{2 \max \{L_i \omega_i \alpha_i^2\}}{n}\right)\\
	&  \overset{\alpha_i \equiv \alpha, \ \text{Proposition }\ref{proposition:preserve-smooth-and-convexity}}{=} 4 \left(\overline{L_i} + \frac{2 \max\{L_i \omega_i\}}{n} \right) \cdot  \alpha^2.
	\end{aligned}
	\end{equation}
	Further assuming that $\inf\limits_{x^\ast \in \cX} \|x^0 - x^\ast\|$ is bounded above for all $\alpha \in [0, 1]$ by value $R \in \RR$, we obtain the following result:
	\begin{align*}
	\frac{1}{\gamma} \|x^0 - x^\ast\|^2 \leq 4 \left(\overline{L_i} + \frac{2 \max\{L_i \omega_i\}}{n} \right) R^2 \alpha^2.
	\end{align*}

	The neighborhood governing term is $\gamma \sigma_{DCGD}$. Let us assume $\alpha_i \equiv \alpha$ and $\omega_i \equiv \omega$ for all $i$.
	\begin{align*}
	&\sigma_{DCGD} = \frac{2}{n^2} \sumin \omega_i \alpha_i^2 \|\nabla f_i (T_i(x^\ast))\|^2\\
	&\overset{T_i(x) \ \text{def}}{=} \frac{2}{n^2} \sumin \omega_i \alpha_i^2 \|\nabla f_i (\alpha_i x^\ast + (1 - \alpha_i) x_i)\|^2\\
	&\overset{\alpha_i \equiv \alpha, \ \omega_i \equiv \omega}{=} \frac{2\omega\alpha^2}{n^2}\sumin \|\nabla f_i (\alpha_i x^\ast + (1 - \alpha_i) x_i)\|^2\\
	\end{align*}
	Since $x^\ast_i$ minimizes $f_i(x)$, $\nabla f_i(x^\ast_i) = 0$. Thus,
	\begin{align*}
	&\frac{2\omega\alpha^2}{n^2}\sumin \|\nabla f_i (\alpha_i x^\ast + (1 - \alpha_i) x_i)\|^2\\
	&= \frac{2\omega\alpha^2}{n^2}\sumin \|\nabla f_i (\alpha_i x^\ast + (1 - \alpha_i) x_i) - \nabla f_i(x_i) \|^2\\
	\end{align*}
	Inequality \eqref{ineq:cvx_smooth} further bounds each norm of gradient differences.
	\begin{align*}
	&\frac{2\omega\alpha^2}{n^2}\sumin \|\nabla f_i (\alpha_i x^\ast + (1 - \alpha_i) x_i) - \nabla f_i(x_i) \|^2\\
	&\leq \frac{4\omega\alpha^2}{n^2}\sumin L_i (f_i(\alpha_i x^\ast + (1 - \alpha_i) x_i) - f_i(x_i))\\
	&\leq \frac{4\omega\alpha^2 \max_i L_i}{n^2}\sumin (f_i(\alpha_i x^\ast + (1 - \alpha_i) x_i) - f_i(x_i))\\
	&= \frac{4\omega\alpha^2\max_i L_i}{n} \left(\tilde{f}(x^\ast) - \avein f_i(x_i)\right)
	\end{align*}
	In the rightmost term point $x^\ast$ depends on $\alpha$ because it is an optimal point of a problem dependent on $\alpha$. But the observation below explicitly reveals the dependence. Since each $f_i$ is convex,
	\begin{align*}
	&\tilde{f}(x) = \avein f_i(\alpha x + (1 - \alpha) x_i)\\
	&\leq \alpha \avein f_i(x) + (1 - \alpha) \avein f_i(x_i).
	\end{align*}
	Taking the minimum of both sides of inequality above we get to
	\begin{align*}
	\tilde{f}(x^\ast) \leq \alpha \tilde{f}(x^\ast(0)) + (1 - \alpha) \avein f_i(x_i).
	\end{align*}
	We established that
	\begin{align}\label{ineq:sigma_cgd_cvx}
	\sigma_{DCGD} \leq \frac{4\omega\max_i L_i}{n} \left(\tilde{f}(x^\ast(0))- \avein f_i(x_i)\right) \alpha^3.
	\end{align}
	Together with equality~\ref{eq:gamma_cgd} it brings us to the following bound on the neighbourhood:
	\begin{align*}
	&\gamma \sigma_{DCGD} \leq \frac{\omega\max_i L_i}{n\overline{L_i} + 2 \omega\max_i L_i} \left(\tilde{f}(x^\ast(0)) - \avein f_i(x_i)\right) \alpha = C \alpha,
	\end{align*}
	where $C$ does not depend on $\alpha$.

\end{proof}

\subsubsection{DIANA}
As for the previous section, we first introduce the algorithm, followed by general theorem, which is then applied to \eqref{eq:FLIX-problem}.

\begin{theorem}
	Assume all conditions of Lemmas~\ref{lm:DIANA} and~\ref{lm:DIANA_sigma_k} hold. Let $ 0 < \gamma \leq \frac{1}{4 \left(L_\alpha + \frac{2\max_i\{L_i \alpha_i^2 \omega_i \}}{n} + \frac{4\max_i \{\beta_i \omega_i L_i \alpha_i^2 \}}{n \min_i \beta_i} \right)}$. Then,
	\begin{equation}
	\begin{aligned}
	&\EE [\tilde{f}(\overline{x}^k) - \tilde{f}(x^\ast)] \leq \left( 2 (\tilde{f}(x^0) - \tilde{f}(x^\ast)) + \frac{4\gamma}{n\min \beta_i} \sigma_0^2 + \frac{1}{\gamma}\|x^0 - x^\ast\|^2 \right) \frac{1}{k}.
	\end{aligned}
	\end{equation}

\end{theorem}

\begin{proof}
	The theorem is a corollary of Corollary 4.1 from~\cite{khaled2020unified} with constants
	\begin{align*}
	&A = L_\alpha~+~\frac{2 \max \{L_i \omega_i \alpha_i^2\}}{n},  B = \frac{2}{n}, D_1 = 0, \\ &\rho = \min_i \beta_i, C = \max_i \{\beta_i \omega_i L_i \alpha_i^2\}, \ D_2 = 0.
	\end{align*}
\end{proof}

\begin{corollary}\label{crl:diana_cvx}
Assume all conditions of Lemmas~\ref{lm:DIANA} and~\ref{lm:DIANA_sigma_k} hold. Let $\alpha_i = \alpha$, $\omega_i = \omega$, $h^0_i = 0 \in \RR^d$ for all $i$, $\gamma = \frac{1}{4 \left(\overline{L_i} + \frac{6 \omega \max_i L_i}{n} \right) \alpha^2}$, $\sup\limits_{\alpha \in [0, 1]} \inf\limits_{x^\ast \in \cX} \|x^0 - x^\ast\| = R$. Then,
\begin{align*}
\EE[f(\overline{x}^k) - f(x^\ast)] \leq \frac{1}{k}(C_1 \alpha^2 + C_2 \alpha),
\end{align*}
where $C_1 = 2\left(\overline{L_i \|x^0 - x_i\|^2} + 2 \left(\overline{L_i} + \frac{6\omega\max_i L_i}{n} \right) R^2 \right)$ and $C_2 = \frac{\omega(\omega + 1) \max_i L_i }{(n \overline{L_i} + 6\omega\max_i L_i )} \left(\tilde{f}(x^\ast(0)) - \avein f_i(x_i)\right)$.
\end{corollary}
\begin{proof}
Let us investigate stepsize $\gamma$. The statement of the corrolary requires the stepsize equals its maximum value $\frac{1}{4 \left(L_\alpha + \frac{2\max_i\{L_i \alpha_i^2 \omega_i \}}{n} + \frac{4\max_i \{\beta_i \omega_i L_i \alpha_i^2 \}}{n \min_i \beta_i} \right)}$. Since $\alpha_i = \alpha, \ \omega_i = \omega, \ \beta_i = \frac{1}{1 + \omega}$ for all $i$,
\begin{equation}\label{eq:tmp_gamma_cvx_cgd}
\begin{aligned}
&\gamma = \frac{1}{4 \left(L_\alpha + \frac{2\max_i\{L_i \alpha_i^2 \omega_i \}}{n} + \frac{4\max_i \{\beta_i \omega_i L_i \alpha_i^2 \}}{n \min_i \beta_i} \right)} = \frac{1}{4 \left(\overline{L_i} + \frac{6\omega\max_i L_i}{n} \right)\alpha^2}.
\end{aligned}
\end{equation}

From definitions it follows that $\sigma_0^2 = \frac{n}{2}\sigma_{DCGD}$ if $h^0_i = 0 \in \RR^d$. Thus,
\begin{align*}
&\sigma^2_0 = \frac{n}{2}\sigma_{DCGD} \overset{\eqref{ineq:sigma_cgd_cvx}}{\leq} 2\omega\max_i L_i \left(\tilde{f}(x^\ast(0)) - \avein f_i(x_i)\right) \alpha^3.
\end{align*}
That is why
\begin{align*}
&\frac{4\gamma}{n\min \beta_i} \sigma_0^2  \overset{\beta_i = \frac{1}{\omega} + 1}{=} \frac{4\gamma(\omega + 1)}{n} \sigma^2_0 \overset{\eqref{eq:tmp_gamma_cvx_cgd}}{=}\frac{\omega + 1}{n  \left(\overline{L_i} + \frac{6\omega\max_i L_i}{n} \right)\alpha^2} \sigma^2_0\\
&=\frac{\omega + 1}{(n \overline{L_i} + 6\omega\max_i L_i )\alpha^2} \sigma^2_0
\leq \frac{2\omega(\omega + 1) \max_i L_i }{(n \overline{L_i} + 6\omega\max_i L_i )} \left(\tilde{f}(x^\ast(0)) - \avein f_i(x_i)\right) \alpha.
	\end{align*}

	Remaining terms are analyzed as for DCGD in convex case.
\end{proof}

\subsection{Nonconvex objectives}

In this section we analyze algorithms in the nonconvex case of functions. Throughout this section we assume that $x^\ast_i$ is a stationary point of function $f_i$. For the sake of convenience we transfer Lemma 1 from~\cite{khaled2020better} to here.

\begin{lemma}[Lemma 1 from~\cite{khaled2020better}]\label{lm:non_cvx_bounded_gradient}
Suppose $f^\ast_i = \min f_i(x)$ exists and $f_i$ is $L_i$-smooth. Then, for any $x \in \RR^d$ it holds that
$$
\|\nabla f_i(x)\|^2 \leq 2 L_i (f_i(x) - f^\ast_i).
$$
\end{lemma}

Throughout this section we will use notation $f^\ast \eqdef \avein f^\ast_i$.
\subsubsection{DCGD}

\begin{theorem}\label{thm:non_cvx_dcgd}
Suppose $f^\ast_i = \min f_i(x)$ exists and $f_i$ is $L_i$-smooth for each $i$. Then, if $\gamma \leq \frac{1}{L_\alpha}$,
\begin{align}
\min_{0 \leq t \leq k - 1} \EE \|\nabla \tilde{f}(x^t)\|^2 \leq \frac{2\left(1 +  \frac{2L_\alpha \gamma^2\max\{L_i \omega_i \alpha_i^2\}}{n}\right)^k}{\gamma k} (\tilde{f}(x^0) - f^\star).
\end{align}
\end{theorem}

\begin{proof}
The proof is a direct application of Theorem 2 in~\citep{khaled2020better} to our setting.
First, let us analyze $\EE \|g(x)\|^2$.
\begin{align*}
&\EE \|g(x)\|^2 = \EE \left\|\avein \cC_i(\nabla [f_i(T_i(x))])\right\|^2\\
& = \EE \left\|\avein \cC_i(\nabla [f_i(T_i(x))]) - \EE \left[ \avein \cC_i(\nabla [f_i(T_i(x))]) \right]\right\|^2 + \left\| \EE \left[ \avein \cC_i(\nabla [f_i(T_i(x))]) \right]\right\|^2\\
& = \EE \left\|\avein \cC_i(\nabla [f_i(T_i(x))]) -  \avein \nabla [f_i(T_i(x))] \right\|^2 + \left\|\avein \nabla f_i(T_i(x))\right\|^2\\
& =  \EE \left\|\avein \cC_i(\nabla [f_i(T_i(x))]) -  \avein \nabla [f_i(T_i(x))] \right\|^2 + \left\|\nabla \tilde{f}(x)\right\|^2\\
& = \frac{1}{n^2} \sumin \|\cC_i(\nabla [f_i(T_i(x))]) - \nabla [f_i(T_i(x))] \|^2 + \left\|\nabla \tilde{f}(x)\right\|^2\\
& \leq \frac{1}{n^2} \sumin \omega_i \|\nabla [f_i(T_i(x))] \|^2 +\left\|\nabla \tilde{f}(x)\right\|^2.
\end{align*}
To analyze the first term we refer ourselves to Lemma~\ref{lm:non_cvx_bounded_gradient}:
\begin{align*}
&\|\nabla [f_i(T_i(x))] \|^2 \overset{T_i\text{ def}}{=} \|\nabla [f_i(\alpha_i x + (1 - \alpha_i)x_i)] \|^2 = \alpha_i^2\|\nabla f_i(\alpha_i x + (1 - \alpha_i)x_i) \|^2\\
&\overset{\text{Lemma }\ref{lm:non_cvx_bounded_gradient}}{\leq} \alpha_i^2 2 L_i (f_i(\alpha_i x + (1 - \alpha_i)x_i) - f^\ast_i)
\end{align*}
Plugging this in the previous inequality we get
\begin{align*}
&\EE \|g(x)\|^2 \leq \frac{1}{n^2} \sumin \omega_i \|\nabla [f_i(T_i(x))] \|^2 +\left\|\nabla \tilde{f}(x)\right\|^2\\
& \leq \frac{1}{n^2} \sumin \omega_i \alpha_i^2 2 L_i (f_i(\alpha_i x + (1 - \alpha_i)x_i) - f^\ast_i) +\left\|\nabla \tilde{f}(x)\right\|^2\\
&\leq \frac{2\max\{L_i \omega_i \alpha_i^2\}}{n} \left(\avein f_i(\alpha_i x + (1 - \alpha_i)x_i) - \avein f^\ast_i \right) + \|\nabla \tilde{f}(x)\|^2\\
& = \frac{2\max\{L_i \omega_i \alpha_i^2\}}{n} (\tilde{f}(x) - f^\ast) + \|\nabla \tilde{f}(x)\|^2,
\end{align*}
what means that Assumption 2 from~\cite{khaled2020better} holds with $A = \frac{2\max\{L_i \omega_i \alpha_i^2\}}{n}, \ B = 1, \ C = 0$. Then, as Theorem 2 from the same paper states, for stepsize $ \gamma \leq \frac{1}{L_\alpha}$ it holds that
\begin{align*}
\min_{0 \leq t \leq k - 1} \EE \|\nabla \tilde{f}(x^t)\|^2 \leq \frac{2\left(1 +  \frac{2L_\alpha \gamma^2\max\{L_i \omega_i \alpha_i^2\}}{n}\right)^k}{\gamma k} (\tilde{f}(x^0) - f^\star).
\end{align*}
\end{proof}

\begin{corollary}\label{crl:dcgd_non_convex}
Suppose assumptions of Theorem~\ref{thm:non_cvx_dcgd} hold, $\alpha_i \equiv \alpha$, $\omega_i \equiv \omega$. Let $\gamma = \gamma_0 \frac{1}{\alpha^2} \leq \frac{1}{\overline{L_i} \alpha^2}$ and $\Delta_0 = \sup\limits_\alpha \tilde{f}(x^0) - f^\ast$.
\begin{align}
\min_{0 \leq t \leq k - 1} \EE \|\nabla \tilde{f}(x^t)\|^2 \leq \alpha^2 \frac{2\left(1 +  \frac{2\overline{L} \omega \gamma_0^2\max_i L_i}{n}\right)^k}{\gamma_0 k} \Delta_0.
\end{align}
\end{corollary}

\begin{proof}
We note that when $\alpha_i \equiv \alpha$, $L_\alpha = \overline{L}\alpha^2$, which follows from the definition of $L_\alpha$. Plugging equations $L_\alpha = \overline{L}\alpha^2$, $\gamma = \gamma_0 \frac{1}{\alpha^2}$, $\alpha_i \equiv \alpha$ and $\omega_i \equiv \omega$ into the main result of Theorem~\ref{thm:non_cvx_dcgd} we get
\begin{align*}
\min_{0 \leq t \leq k - 1} \EE \|\nabla \tilde{f}(x^t)\|^2 \leq \alpha^2 \frac{2\left(1 +  \frac{2\overline{L}\omega \gamma_0^2\max_i L_i }{n}\right)^k}{\gamma_0 k} (\tilde{f}(x^0) - f^\star).
\end{align*}
Noting that $\tilde{f}(x^0) - f^\ast \leq \Delta_0$ we finish the proof.
\end{proof}

\subsubsection{DIANA}

\begin{theorem}\label{thm:diana_nncvx}
Suppose $f^\ast_i = \min f_i(x)$ exists and $f_i$ is $L_i$-smooth for each $i$. Let $\beta_k \equiv \beta$ and $\omega_k \equiv \omega$ for all $k \in \{1, \dots, n\}$. Suppose $\beta \in \frac{1}{1 + \omega} \left[1, \min \left\{1 + \frac{1}{2(1 + \omega)}, \sqrt{\frac32 + \omega + \frac{1}{2\omega}} \right\} \right]$ and $\gamma~\leq~\frac{1}{L_\alpha} \cdot \min \left\{\frac{1}{2\sqrt{\eta_0}} , \frac{2}{\sqrt{1 + 8\eta_0} + 1}  \right\} $, $\eta_0 = \frac{(1+\omega)\omega (1 + 2(1+\omega))}{n}$, and $h^0_i = \nabla [f_i(T_i(x^0))]$. Let $\hat{x}$ be a point chosen uniformly at random among iterates $x^0, x^1, \dots, x^{k-1}$ generated by DIANA. Then,
\begin{align}
\EE \|\nabla \tilde{f}(\hat{x})\|^2 \leq \frac{2}{\gamma k} (\tilde{f}(x^0) - f^\ast).
\end{align}
\end{theorem}

\begin{proof}
The proof is following the proof of Theorem 1 in~\cite{li2020unified} with small modifications: we rewrite the main proposition in the different way and give slightly different bound on the stepsize. First, Lemma 9 in~\cite{li2020unified} says that
\begin{align}
\label{ineq:g_k_diana_nncvx}&\EE \|g^k\|^2 \leq \|\nabla f(x^k)\|^2 + \frac{1 + \omega}{n}\sigma_k^2,\\
\label{ineq:sigma_k_diana_nncvx} &\EE\sigma^2_{k+1} \leq (1 - \rho) \sigma^2_k + \frac{\omega (1 + r) L_\alpha^2 \gamma^2}{1 + \omega}\|\nabla f(x)\|^2,
\end{align}
where
\begin{align}
&\sigma_k^2 = \frac{\omega}{(1 + \omega) n}\sumin \|\nabla [f_i(T_i(x^k))] - h_i^k \|^2\\
\label{eq:rho_def_diana_nncvx}&\rho = \theta - \frac{\omega (1 + r)}{n} L_\alpha^2 \gamma^2\\
&\theta = \min\{1 - \beta^2 \omega, 2 \beta - \frac{1 - \beta}{r} - \beta^2 - \beta^2\omega\}
\end{align}
 and $r$ is an arbitrary positive number. We choose $r = 2(1 + w) > 0$. For the sake of readability, we define function $\zeta$ in the following way
\begin{align*}
&2 \beta - \frac{1 - \beta}{2 ( 1 + \omega)} - \beta^2 - \beta^2\omega = - (\omega + 1) \beta^2 + \left(2 + \frac{1}{2(1 + \omega)}\right)\beta - \frac{1}{2(1 + \omega)} \eqqcolon \zeta(\beta)\\
\end{align*}
That is why we can write that $\theta \in \min \{1 - \beta^2 \omega, \zeta(\beta) \}$.

Second, let us analyze the lower bound for $\theta$. Note that for $\beta$ lying in the range $\frac{1}{1 + \omega} \left[1, \min \left\{1 + \frac{1}{2(1 + \omega)}, \sqrt{\frac32 + \omega + \frac{1}{2\omega}} \right\} \right]$ we have
\begin{align*}
&1 - \beta^2 \omega \geq 1 - \frac{\omega}{(1 + \omega)^2} \left(\frac32 + \omega + \frac{1}{2\omega} \right) = 1 - \frac{\omega}{(1 + \omega)^2} \frac{(1 + 2\omega)(1 + \omega)}{2\omega} \\
&= 1 - \frac{1 + 2\omega}{2(1 + \omega)} = \frac{1}{2(1 + \omega)},
\end{align*}
and
\begin{align*}
&\zeta\left(\frac{1}{1 + \omega}\right) = - \frac{1}{1 + \omega} + \left(2 + \frac{1}{2(1 + \omega)}\right) \frac{1}{1 + \omega} - \frac{1}{2(1 + \omega)} \\
& = \frac{1}{2(1 + \omega)} + \frac{1}{2(1 + \omega)^2} \geq \frac{1}{2(1 + \omega)},\\
\end{align*}
and
\begin{align*}
&\zeta\left(\frac{1}{1 + \omega} + \frac{1}{2(1+\omega)^2}\right) = - \frac{1}{1 + \omega} - \frac{1}{(1+\omega)^2} - \frac{1}{4(1+\omega)^3} \\
& \qquad + \left(2 + \frac{1}{2(1 + \omega)}\right) \left(\frac{1}{1 + \omega} + \frac{1}{2(1+\omega)^2}\right) - \frac{1}{2(1 + \omega)}\\
& = \frac{1}{2(1 + \omega)} + \frac{1}{2(1 + \omega)^2} \geq \frac{1}{2(1 + \omega)}.\\
\end{align*}
Since $\zeta$ is a quadratic function with the ends of the parabola pointed downwards, for all $\beta$ in the range we have $\zeta(\beta) \geq \frac{1}{2(1 + \omega)}$, which altogether means that $\theta \geq \frac{1}{2(1 + \omega)} > 0$.

If we enforce
\begin{align*}
\frac{\omega (1 + r)}{n} L_\alpha^2 \gamma^2 \leq \frac{1}{4(1 + \omega)}
\end{align*}
or
\begin{align*}
\gamma \leq \frac{1}{2L_\alpha} \sqrt{\frac{n}{(1+\omega)\omega (1 + r)}} = \frac{1}{2L_\alpha} \sqrt{\frac{n}{(1+\omega)\omega (1 + 2(1+\omega))}},
\end{align*}
then $\rho = \theta - \frac{\omega (1 + r)}{n} L_\alpha^2 \gamma^2 \geq \frac{1}{4 ( 1 + \omega )}$.

Now let us switch to the proof of the convergence. We first note that due to smoothness of $f_i$s  and inequality~\ref{ineq:g_k_diana_nncvx} we get
\begin{align*}
\EE f(x^{k+1}) &\leq f(x^k) + \EE\langle \nabla f(x^k), x^{k+1} - x^k\rangle + \frac{L_\alpha}{2} \EE\|x^{k+1} - x^k\|^2\\
& = f(x^k) - \gamma \| \nabla f(x^k)\|^2 + \frac{L_\alpha \gamma^2}{2} \EE\|g^k\|^2\\
& \leq f(x^k) - \left(\gamma - \frac{L_\alpha \gamma^2}{2}\right) \| \nabla f(x^k)\|^2 +  \frac{L_\alpha \gamma^2}{2}\frac{1 + \omega}{n}\sigma_k^2.
\end{align*}
Let us fix $\xi > 0$. Then according to~\ref{ineq:sigma_k_diana_nncvx} we have
\begin{align*}
&\EE [f(x^{k+1}) - f^\ast + \xi \sigma_{k+1}^2] \\
& \leq f(x^k) - f^\ast - \left(\gamma - \frac{L_\alpha \gamma^2}{2} - \xi\frac{\omega (1 + r) L_\alpha^2 \gamma^2}{1 + \omega}\right) \| \nabla f(x^k)\|^2 +  \left(\xi(1 - \rho) + \frac{L_\alpha \gamma^2}{2}\frac{1 + \omega}{n}\right)\sigma_k^2.
\end{align*}
Let us notate $\Delta^k = f(x^k) - f^\ast + \xi \sigma_k^2$ and set $\xi = \frac{L_\alpha \gamma^2}{2\rho}\frac{1 + \omega}{n}$. Then from previous inequality it follows that
\begin{align}\label{ineq:tmp_delta_k}
\EE \Delta^{k+1} \leq \Delta^k - \left(\gamma - \frac{L_\alpha \gamma^2}{2} - \xi\frac{\omega (1 + r) L_\alpha^2 \gamma^2}{1 + \omega}\right) \| \nabla f(x^k)\|^2.
\end{align}
Let us define $\gamma'$ the coefficient in front of $\|\nabla f(x^k)\|^2$ in the last inequality. When is $\gamma'$ larger than $\frac{\gamma}{2}$?
\begin{align*}
&\gamma - \frac{L_\alpha \gamma^2}{2} - \xi\frac{\omega (1 + r) L_\alpha^2 \gamma^2}{1 + \omega} = \gamma - \frac{L_\alpha \gamma^2}{2} - \frac{L_\alpha \gamma^2}{2\rho}\frac{1 + \omega}{n}\frac{\omega (1 + r) L_\alpha^2 \gamma^2}{1 + \omega}\\
& = \gamma - \frac{L_\alpha \gamma^2}{2} - \frac{L_\alpha \gamma^2}{2\rho}\frac{\omega (1 + r) L_\alpha^2 \gamma^2}{n} = \gamma - \frac{\gamma}{2} \left(L_\alpha \gamma + \frac{L_\alpha^3 \gamma^3 \omega (1 + r)}{\rho n} \right) \geq \frac{\gamma}{2},
\end{align*}
if
\begin{align*}
L_\alpha \gamma + \frac{L_\alpha^3 \gamma^3 \omega (1 + r)}{\rho n} \leq 1.
\end{align*}
Let us define $y \coloneqq L_\alpha \gamma$ and $a \coloneqq \frac{\omega (1 + r)}{n}$. Then according to~\eqref{eq:rho_def_diana_nncvx} $\rho = \theta - a y^2$. Plugging this back to the last inequality we get
\begin{align*}
&y + \frac{ay^3}{\theta - ay^2} \leq 1\\
& \Leftrightarrow \frac{y\theta - ay^3 + ay^3}{\theta - ay^2} \leq 1\\
& \Leftrightarrow \frac{y\theta}{\theta - ay^2} \leq 1\\
& \overset{\theta - ay^2 = \rho > 0}{\Leftrightarrow} y \theta \leq \theta - ay^2\\
& \overset{\theta > 0}{\Leftrightarrow} \frac{a}{\theta} y^2 + y - 1 \leq 0.
\end{align*}
The last inequality holds for all $y$ lying between zero and the positive root of the quadratic equation. The positive root is $\frac{\sqrt{1 + \frac{4a}{\theta}} - 1}{\frac{2a}{\theta}} = \frac{\frac{4a}{\theta}}{\frac{2a}{\theta} \left(\sqrt{1 + \frac{4a}{\theta}} + 1\right)} = \frac{2}{\sqrt{1 + \frac{4a}{\theta}} + 1}$. That means that the last inequality holds if
\begin{align*}
y \leq \frac{2}{\sqrt{1 + \frac{4a}{\theta}} + 1}
\end{align*}
or the tightest bound for $\gamma$ achieved on smallest $\theta$ gives
\begin{align*}
\gamma \leq \frac{1}{L_\alpha}\frac{2}{\sqrt{1 + \frac{4\omega (1 + 2 ( 1 + \omega))}{n\theta_{\text{min}}}} + 1} = \frac{1}{L_\alpha}\frac{2}{\sqrt{1 + \frac{8 (1 + \omega) \omega (1 + 2 ( 1 + \omega))}{n}} + 1}.
\end{align*}
Going back to~\ref{ineq:tmp_delta_k} we finally get
\begin{align*}
&\EE \Delta^{k+1} \leq \Delta^k - \frac{\gamma}{2}\|\nabla f(x^k)\|^2\\
&\EE \|\nabla f(x^k)\|^2 \leq \frac{2}{\gamma} (\EE \Delta^{k} - \EE\Delta^{k+1}).
\end{align*}
That is why
\begin{align*}
\EE \|\nabla f(\hat{x})\|^2 = \frac{1}{k} \sum_{t = 0}^{k - 1} \EE \|\nabla f(x^t)\|^2 \leq \frac{2}{\gamma k} (\Delta^0 - \EE \Delta^k) \leq \frac{2}{\gamma k} \Delta^0.
\end{align*}

The proof is written for general function $f(x) = \avein f_i(x)$. To apply the theorem for our setting, we replace $f$ by $\tilde{f}$.
\end{proof}

\begin{corollary}\label{crl:diana_nncvx}
Suppose all conditions of Theorem~\ref{thm:diana_nncvx} hold and $\alpha_i \equiv \alpha$ for all $i$ from $\{1, \dots, n\}$. Let $\gamma = \frac{1}{\alpha^2} \gamma_0$, where $\gamma_0 = \frac{1}{\overline{L}}\min \left\{\frac{1}{2\sqrt{\eta_0}} , \frac{2}{\sqrt{1 + 8\eta_0} + 1}  \right\}$. Let $\Delta_0 = \sup\limits_\alpha \tilde{f}(x^0) - f^\ast$ Then,
\begin{align}
\EE \|\nabla \tilde{f}(\hat{x})\|^2 \leq \alpha^2 \frac{2}{\gamma_0 k} \Delta_0.
\end{align}
\end{corollary}

\begin{proof}
According to the definition of $L_\alpha = \avein L_i \alpha_i^2$, in the case of equal $\alpha_i$s we get $L_\alpha = \overline{L}\alpha^2$. It remains to notice that the stepsize in the corollary condition still satisfies the condition on the stepsize $\gamma$ in Theorem~\ref{thm:diana_nncvx} and $\Delta_0 \geq \tilde{f}(x) - f^\ast$.
\end{proof}

\clearpage
\section{DISCUSSION OF OTHER MODEL MIXTURE METHODS}
\label{sec:rel-work-discussion}

\citet{expl_mixture_Deng2020} propose that each client solves the local problem $$\min_{v \in \R^d} f_i (\alpha_i v + (1 - \alpha_i) w^\ast),$$ where $w^\ast$ is the minimizer of~\eqref{eq:ERM}. This does not result in any personalization since the ``personalized'' solution on each node is a reparameterization of each local solution $x_i$. Furthermore, the convergence theory that \citet{expl_mixture_Deng2020} develop does not recover the linear convergence of gradient descent at $\alpha_i =1$.

\citet{mansour2020approaches} introduce a similar method, MAPPER, where they propose to solve $$\min_{z, \alpha_i, y_i} \frac{1}{n} \sum \limits_{i=1}^{n} f_{i} (\alpha_i y_i + (1 - \alpha_i) z).$$ Again, this objective is trivially minimized by setting $\alpha_i = 1$, $y_i = \min_{x \in \R^d} f_i (x)$, and $z = 0$ (i.e.\ with no personalization at all).

\citet{zec2021specialized} also introduce a similar formulation based on the mixture of experts framework, where they propose to first learn the minimizer $x_\ast$ of \eqref{eq:ERM}, learn the optimal local models $x_1, x_2, \ldots, x_n$, and then learn a mixture of both the global and local models (i.e. the $\alpha_i$) on each client. Unfortunately, this is also ill-defined, as $\alpha_i = 1$ will always perform best on the local training set, and hence if there is no additional data the optimization process cannot improve over the local minimizers $x_1, \ldots, x_n$.

\clearpage
\section{EXPERIMENTAL DETAILS}
\label{sec:exp_details}

{\bfseries Generalization experiment 1: Fitting Sine Functions.} We train each local model $i$ until the gradient norm $\|\nabla f_i(x)\| $ is below $10^{-2}$. Stopping criteria for global models is the same but with respect to FLIX formulation:  $\|\avein \nabla f_i(T_i(x))\| < 10^{-2}$. Gradient descent with line search at each iteration was used as a local and a meta optimizer.

{\bfseries Generalization experiment 2: Comparison to FOMAML and Reptile.} We preprocess the raw data in the same way as in~\cite{reddi2021adaptive} for Stackoverflow logistic regression task, i.e., each feature vector is a bag-of-words representation of a user's sentence, each label vector is a binary vector showing if a sentence relates to a particular question tag or not. Word vocabulary for the feature dataset is restricted to the 10000 most common words. We restrict the task to the 500 most used tags. For preprocessing, we used Tensorflow computational procedures from~\cite{adaptivefedopt_github}.

To select 50 clients for the Figure~\ref{fig:alphai-matters}, we map first 5000 clients from train dataset to vector space with BERT~\cite{bert} and run k-means with 10 clusters. The first 50 clients from the first cluster have been selected for the experiment.

The hold-out validation dataset is of size 100, for the first experiment, and of size 110, for the second one. To compute pure local models we run gradient descent until the norm of the gradient is less than $[10^{-1}, 10^{-2}, 10^{-3}, 10^{-4}, 10^{-5}]$, respectively. We train pure local models until the gradient norm is less than $10^{-4}$, for the first experiments, and $10^{-5}$, for the second, as this tolerance level was observed to have the lowest generalization error on the validation dataset in our experiments.

Then, we use a grid-search to find optimal stepsizes for gradient descent used for training FLIX. Our grid for step-sizes are $[5 \cdot 10^{-2}, 5 \cdot 10^{-1}, \dots , 5 \cdot 10^{6}]$. The results are presented in Figures~\ref{fig:grid_search} and~\ref{fig:grid_search2}.
\begin{figure*}[t]
	\includegraphics[width=0.87\linewidth]{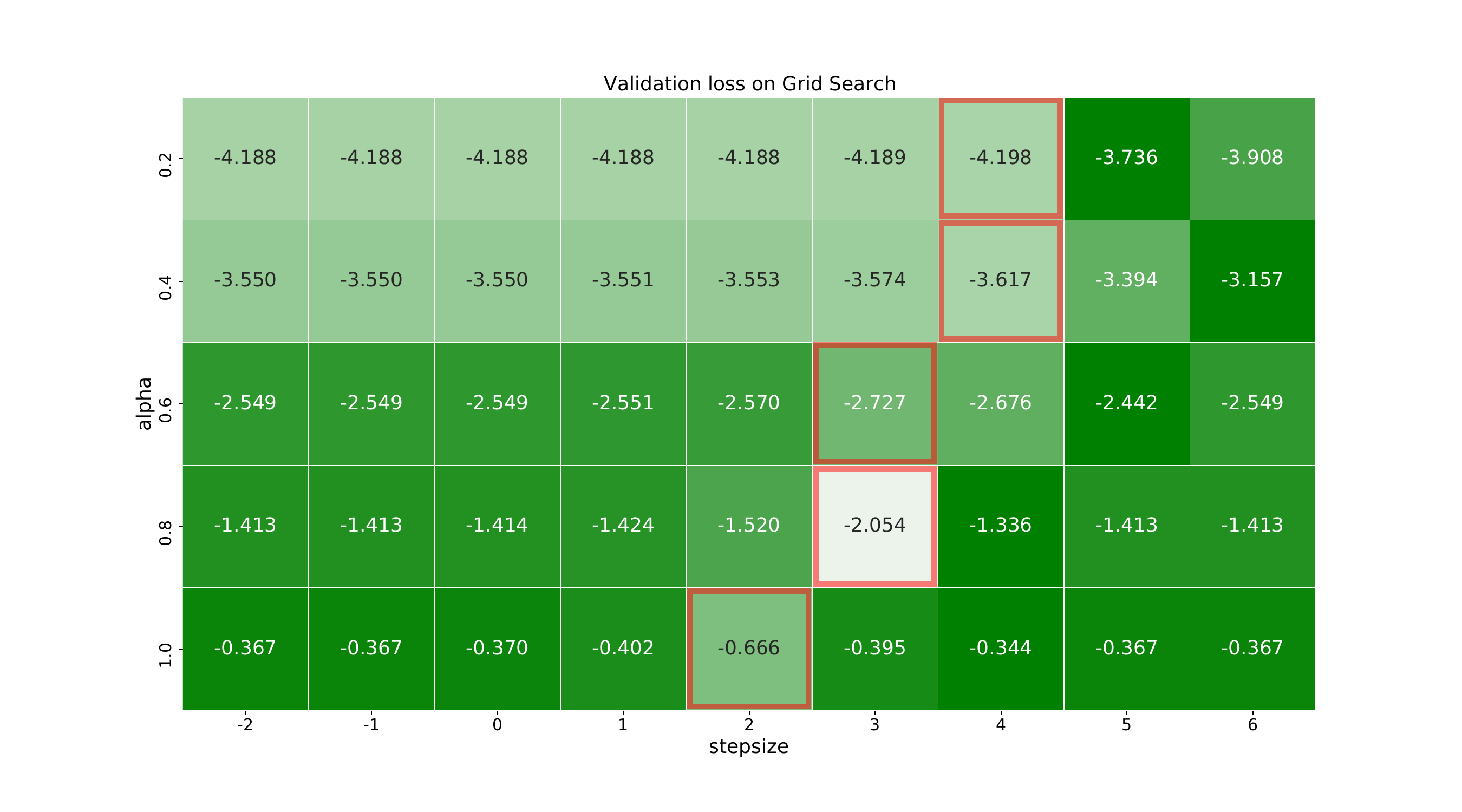}
	\caption{Grid-Search of stepsizes for FLIX for the first Stack Overflow experiment. Values in cells are logarithms of validation losses. Horizontal axis corresponds to logarithm of scaled step-sizes. Red frames correspond to minimum value in a row. Colors are scaled row-wise.}
	\label{fig:grid_search}
\end{figure*}

\begin{figure*}[t]
	\includegraphics[width=0.87\linewidth]{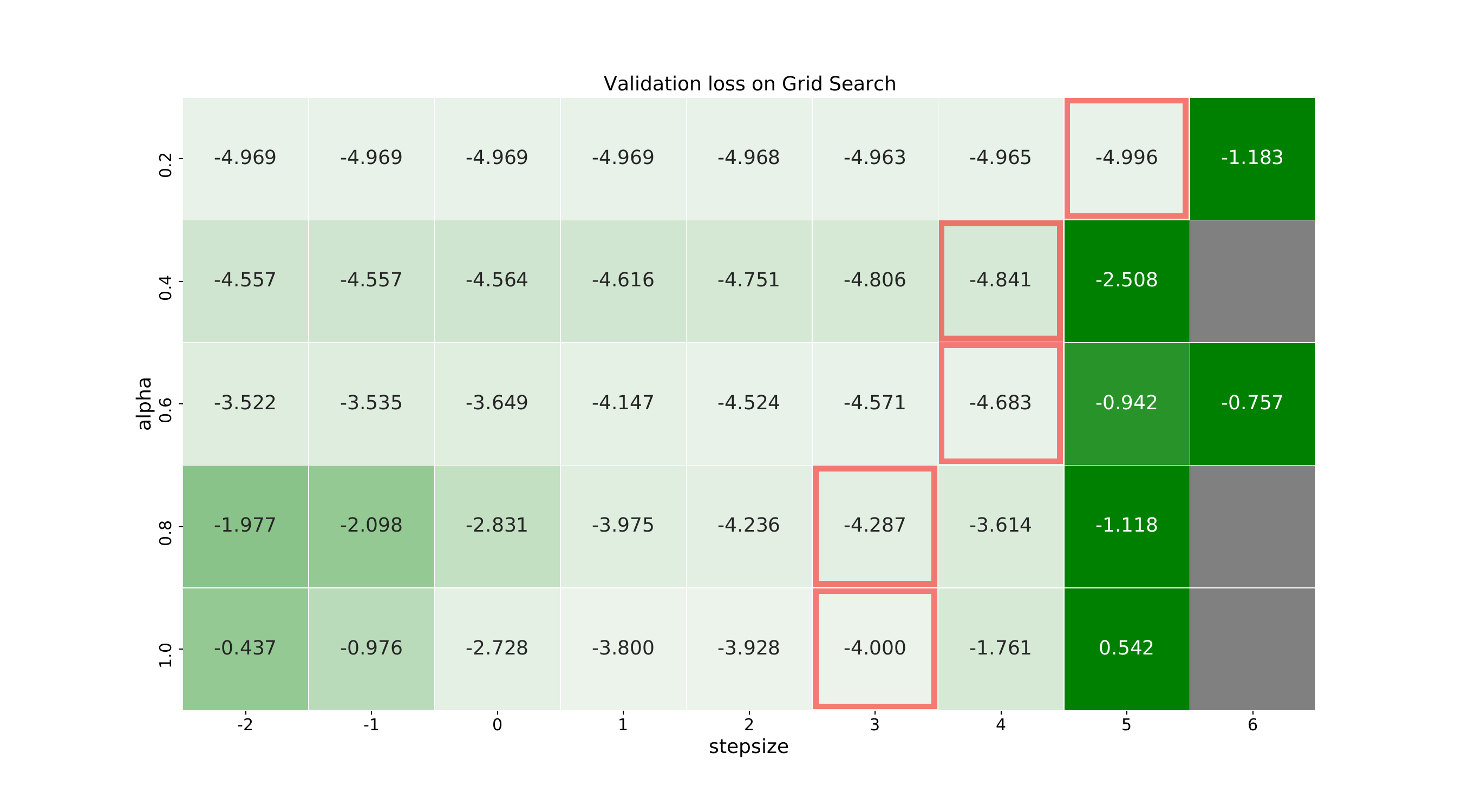}
	\caption{Grid-Search of stepsizes for FLIX for the second Stack Overflow experiment. Values in cells are logarithms of validation losses. Horizontal axis corresponds to logarithm of scaled step-sizes. Red frames correspond to minimum value in a row. Colors are scaled row-wise. Grey cells correspond to nan values.}
	\label{fig:grid_search2}
\end{figure*}
As table shows, the smaller $\alpha$ is, the higher step-size the task needs to achieve the best generalization, which is in line with our convergence results, see Theorem~\ref{theorem:dgd-convergence} and Table~\ref{table:results}.

To train FOMAML and Reptile, we set the number of inner steps to five and grid search outer and inner loop step-sizes. The explored outer step-sizes are the same as for FLIX. Inner step-sizes iterate over the set $[10^{-3}, 10^{-2}, \dots, 10]$.

After grid search, we run gradient descent for each value of alpha of FLIX, FOMAML, and Reptile for 10 000 and 50000 iterations in the first and seconds experiments accordigly and report obtained test accuracy.

{\bfseries Generalization experiment 3: Comparison to FedAvg with CNNs and LSTMs on large-scale datasets.} In~\Cref{table:high_scale_best_par} we indicate the best parameters for the tasks discussed in~\Cref{sec:experiments}.

\begin{table}[t]
\caption{Combinations of the parameters that achieve the highest accuracy. See~\Cref{sec:experiments} for the description of the grids.}
\label{table:high_scale_best_par}
\centering
\begin{tabular}{lll}
\textbf{PARAMETER} & \textbf{EMNIST} & \textbf{SHAKESPEARE} \\
\hline\\
Personalization parameter $\alpha$ & 0.3 & 0.5 \\
PLM batch size & 10 & 4\\
PLM learning rate & 0.001& 0.1 \\
FLIX batch size & 20 & 1 \\
FLIX learning rate & 0.01 & $10^{0.5}$\\
\end{tabular}
\end{table}

\clearpage
\section{EXTRA EXPERIMENTS}
\label{sec:extra-experiments}

{\bfseries Logistic regression with $l_2$ regularizer.} On figures~\ref{diana_log_reg_omega_all}, ~\ref{diana_log_reg_omega_comm_cost_all},~\ref{cgd_log_reg_omega_all}, and~\ref{cgd_log_reg_omega_comm_cost_all} we present the plots for regularized logistic regression loss for all datasets mentioned in the previous section and for all 3 algorithms: GD, CGD and DIANA. Number of machines $n$ is set to 50 for DIANA and GD, and is set to 8 for CGD. Plots showing the dependence between loss and communication cost show that although in terms of communication rounds DIANA and CGD lose to classical GD, they are better with respect to communication cost, which is of more practical importance. We run all the algorithms with their best theoretical step-sizes.

\begin{figure}[h!]
\includegraphics[width=\linewidth]{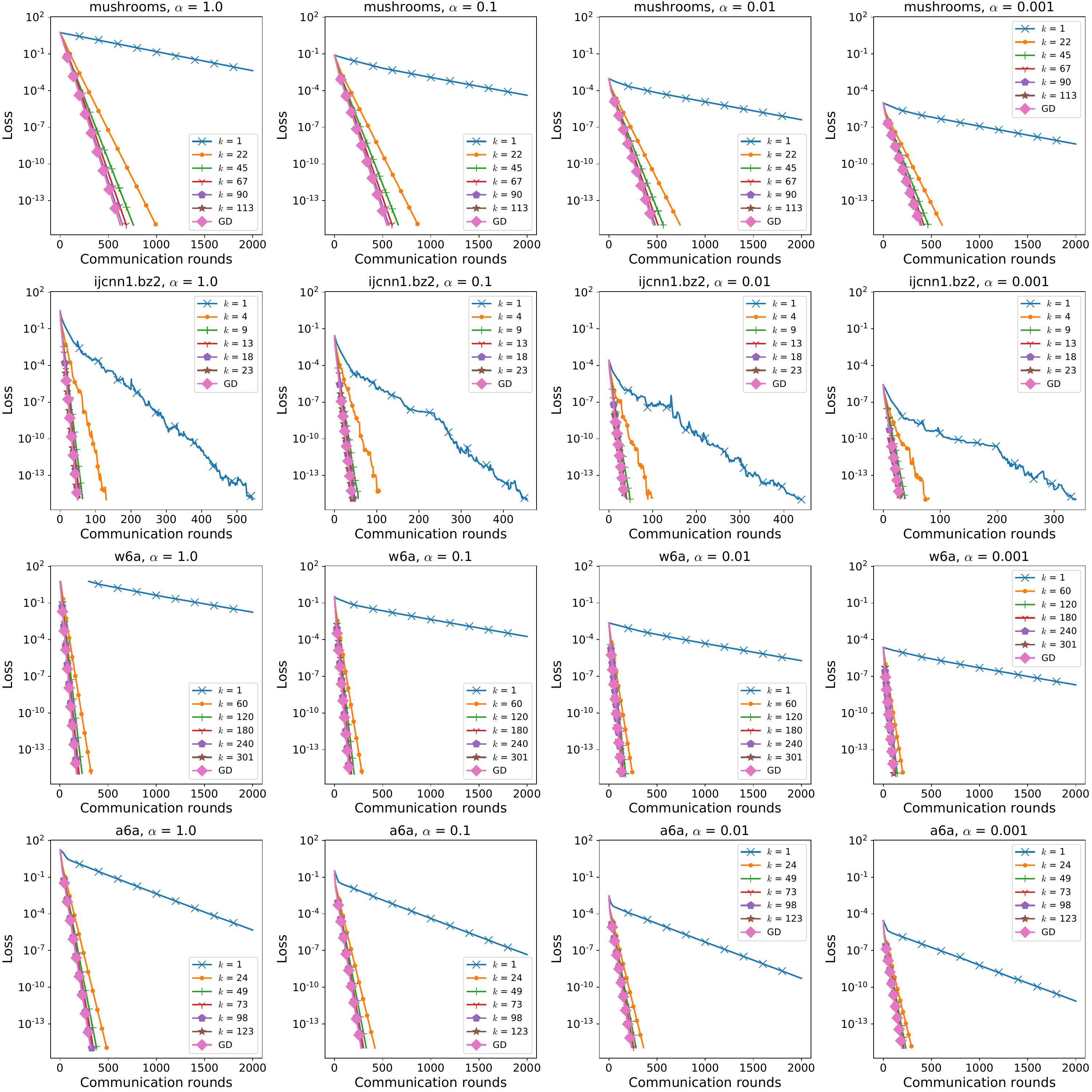}
\caption{Loss $f(x) - f^\ast$ vs. \# of communication rounds of DIANA and GD for logistic regression problem $l_2$ regularizer for four datasets, $k$  is a sparsification parameter of Random-$k$ compressor.}
\label{diana_log_reg_omega_all}
\end{figure}

\begin{figure}[h!]
\includegraphics[width=\linewidth]{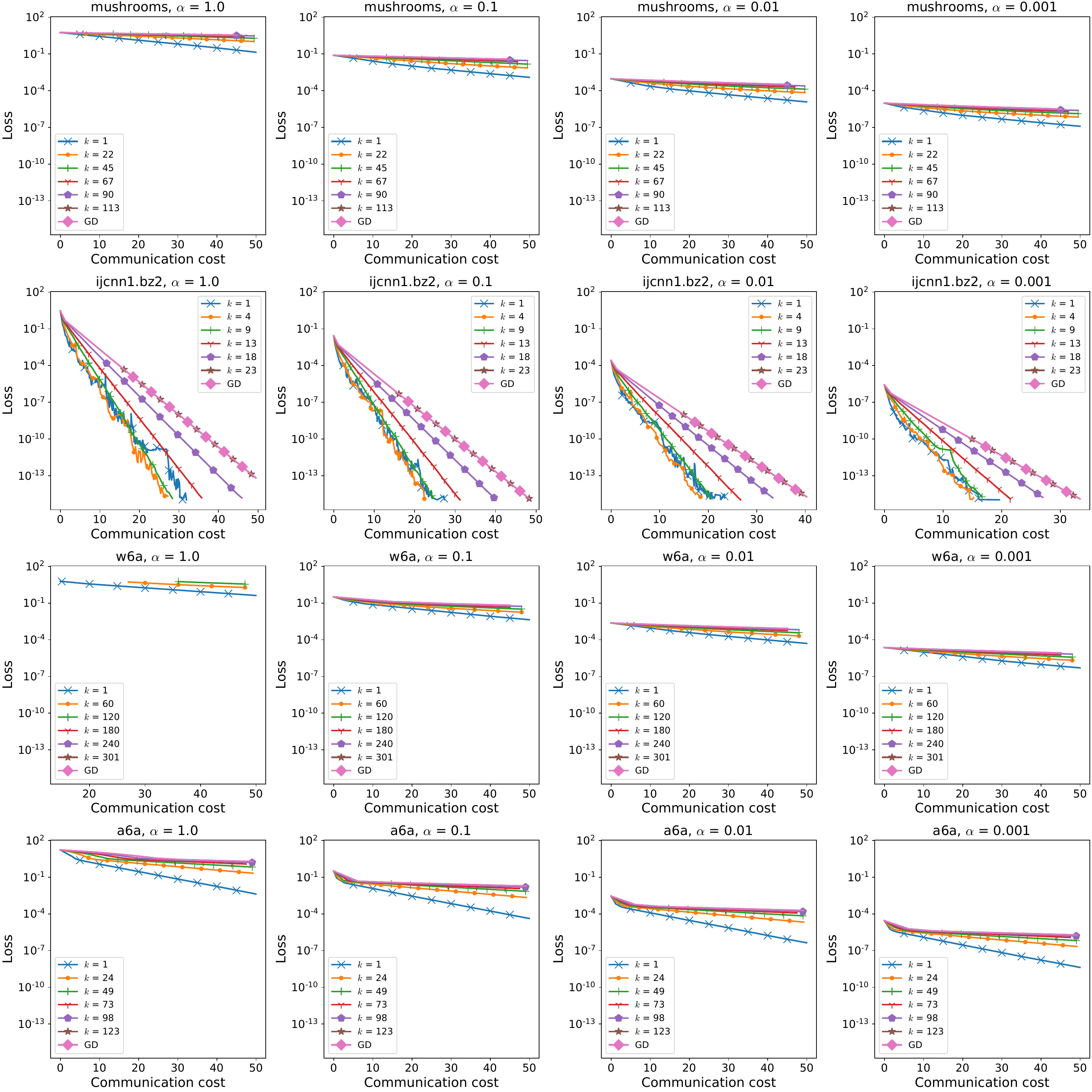}
\caption{Loss $f(x) - f^\ast$ vs. communication cost in unit of thousands of float numbers of DIANA and GD for logistic regression problem $l_2$ regularizer for four datasets, $k$  is a sparsification parameter of Random-$k$ compressor. }
\label{diana_log_reg_omega_comm_cost_all}
\end{figure}

\begin{figure}[h!]
\includegraphics[width=\linewidth]{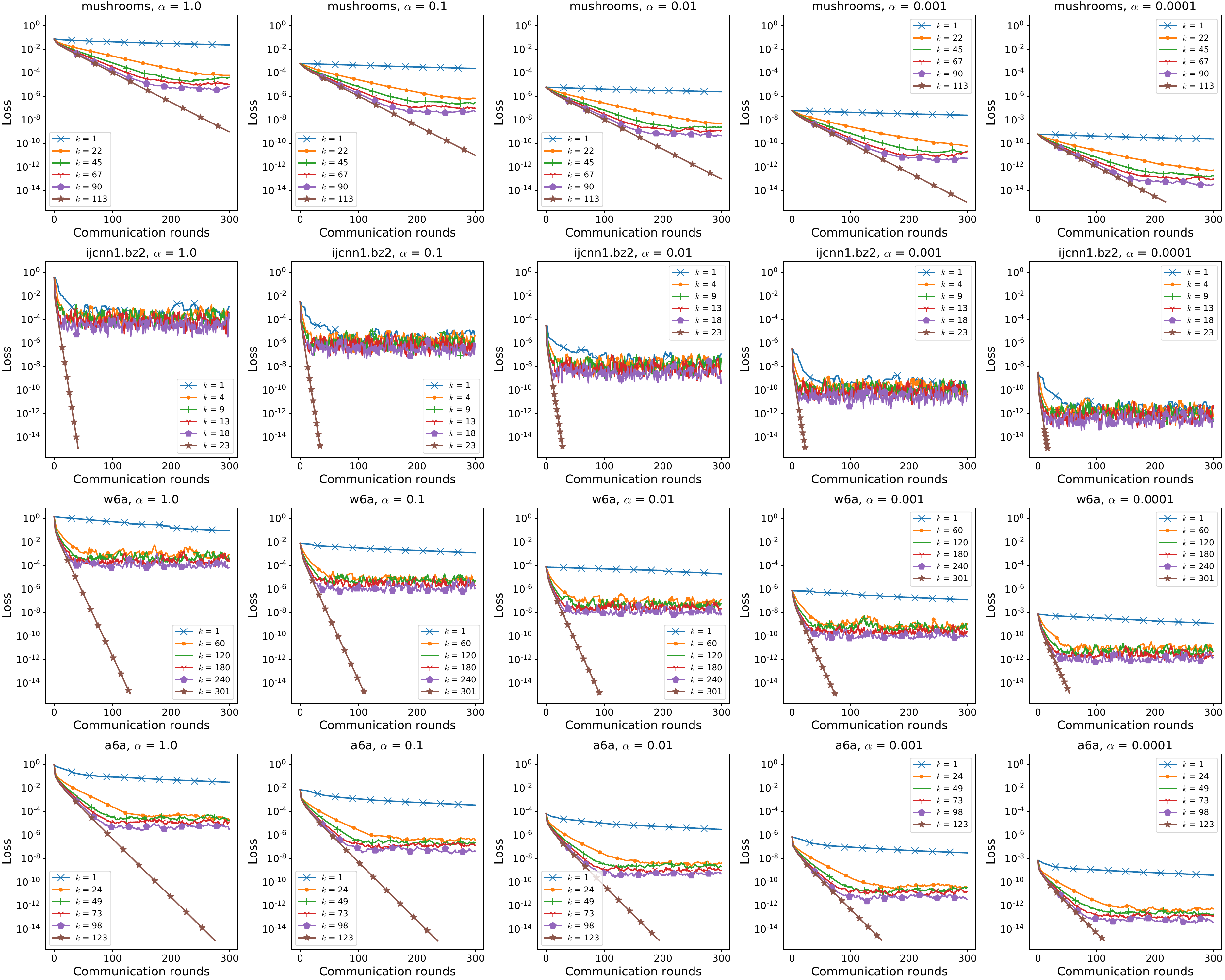}
\caption{Loss $f(x) - f^\ast$ vs. \# of communication rounds of CGD for logistic regression problem $l_2$ regularizer for four datasets, $k$  is a sparsification parameter of Random-$k$ compressor.}
\label{cgd_log_reg_omega_all}
\end{figure}

\begin{figure}[h!]
\includegraphics[width=\linewidth]{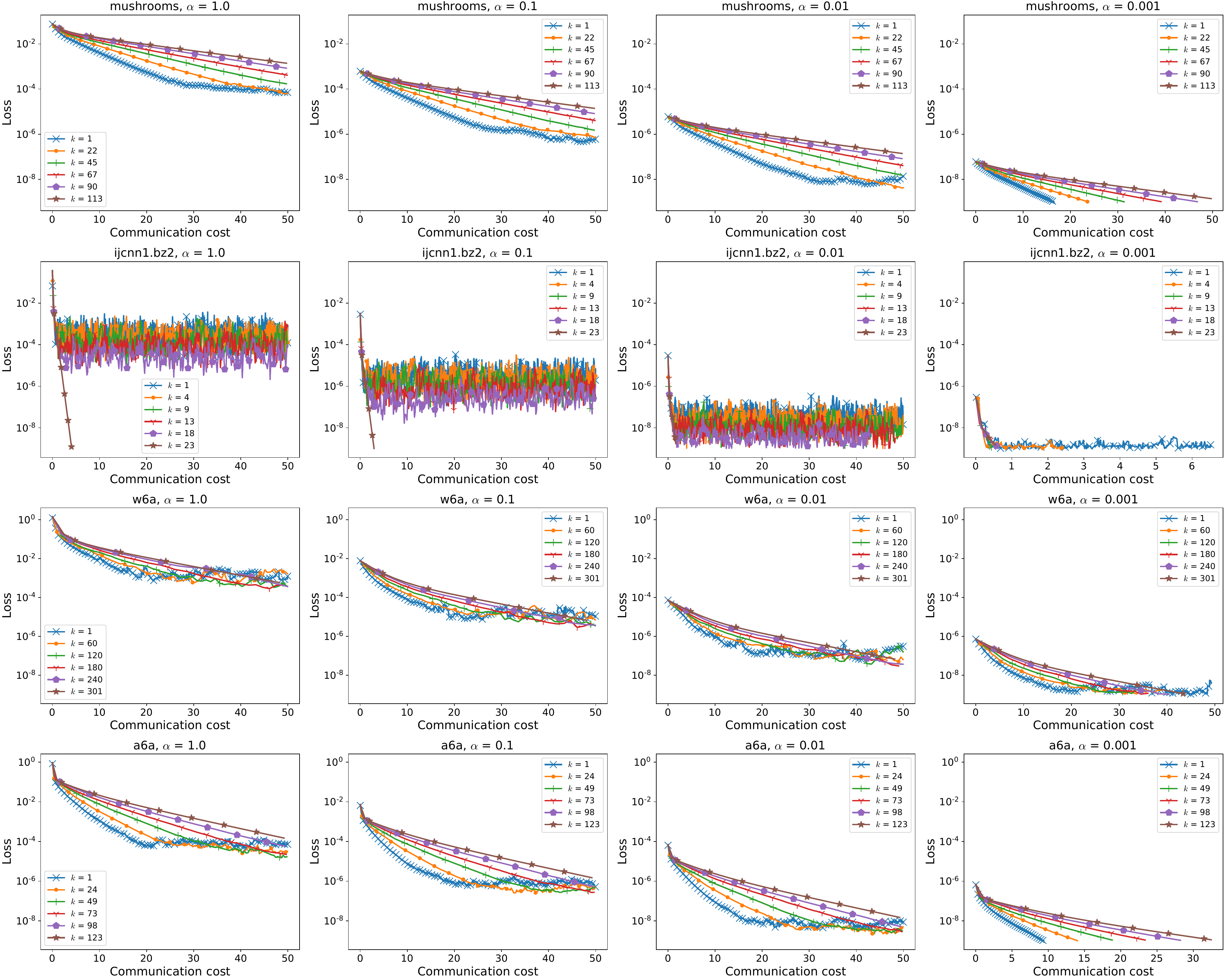}
\caption{Loss $f(x) - f^\ast$ vs. communication cost in unit of thousands of float numbers of CGD for logistic regression problem $l_2$ regularizer for four datasets, $k$  is a sparsification parameter of Random-$k$ compressor.}
\label{cgd_log_reg_omega_comm_cost_all}
\end{figure}

{\bfseries Unregularized logistic regression.} Setting $\lambda = 0$ in problem~\ref{fnct:logreg_l2}, we obtain unregularized logistic regression task, which is well-known to be a convex problem. Similarly to regularized case, we present four plots (~\ref{diana_log_reg_omega_all_cvx}, ~\ref{diana_log_reg_omega_comm_cost_all_cvx},~\ref{cgd_log_reg_omega_all_cvx}, and~\ref{cgd_log_reg_omega_comm_cost_all_cvx}) exhibiting convergence of three algorithms in terms of communication rounds and overall communication cost for four LIBSVM datasets. Number of machines $n$ is set to 100 for DIANA and GD, and to 8 for CGD. We run all the algorithms with their best theoretical step-sizes.

\begin{figure}[h!]
\includegraphics[width=\linewidth]{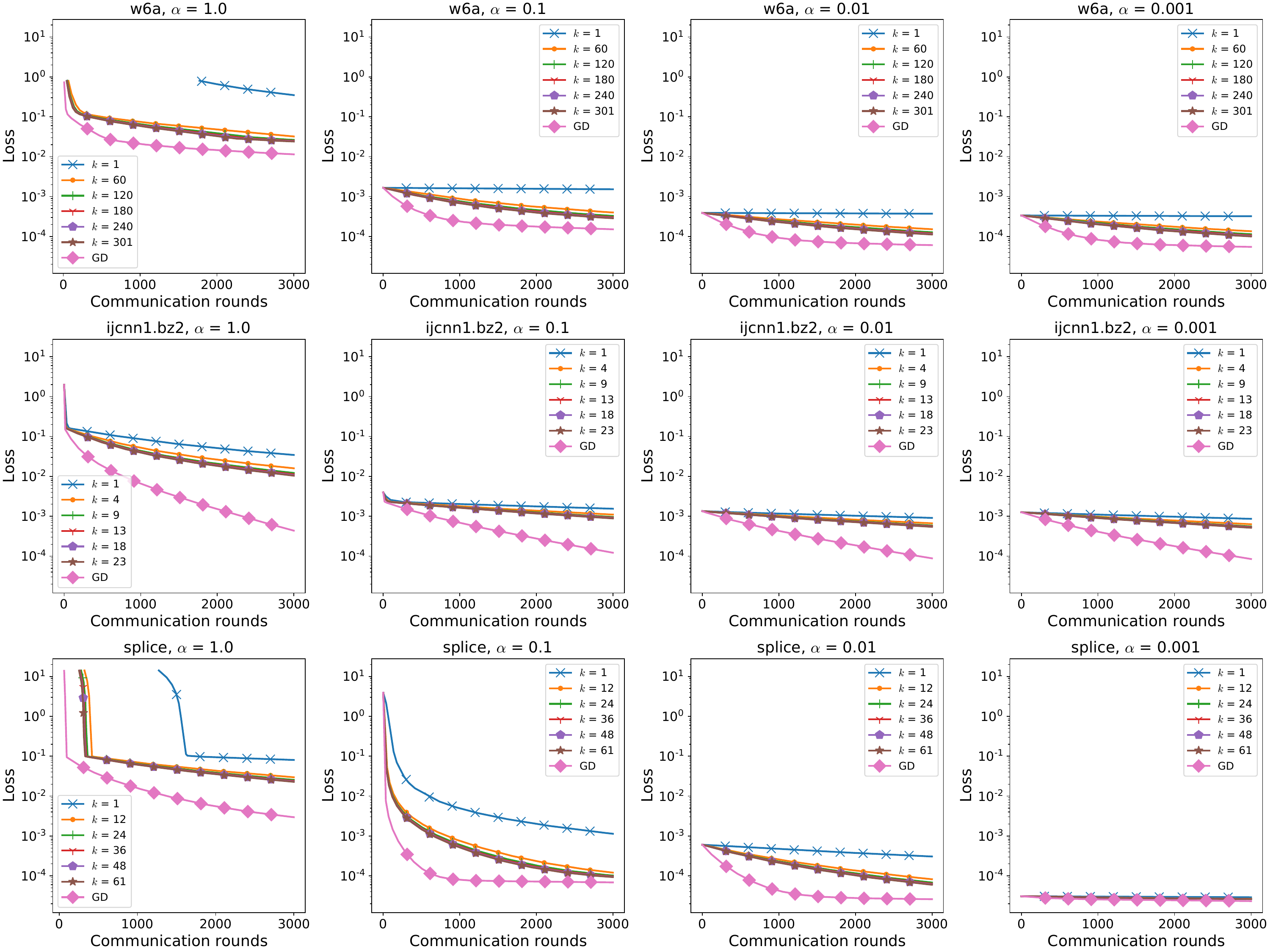}
\caption{Loss $f(x) - f^\ast$ vs. \# of communication rounds of DIANA and GD for unregularized logistic regression problem for four datasets, $k$  is a sparsification parameter of Random-$k$ compressor.}
\label{diana_log_reg_omega_all_cvx}
\end{figure}

\begin{figure}[h!]
\includegraphics[width=\linewidth]{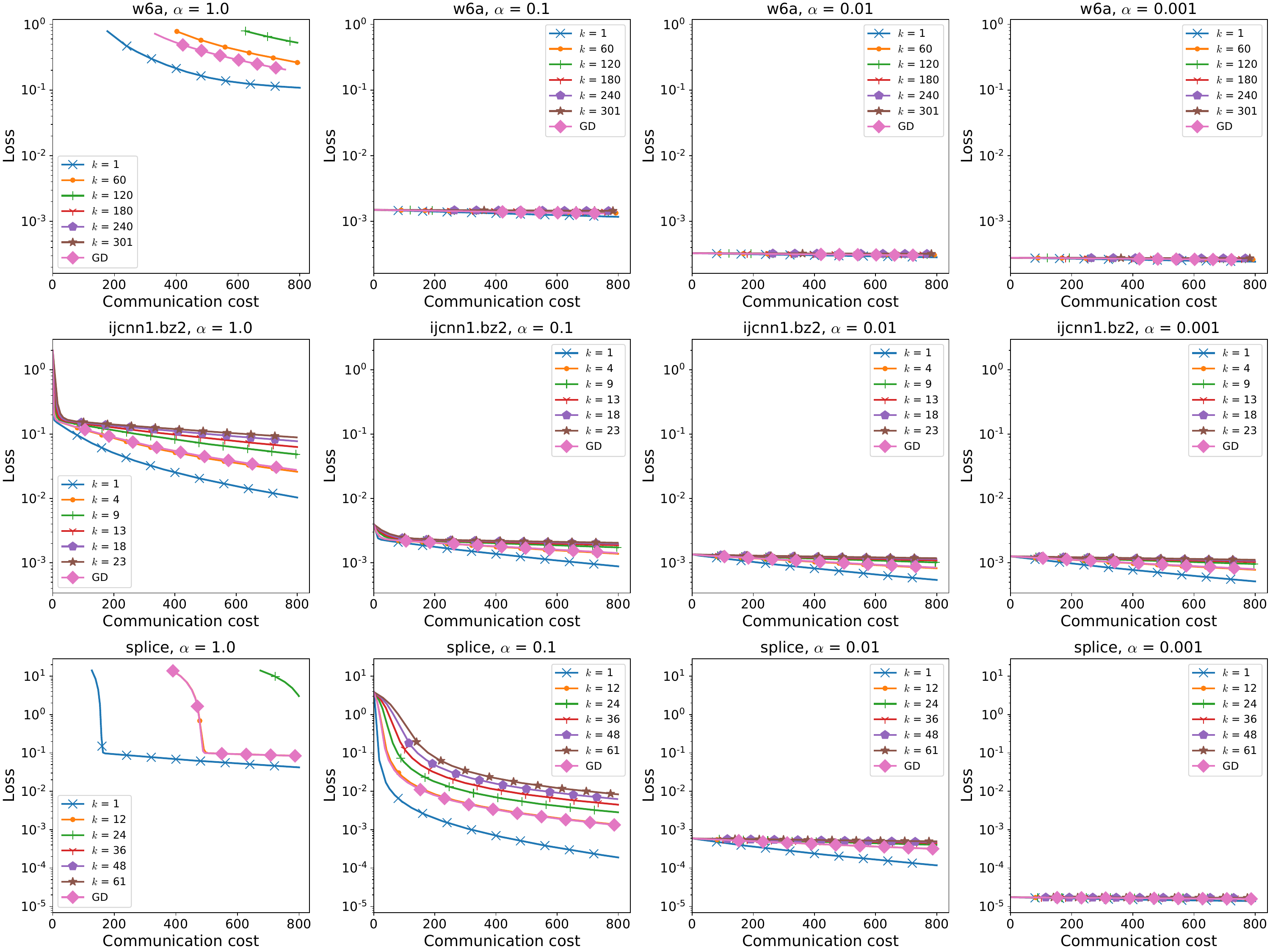}
\caption{Loss $f(x) - f^\ast$ vs. communication cost in unit of thousands of float numbers of DIANA and GD for unregularized logistic regression problem for four datasets, $k$  is a sparsification parameter of Random-$k$ compressor. }
\label{diana_log_reg_omega_comm_cost_all_cvx}
\end{figure}

\begin{figure}[h!]
\includegraphics[width=\linewidth]{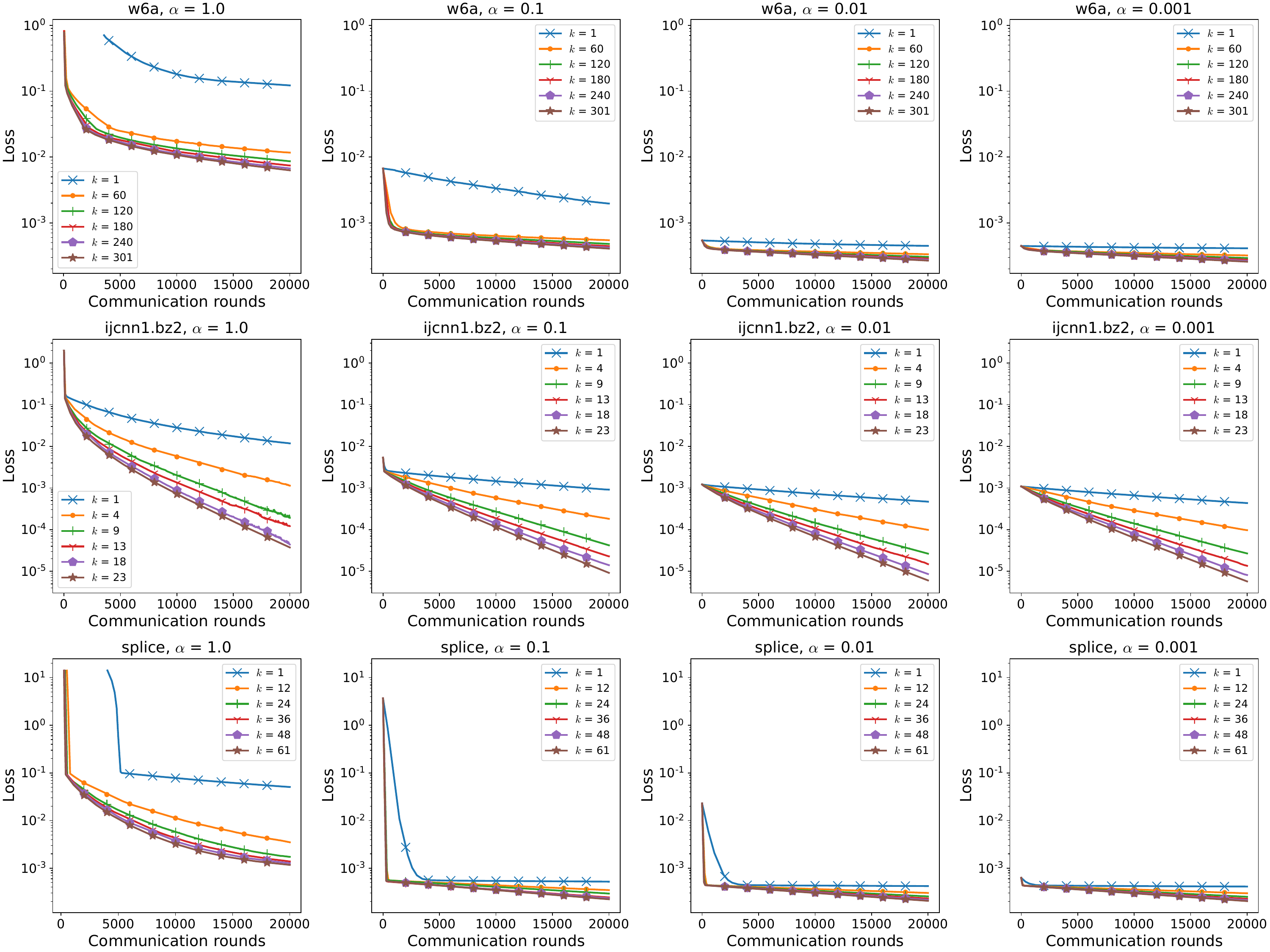}
\caption{Loss $f(x) - f^\ast$ vs. \# of communication rounds of CGD for unregularized logistic regression problem for four datasets, $k$  is a sparsification parameter of Random-$k$ compressor.}
\label{cgd_log_reg_omega_all_cvx}
\end{figure}

\begin{figure}[h!]
\includegraphics[width=\linewidth]{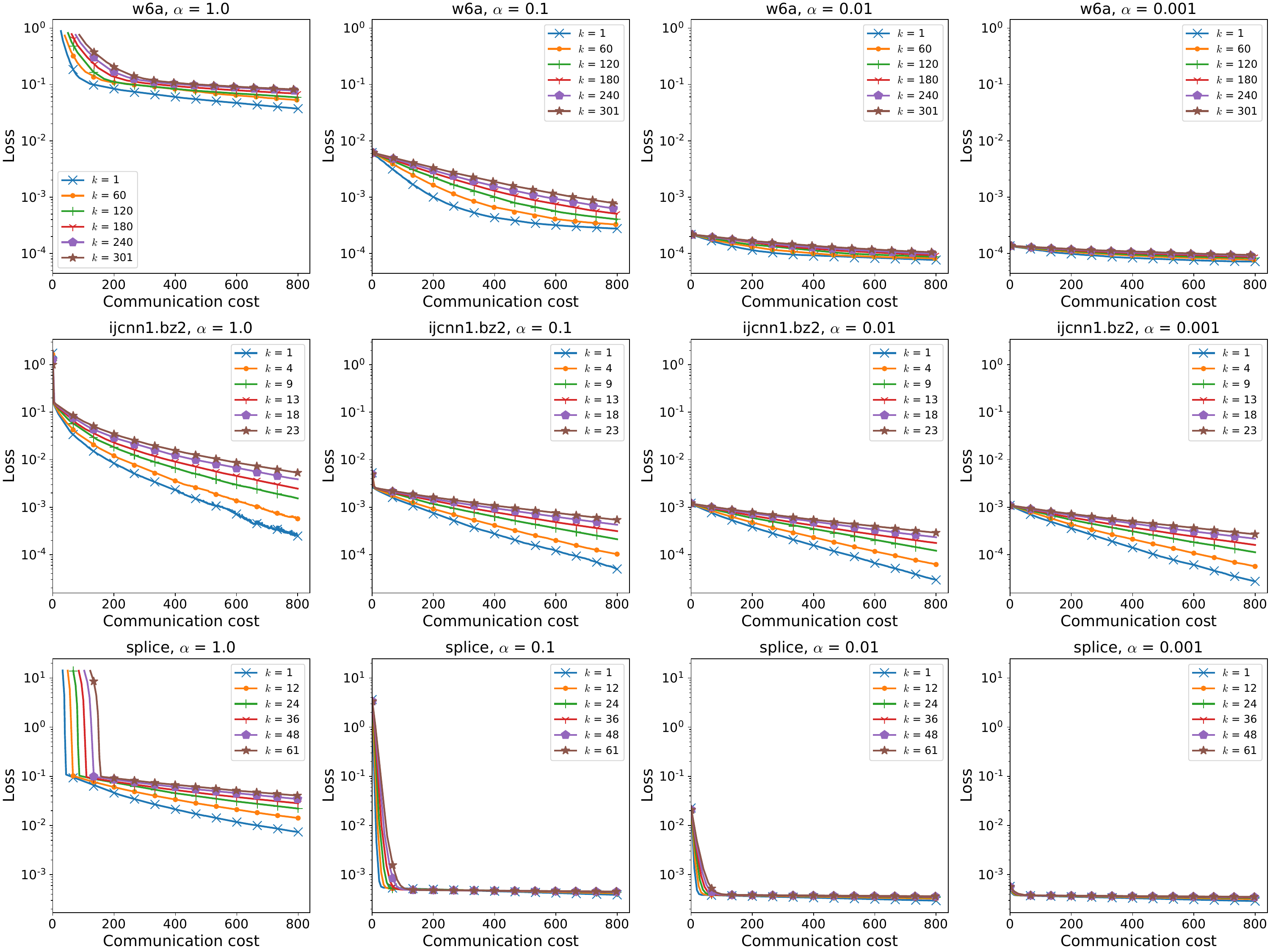}
\caption{Loss $f(x) - f^\ast$ vs. communication cost in unit of thousands of float numbers of CGD for unregularized logistic regression problem for four datasets, $k$  is a sparsification parameter of Random-$k$ compressor.}
\label{cgd_log_reg_omega_comm_cost_all_cvx}
\end{figure}

\end{document}